\documentclass{article}
\usepackage{arxiv}
\usepackage[utf8]{inputenc} 
\usepackage[T1]{fontenc}    
\usepackage{hyperref}       
\usepackage{url}            
\usepackage{booktabs}       
\usepackage{amsfonts}       
\usepackage{nicefrac}       
\usepackage{microtype}      
\usepackage{xcolor}         
\usepackage{makecell}
\usepackage[final]{changes}

\usepackage[english]{babel}
\usepackage{amsmath,amsfonts,amssymb}
\usepackage[export]{adjustbox}

\newcommand{\R}{\mathbb{R}}
\newcommand{\E}{\mathbb{E}}

\usepackage{bm}
\usepackage{multirow}
\usepackage{amsthm}
\usepackage{hyperref}  
\usepackage{thm-restate}

\usepackage{lipsum}		
\usepackage{graphicx}
\usepackage{natbib}
\usepackage{doi}
 \newtheorem{lem}{{Lemma}}
 \newtheorem{assumption}{Assumption}
 \newtheorem{cor}{{Corollary}}

\newtheorem{exmp}{Example}
 \newtheorem{rem}{Remark}
 
 \newtheorem{defi}{{Definition}}
 
\usepackage{algorithm} 
\usepackage{algorithmic}
\hypersetup{colorlinks, linkcolor=blue, citecolor=blue, urlcolor=magenta, linktocpage, plainpages=false}
\newcommand{\hw}{\hat{\omega}}
\newcommand{\hu}{\hat{u}}
\newcommand{\hlambda}{\hat{\lambda}}

\newcommand{\ustar}{u^{\ast}}
\newcommand{\lambdastar}{\lambda^{\ast}}
\usepackage{caption}

\title{Effective Bilevel Optimization via Minimax Reformulation}

\author{
 Xiaoyu Wang\thanks{Equal contribution.}\protect\phantom{\footnotesize 1} \\
	HKUST \\
	\texttt{maxywang@ust.hk} \\
 	\And
 Rui Pan\footnotemark[1]\protect\phantom{\footnotesize 1}  \\
	HKUST\\
	\texttt{rpan@connect.ust.hk} \\
	\AND
	Renjie Pi \\
	HKUST \\
	\texttt{rpi@connect.ust.hk} \\
     \And
	Jipeng Zhang \\
	HKUST \\
	\texttt{jzhanggr@connect.ust.hk} \\
}

\date{}

\begin{document}
\maketitle

\begin{abstract}
Bilevel optimization has found successful applications in various machine learning problems, including hyper-parameter optimization, data cleaning, and meta-learning. However, its huge computational cost presents a significant challenge for its utilization in large-scale problems. This challenge arises due to the nested structure of the bilevel formulation, where each hyper-gradient computation necessitates a costly inner optimization procedure.
To address this issue, we propose a reformulation of bilevel optimization as a minimax problem, effectively decoupling the outer-inner dependency. Under mild conditions, we show these two problems are equivalent. Furthermore, we introduce a multi-stage gradient descent and ascent (GDA) algorithm to solve the resulting minimax problem with convergence guarantees.
Extensive experimental results demonstrate that our method outperforms state-of-the-art bilevel methods while significantly reducing the computational cost.

\end{abstract}


\section{Introduction}


Bilevel optimization (BLO) has recently garnered considerable interest in research for its effectiveness in a variety of machine learning applications. Problems with hierarchical structures, such as hyperparameter optimization \citep{pmlr-v22-domke12,pmlr-v37-maclaurin15,iterative_der,lorraine2020optimizing}, meta-learning \citep{andrychowicz2016learning,franceschi2018bilevel,rajeswaran2019meta}, and reinforcement learning \citep{konda1999actor,hong2020two}, can all be represented as bilevel problems, making them well-suited to bilevel optimization methods. Formally, a bilevel optimization problem is defined as follows:
\begin{align}\label{P:bilevel}
    \min_{\lambda \in \Lambda} \quad & \quad  \mathcal{L}(\lambda) = L_{1}(u^{\ast}(\lambda), \lambda) \notag \\
    \text{s.t.} \quad & \quad u^{\ast}(\lambda) = \arg\min_{u} L_2(u, \lambda)
\end{align}
Here, $L_1(\cdot, \cdot): \R^{d}\times \R^{p} \rightarrow \R$ is the outer objective function, $L_2(\cdot, \cdot): \R^{d}\times \R^{p} \rightarrow \R$ is the inner objective function, and $\lambda \in \Lambda$ is the outer variable that is learned. For instance, in hyperparameter optimization, $L_1$ and $L_2$ correspond to the validation and training losses, respectively, while $u$ represents model parameters trained within $L_2$. The variable $\lambda$ represents the hyperparameters to be tuned, such as weight decay. Traditionally, $\lambda$ (e.g., weight decay) has been a single scalar value determined manually or through grid search~\citep{bergstra2012random}. Bilevel optimization, however, enables the automatic tuning of $\lambda$, which is especially useful when $\lambda$ is high-dimensional, as in per-parameter weight decay~\citep{grazzi2020iteration, luketina2016scalable, mackayself}, data cleaning~\citep{ren2018learning, shu2019meta, lorraine2020optimizing, yongholistic, gaoself}, and neural architecture search~\citep{liudarts, caiproxylessnas, xu2019pc, white2021bananas, shi2020bridging}.


Despite its flexibility and broad applicability, bilevel optimization remains underutilized in large-scale problems, such as foundation model training~\citep{radford2019language,touvron2023llama}. The main scalability challenges arise from the fundamental structure of bilevel optimization, which presents the following issues:

\begin{enumerate}
\item Bilevel optimization involves an outer-inner dependency, often leading to high computational costs.
\item Most bilevel optimization methods rely on second-order information, such as Hessians, which require significant memory.
\item The hierarchical nature of bilevel optimization complicates the theoretical analysis, particularly for stochastic extensions.

\end{enumerate}




Numerous approaches have been suggested to address these issues~\citep{shaban2019truncated, lorraine2020optimizing, mehra2021penalty}. However, none have fully resolved these challenges or produced a framework capable of handling very large-scale bilevel optimization problems, such as neural networks with billions of parameters and hyperparameters. This paper introduces a novel approach to completely address these limitations, with a simple yet effective main concept:

\quad \quad \quad \textit{Interpret the requirement for an inner optimum as an added constraint with a large penalty.}

This strategy naturally removes the outer-inner dependency from the original bilevel problem, reformulating it as an equivalent minimax optimization problem. We present a gradient-based MinimaxOPT algorithm for solving this minimax problem, which retains the same time and space complexity as direct training with gradient descent, and can be seamlessly extended to stochastic settings. Moreover, popular optimizers like SGD with momentum or Adam~\citep{Adam} can be integrated into MinimaxOPT without issue. To our knowledge, this is the first approach that scales bilevel optimization to extremely large problem sizes while maintaining compatibility with state-of-the-art optimizers.

\deleted{\added{In the framework of bilevel optimization \eqref{P:bilevel}} \replaced{the}{The} objective functions $L_1, L_2$ are \added{normally assumed to be} continuously differentiable and the constraint $\Lambda$ \replaced{being}{is} a closed convex set of $\R^p$. \added{In addition,} \replaced{suppose}{Suppose} that for each $\lambda$, the inner problem is solvable and admits a unique solution.}

\deleted{The bilevel problem is difficult to solve due to its nested nature that the dependence of the outer problem on $\lambda$ is also induced from the minimizer of the inner problem.}

\deleted{Gradient-based bilevel optimization is a widely used category of BLO in large-scale scenarios. The outer-level gradient-based method requires computing the hyper-gradient via the chain rule to update the outer variable:}

\added{\textbf{Related work.} The high computational cost of most gradient-based bilevel optimization methods arises from the well-known problem of outer-inner dependency. In particular, the outer-level optimization process necessitates the calculation of hyper-gradients}
\begin{align}\label{equ:hyper:grad}
\frac{\partial L_1}{\partial \lambda} = \frac{\partial L_1(u^{\ast}(\lambda), \lambda)}{\partial \lambda}  + \frac{\partial L_1(u^{\ast}(\lambda), \lambda)}{\partial u^{\ast}} \frac{\partial u^{\ast}(\lambda)}{\partial \lambda}.
\end{align}
\deleted{and the indirect gradient $ \frac{\partial u^{\ast}(\lambda)}{\partial \lambda}$ is not easy to access.} If the inner function is smooth, the derivative $\frac{\partial u^{\ast}(\lambda)}{\partial \lambda}$ \replaced{can be}{is} derived by implicit function theorem \replaced{$\frac{\partial L_2}{\partial u} (u^{\ast}(\lambda), \lambda)=0$}{$\frac{\partial L_2(u^{\ast}(\lambda), \lambda)}{\partial \lambda}=0$} \replaced{via}{so that}
\begin{align}\label{equ:implicit:grad}
\frac{\partial^2 L_2(u^{\ast}(\lambda), \lambda)}{\partial^2 u} \frac{\partial u^{\ast}(\lambda)}{\partial \lambda} + \frac{\partial^2 L_2(u^{\ast}(\lambda), \lambda)}{\partial \lambda \partial u}= 0    
\end{align}
\deleted{This allows to compute the hyper-gradient $\frac{\partial L_1}{\partial \lambda}$ by plugging the solution $\frac{\partial u^{\ast}}{\partial \lambda}$ of the linear system (\ref{equ:implicit:grad}) into (\ref{equ:hyper:grad}). However, the formula of $u^{\ast}(\lambda)$ is unknown or expensive to compute exactly.}

\deleted{One direct way to solve the linear system is to invert Hessian $\frac{\partial^2 L_2(u^{\ast}(\lambda), \lambda)}{\partial^2 u} $ which needs $\mathcal{O}(d^3)$ operations and is intractable for large-scale tasks.}

There are two significant challenges with this paradigm: 1) obtaining the minimizer $u^{\ast}$ of the inner problem, or at least an approximator, is necessary for calculating $\frac{\partial L_1(u^{\ast}, \lambda)}{\partial \lambda}$, $\frac{\partial L_1(u^{\ast}, \lambda)}{\partial u}$, $\frac{\partial^2 L_2(u^{\ast}, \lambda)}{\partial^2 u}$, and $\frac{\partial^2 L_2(u^{\ast}, \lambda)}{\partial \lambda \partial u}$; 2) computing $\frac{\partial u^{\ast} (\lambda)}{\partial \lambda}$ in Equation~\eqref{equ:implicit:grad} involves the Jacobian and Hessian of the inner function and may even require the Hessian inverse $\left(\frac{\partial^2 L_2(u^{\ast}, \lambda)}{\partial^2 u}\right)^{-1}$ if Equation~\eqref{equ:implicit:grad} is solved straightforwardly.

\deleted{The computational complexity of hyper-gradient for the bilevel models is inevitably very high without more careful design, especially in large-scale and high-dimensional practical applications.} 

To reduce the computational cost of the aforementioned approach, two types of methods have been proposed in past literature: 1) approximate implicit differentiable (AID) methods~\citep{pmlr-v22-domke12,pmlr-v48-pedregosa16,grazzi2020iteration,lorraine2020optimizing} and 2) iterative differentiable  (ITD) methods~\citep{pmlr-v22-domke12,pmlr-v37-maclaurin15,iterative_der,shaban2019truncated,grazzi2020iteration}.

In approximate implicit differentiable methods, $u^{\ast}(\lambda)$ is typically approximated by applying gradient-based iterative methods to optimize the inner problem. For instance, \citet{pmlr-v48-pedregosa16} proposed a framework that solves the inner problem and the linear system (\ref{equ:implicit:grad}) with some tolerances to balance the speed and accuracy, managing to optimize the hyper-parameter in the order of one thousand. Additionally, \citet{grazzi2020iteration} explored conjugate gradient (CG) and fixed-point methods to solve the linear system (\ref{equ:implicit:grad}) in conjunction with the AID framework. \citet{pmlr-v22-domke12} also utilized CG during the optimization process and demonstrated that the implicit CG method may fail with a loose tolerance threshold. Finally, \citet{lorraine2020optimizing} employed the Neumann series to approximate the inverse-Hessian, where Hessian- and Jacobian- vector products were used for hyper-gradient computation.

\deleted{The major drawback of AID methods is that they require a delicate tuning of tolerance to tradeoff the approximation accuracy and computation expense. If the tolerances of the AID methods are too loose, the resulting hyper-gradient can be very inaccurate. Conversely, the computation expense will be very high if the tolerances are too small. In practice, tolerance needs to be set appropriately to avoid enormous expenses  
without creating a useless hyper-gradient.}

In iterative differentiable methods, \citet{bengio2000gradient} applied reverse-mode differentiation (RMD), also known as backpropagation in the deep learning community, to hyper-parameter optimization.  \citet{pmlr-v22-domke12} considered the iterative algorithms to solve the inner problem for a given number of iterations and then sequentially computed the hyper-gradient using the back-optimization method. 
One well-known drawback of this conventional reverse-mode differentiation is that it stores the entire trajectory of the inner variables in memory, which becomes unmanageable for problems with many inner training iterations and is almost impossible to scale. To address this issue, \citet{pmlr-v37-maclaurin15} computes the hyper-gradient by  
reversing the inner updates of the stochastic gradient with momentum.  During the reverse pass, the inner updates are computed on the fly rather than stored in memory to reduce the storage of RMD. \citet{franceschi2018bilevel} studied the forward-mode and reverse-mode differentiation to compute the hyper-gradient of any iterative differentiable learning dynamics. Finally, \citet{shaban2019truncated} performed truncated back-propagation through the iterative optimization procedure, which utilizes the last $K_0$ intermediate variables rather than the entire trajectory to reduce memory cost. However, this approach sacrifices the accuracy of the hyper-gradient and leads to performance degradation.

Recently, stochastic bilevel optimization has also gained popularity in large-scale machine learning applications~\citep{ghadimi2018approximation,hong2020two,ji2021bilevel,chen2022single,khanduri2021near}. In this setting, the inner function $L_2$ and outer function $L_1$ either take the form of an expectation with respect to a random variable or adopt the finite sum form over a given dataset $\mathcal{D}$:
\begin{align}\label{P:bilevel:S}
    \min_{\lambda \in \Lambda} & \quad \E_{\xi}[L_1(u^{\ast}(\lambda), \lambda; \xi)] \notag\\
    \text{s.t.} & \quad u^{\ast}(\lambda) = \arg\min_{u} \E_{\zeta}[L_2(u, \lambda; \zeta)]
\end{align}
In this line of work, \citet{ghadimi2018approximation} proposed a bilevel stochastic approximation (BSA) algorithm, which employs stochastic gradient descent for the inner problem and computes the outer hyper-gradient (\ref{equ:hyper:grad}) by calling mutually independent samples to estimate gradient and Hessian. However, it primarily focuses on the theoretical aspects of BSA and lacks empirical results. Subsequently, \citet{ji2021bilevel} proposed stocBiO, a stochastic bilevel algorithm that calculates the mini-batch hyper-gradient estimator via the Neumann series and utilizes Jacobian- and Hessian-vector products. Some recent works propose the fully first-order method~\citep{pmlr-v202-kwon23c,chen2023near} which treats the inner-level problem as the penalty term, and the zeroth-order method~\citep{NEURIPS2022_1a82986c,yang2023achieving,aghasi2024fully} which estimates the hyper-gradient by finite difference.

Due to the nested structure of the bilevel problems, until now, few AID or ITD type methods can be readily extended to stochastic settings. Furthermore, this outer-inner dependency inevitably makes the computation of hyper-gradient reliant on the gradients of inner solutions. As all aforementioned AID/ITD methods follow this two-loops manner, it remains a significant challenge to achieve both low computational cost and competitive model performance for large-scale problems. Maintaining good theoretical guarantees on top of that would be even more difficult.

\deleted{The existing gradient-based bilevel methods always follow the two-loops manner. 
The big challenge is to compute the hyper-gradient, which inevitably involves the gradient of the inner solution. For AID and RMD methods, either the computation expense or space requirement is always very high, or one should sacrifice the accuracy to reduce the expense.}


\subsection{Contributions}
In this work, it becomes possible for the first time. To accomplish this, a novel paradigm is proposed: instead of directly solving the original bilevel problem (\ref{P:bilevel}), we approximate it using an equivalent \emph{minimax} formulation. To the best of our knowledge, this is the first method that has the potential to simultaneously achieve scalability, algorithmic compatibility, and theoretical extensibility for general bilevel problems. Specifically,

\deleted{This work proposes a new perspective to approximate the bilevel problem (\ref{P:bilevel}) by a \emph{minimax} formulation. We introduce an auxiliary variable $\omega$ to transform the nested problem (constraint) in (\ref{P:bilevel}) by $L_2(\omega, \lambda) - \max_{u} L_2(u, \lambda) \geq 0$ for any $\omega$, and then penalize this constraint to the outer function and formulate it as} 

\deleted{where $u \in \R^d$ is the inner variable, $\omega \in \R^d$, $\alpha > 0$ is the multiplier.
A detailed description of the proposed minimax problem and method is given in Section~\ref{sec:proposed:method}. We make the following contributions:}

\begin{itemize}
  \item This work proposes a new paradigm for general bilevel optimization that involves converting the problem into an equivalent minimax form. This approach opens up new possibilities for developing bilevel optimization methods that can address large-scale problems.
  \item An efficient optimization algorithm called MinimaxOPT is introduced to solve the minimax problem, which shares the same time/space complexity as gradient descent and can be extended to its stochastic version with ease. Furthermore, it can be seamlessly combined with popular optimizers, such as SGD momentum or Adam~\citep{Adam}.
  \item MinimaxOPT enjoys nice theoretical properties as common minimax optimization algorithms, where we provide theoretical convergence guarantees for cases when $L_2$ is strongly convex and $L_1$ is convex or strongly convex. Empirical results on multiple tasks are also provided to demonstrate its superiority over common bilevel optimization baselines.
\end{itemize}

\section{Proposed Problem and Method}\label{sec:proposed:method}

\deleted{In this section, we will describe how the minimax formulation (\ref{P:min:max}) can approximate the bilevel problem (\ref{P:bilevel}) and propose a minimax algorithm to solve problem~(\ref{P:min:max})}

\added{In the proposed reformulation of bilevel problem (\ref{P:bilevel}), an additional auxiliary variable $\omega$ is introduced to transform the nested inner problem $u^{\ast}(\lambda) = \arg\min_{u} L_2(u, \lambda)$ to constraint $L_2(\omega, \lambda) - \min_u L_2(u, \lambda) = 0$, where $\omega$ serves as a proxy to represent $u^{\ast}(\lambda)$:}
\begin{align}\label{P:bilevel_proxy}
    \min_{\lambda \in \Lambda} \quad & \quad  L_{1}(\omega, \lambda) \notag \\
    \text{s.t.} \quad & \quad  L_2(\omega, \lambda) - \min_u L_2(u, \lambda) = 0
\end{align}
\added{The additional constraint in~\eqref{P:bilevel_proxy} is then penalized with factor $\alpha$ in the outer function:}
\begin{align}\label{P:min:max}
    \min_{\omega, \lambda \in \Lambda}\max_{u}  L^{\alpha}(u, \omega, \lambda):= L_1(\omega, \lambda) + \alpha (L_2(\omega, \lambda) - L_2(u, \lambda))
\end{align}
\added{where $u \in \R^d$ is the inner variable whose optimum value is still $u^{\ast}(\lambda)$, while $\omega \in \R^d$, $\alpha > 0$ is the introduced multiplier. Intuitively, since $L_2(\omega, \lambda) -  \min_{u} L_2(u, \lambda)$ is always positive, for sufficiently large $\alpha$, minimizing $L^{\alpha}$ w.r.t. $\omega, \lambda$ is approximately equivalent to minimizing the second term $\alpha (L_2(\omega, \lambda) - \min_{u} L_2(u, \lambda))$, hence the inner constraint can be roughly satisified. It will then automatically turn to minimize the first term since $L_1(\omega, \lambda)$ now becomes the bottleneck when the inner constraint is satisfied. In such a manner, the approximation of both inner constraint and outer optimum can be obtained during the same optimization process, and $\alpha$ controls the priority.}

\deleted{Since $L_2(\omega, \lambda) -  \min_{u} L_2(u, \lambda)$ is always positive, for sufficient large $\alpha$, in order to minimize $L^{\alpha}$ w.r.t. $\omega, \lambda$, the second term $\alpha (L_2(\omega, \lambda) - \min_{u} L_2(u, \lambda)) \leq \mathcal{O}(1)$ must be controlled. It results in that the variable $\omega$ approximates the minimizer $u^{\ast}$ of $L_2(u, \lambda)$, and the accuracy is monitored by the multiplier $\alpha$.}

When $\alpha$ goes to infinity, the bilevel problem (\ref{P:bilevel}) and the proposed minimax problem~(\ref{P:min:max}) becomes \added{exactly} equivalent under some mild conditions. A precise description is given in Theorem~\ref{thm:equivalent:b:m}.

\begin{restatable}[]{thm}{thmequiv} \label{thm:equivalent:b:m}
Let $\lambda^{\ast}$ denote the solution of the bilevel problem and $u^{\ast}  = u(\lambda^{\ast})$ be the corresponding minimizer of the inner problem. We let  $(\hat{u}, \hat{\omega}, \hat{\lambda})$ denote the optimal solution of the minimax problem (\ref{P:min:max}).
Suppose that
\begin{itemize}
    \item[(i)] $L_1(\omega, \lambda) \ge 0$ is $\hat{L}_1$-Lipschitz continuous  with resepct to $\omega$.
    \item[(ii)] There exist $M, r > 0$, s.t. $| L_2(\hat{\omega}, \hat{\lambda}) - L_2(\hat{u}, \hat{\lambda})| \leq \Delta \Rightarrow \left\| \hat{\omega} - \hat{u} \right\|^r \leq M \Delta$.
\end{itemize}
  Denote $L_1^{\ast} \triangleq L_1(\ustar, \lambdastar)$, then for any fixed  $\alpha > 0$, the following statements hold:
\begin{itemize}
    \item[(1)] $0 \le L_2(\hat{\omega}, \hat{\lambda}) - L_2(\hat{u}, \hat{\lambda}) \leq \frac{L_1^{\ast}}{\alpha} $.

   \item[(2)] $L_1(u^{\ast}, \lambda^{\ast}) - \hat{L}_1 \cdot \left(\frac{ M L_1^{\ast}}{\alpha}\right)^{1/r}   \leq L_1(\hat{\omega}, \hat{\lambda}) \leq L_1(u^{\ast}, \lambda^{\ast}).$

\end{itemize}
\end{restatable}
Letting $\alpha \to \infty$, if $\omega, u \in \Omega$ and $\lambda \in \Lambda$ are all compact sets, we can obtain the exact equivalence $L_1(\hat{\omega}, \hat{\lambda}) = L_1(u^{\ast}, \lambda^{\ast})$, where a stronger result is available in Appendix~\ref{appendix:exact-equivalence-when-alpha-goes-to-inf}. Furthermore, Theorem \ref{thm:equivalent:b:m} shows that the minimax problem (\ref{P:min:max}) is an approximation of the bilevel problem (\ref{P:bilevel}) for any fixed $\alpha$ under mild conditions. The only uncommon condition is the second assumption, which is actually easy to satisfy. For example, $\mu$-strongly convex function (w.r.t. $\omega$) satisfies \emph{condition (ii)} with $r=2, M=2/\mu$ for $\hu = \arg\min_u L_2(u, \hlambda)$
\begin{align*}
 L_2(\hat{\omega}, \hat{\lambda}) -L_2(\hat{u}, \hat{\lambda}) \geq \left\langle \nabla_{u} L_2(\hat{u}, \hat{\lambda}), \hat{\omega} - \hat{u}\right\rangle + \frac{\mu}{2} \left\| \hat{\omega} - \hat{u} \right\|^2 \mathop{\ge}^{(a)} \frac{\mu}{2}\left\| \hat{\omega} - \hat{u} \right\|^2,
\end{align*}
\added{where $(a)$ uses the optimality of $\hat{u}$ on $L_2(u, \hlambda)$, which can be induced by the optimality of $\hat{u}$ on $L^{\alpha}$. More examples can be found in Appendix~\ref{appendix:bilevel-example-one-dim}.}

Before showing the theoretical results and proofs, we define
\begin{align}
    \Phi^{\alpha}(\omega, \lambda):= \max_{u} L^{\alpha}(u, \omega, \lambda); \quad u^{\ast}(\lambda) = \arg\max_{u} L^{\alpha}(u, \omega, \lambda)
\end{align}
 and 
\begin{align}
    \Gamma^{\alpha}(\lambda) = \min_{\omega} \Phi^{\alpha}(\omega, \lambda); \quad \omega_{\alpha}^{\ast}(\lambda) = \arg\min_{\omega} \Phi^{\alpha}(\omega, \lambda).
\end{align}
We make the following assumptions for the proposed minimax problem throughout this subsection.

\begin{assumption}\label{assumpt:v2}
We suppose that
\begin{itemize}
    \item[(1)] $L_1(\omega, \lambda)$ is twice continuous and differentiable, $\ell_{10}$-Lipschtiz continuous; $\ell_{11}$ gradient Lipschitz.
    \item[(2)]   $L_2(\omega, \lambda)$ is $\ell_{21}$ gradient Lipschtiz, $\ell_{22}$-Hessian Lipschtiz, and $\mu_2$-strongly convex in $\omega$.
\end{itemize}
\end{assumption}

\begin{restatable}[Stronger Equivalence to Bilevel optimization]{thm}{thmequivstronger} \label{thm:equiv}
Under Assumption \ref{assumpt:v2}, if $\alpha > 2\ell_{11}/\mu_2$, we have
\begin{subequations}
\begin{align}
|\mathcal{L}(\lambda) - \Gamma^{\alpha}(\lambda)| & \leq  \mathcal{O}\left(\frac{\kappa^2}{\alpha}\right)  \label{inequ:zero:dif}\\ 
 \left\| \nabla \mathcal{L}(\lambda) - \nabla \Gamma^{\alpha}(\lambda)\right\| & \leq  \mathcal{O}\left(\frac{\kappa^3}{\alpha}\right)\label{inequ:one:dif}\\ 
\left\| \nabla^2 \Gamma^{\alpha}(\lambda) \right\| & \leq \mathcal{O}(\kappa^3).  \label{inequ:second:dif} 
\end{align}
\end{subequations}
\end{restatable}


\deleted{From (2) of Theorem~\ref{thm:equivalent:b:m}, the multiplier $\alpha$ controls the accuracy of such an approximation. Especially if the multiplier $\alpha$ is dynamically adjusted and increased to infinity, the minimax problem (\ref{P:min:max}) is exactly equivalent to the bilevel problem (\ref{P:bilevel}). In Algorithm~\ref{alg:minmax:general}, we design the multi-stage gradient descent and ascent method to solve the minimax problem~(\ref{P:min:max}).}

\added{To solve the above equivalent minimax problem~\eqref{P:min:max}, we propose a general multi-stage gradient descent and ascent method named MinimaxOPT in Algorithm~\ref{alg:minmax:general}.}
 At each iteration, the algorithm performs gradient ascent over the variable $u$ and gradient descent over the variables $\omega$ and $\lambda$. This enables us to update the variables synchronously \added{and completely remove the outer-inner dependency issue in bilevel problems.} \deleted{, unlike the bilevel methods, which update the outer variable until the inner optimization converges.} The multiplier $\alpha$ is increased by a factor $\tau > 1$ after each stage \added{and gradually approaches infinity during the process}. \added{It is worth noticing that} Algorithm~\ref{alg:minmax:general} involves unequal step-sizes for $u,\omega$ and $\lambda$\added{, which is mainly in consideration of their difference in the theoretical properties entailed by their different mathematical forms.}

 \begin{algorithm}[ht]
\caption{Multi-Stage Stochastic MinimaxOPT}
\label{alg:minmax:general}
\begin{algorithmic}[1]
\STATE {\bfseries Input:} step-size sequences $\left\lbrace \eta_i^{u}, \eta_i^{\omega}, \eta_i^{\lambda}\right\rbrace$, initial penalty $\alpha_{-1}$, penalty sequence $\left\lbrace \Delta_{\alpha}^i\right\rbrace_{i=0}^N$, and initialization $u_0^0$, $\lambda_0^0$, $\omega_0^0$
\FOR{$i = 0: N$}
\STATE{$\alpha_{i} = \alpha_{i-1} + \Delta_{\alpha}^i  $} 
\FOR{$k = 0:K_i-1$}
\STATE{$(\tilde{u}_{0}, \tilde{\omega}_0)=(u_k^i, \omega_k^i)$}
\FOR{$t=0: T_k^i-1$}
\STATE{Generating iid samples  $D_{i,k}^t = \left\lbrace S_{\text{train}}^t, S_{\text{val}}^t \right\rbrace$ from $S_{train}$ and $S_{val}$}
\STATE{$\tilde{u}_{t+1} = \tilde{u}_{t} - \eta_i^u \nabla_u L^{\alpha_i}(\tilde{u}_{t}, \tilde{\omega}_{t}, \lambda_k^i; D_{i,k}^t) $
}
\STATE{$\tilde{\omega}_{t+1} = \tilde{\omega}_{t} - \eta_i^{\omega}\nabla_{\omega} L^{\alpha_i}(\tilde{u}_{t}, \tilde{\omega}_{t}, \lambda_k^i; D_{i,k}^t))$ }
\ENDFOR
\STATE{$(u_{k+1}^i, \omega_{k+1}^i)=(\tilde{u}_{T_k^i}, \tilde{\omega}_{T_k^i})$}
\STATE{Generating iid samples  $D_{i,k} = \left\lbrace S_{\text{train}}^{i,k}, S_{\text{val}}^{i,k} \right\rbrace$ from $S_{train}$ and $S_{val}$}
\STATE{${\lambda}_{k+1}^i = {\lambda}_{k}^i -   \eta_i^{\lambda} \nabla_{\lambda} L^{\alpha_i}(u_{k+1}^i, \omega_{k+1}^i, \lambda_k^i; D_{i,k})$}
\ENDFOR
\STATE{ $(u_{0}^{i+1},\omega_{0}^{i+1}, \lambda_{0}^{i+1})= (u_{K_i}^i, \omega_{K_i}^i, \lambda_{K_i}^i)$}
\ENDFOR
\STATE{{\bf Output:} $(u_{0}^{N+1}, \omega_{0}^{N+1}, \lambda_{0}^{N+1})$}
\end{algorithmic}
\end{algorithm}   

{\bf Discussions with the current algorithms:} When $K_i = 1$, then this framework degenerates to the F2A algorithm of \citep{kwon2023fully}. In \citep{pmlr-v202-shen23c}, the authors develop the value-function-based penalty function: at each step,  they first run gradient descent for $u$ until converging and then update $\omega$ and $\lambda$ sequentially.



\subsection{Stochastic Extension of Minimax Formulation}
The minimax formulation of the stochastic bilevel optimization is 
\begin{align}
 \min_{\omega, \lambda \in \Lambda} \max_{u} \E_{\xi \sim \Xi }[L_1(\omega, \lambda;\xi)] + \alpha \left(\E_{\zeta \sim Z}[L_2(\omega, \lambda;\zeta)] - \E_{\zeta \sim Z}[L_2(u, \lambda; \zeta)] \right)
\end{align}
We \added{can} randomly sample the outer function $L_1$ by a mini-batch set $S_1$ without replacement and the inner function by another mutually independent mini-batch set $S_2$ without replacement, then we optimize the following stochastic version of the minimax problem:
\begin{align}
\min_{\omega, \lambda \in \Lambda} \max_{u} L_{\mathcal{D}}^{\alpha} = L^{\alpha}(u, \omega, \lambda; \mathcal{D}) := L_1(\omega, \lambda; S_1) + \alpha\left(L_2(\omega, \lambda; S_2) - L_2(u, \lambda; S_2)\right)
\end{align}
where $\mathcal{D} = \left\lbrace S_1, S_2\right\rbrace$ and  $S_{1}$ is i.i.d. from the samples set $\left\lbrace 1,2,\cdots, m\right\rbrace$ of $L_1$, $S_{2}$ are i.i.d. from the sample set  $\left\lbrace 1,2,\cdots, n\right\rbrace$ of $L_2$ and independent with $S_{1}$. We use $\mathcal{F}_k$ to denote the random information before the iteration $(u_k, \omega_k, \lambda_k)$, that is $\mathcal{F}_k:= \left\lbrace (u_k, \omega_k, \lambda_k), D_{k-1}, \cdots, D_1\right\rbrace$. As we can see $L_{D_k}^{\alpha}$ is unbiased estimation of $L^{\alpha}$
\begin{align}
     \E[L_{D_k}^{\alpha}(u_k, \omega_k, \lambda_k) \mid \mathcal{F}_k] = L^{\alpha}(u_k, \omega_k, \lambda_k)
 \end{align}

\begin{rem}
One significant advantage of the MinimaxOPT algorithm is its ability to be easily extended to large-scale scenarios where only stochastic gradient oracles are available. For instance, in tackling the stochastic bilevel optimization problem (\ref{P:bilevel:S}), one can simply replace the gradient oracles with stochastic gradients to extend the algorithm. Additionally, popular optimizers such as Adam or SGD momentum can be incorporated into the algorithm, and the resulting generalized algorithms enjoy the same theoretical guarantees as applying Adam/SGD momentum to minimax problems. 
\end{rem}

\deleted{\replaced{, at}{.
At} each iteration, we \added{can} sample the outer function $L_1$ by a mini-batch set $S_1^k$ and the inner function by another mutually independent mini-batch set $S_2^k$, then we optimize the following minimax problem:}
\deleted{where $L_1(u^{\ast}(\lambda), \lambda; S_1^k) = \frac{1}{|S_1^k|}\sum_{i \in S_1^k} L_1(u^{\ast}(\lambda),\lambda; \xi)$ and $L_2(u, \lambda; S_2^k) = \frac{1}{|S_2^k|}\sum_{i \in S_2^k} L_2(u,\lambda; \zeta)$. Replacing the full gradient in Algorithm~\ref{alg:minmax:general} with the corresponding mini-batch stochastic gradients and incorporating momentum acceleration results in Algorithm \ref{alg:minmax:general} (in Appendix~\ref{append:dnn}). Its empirical performance will be shown in Section~\ref{sec:numerical}.}

\deleted{In the next example, we verify that the proposed minimax problem (\ref{P:min:max}) shares the same solution as the bilevel problem~(\ref{P:bilevel}) and Algorithm \ref{alg:minmax:general} yields a high-accuracy solution for the bilevel problem.}


\section{ Preliminaries and Theoretical Analysis}
In this section, we provide the theoretical convergence guarantees for the proposed algorithm (Algorithm \ref{alg:minmax:general}).
Before presenting the main results, we introduce some basic concepts and definitions used throughout this paper.
\begin{defi}
A function $f: \R^d \rightarrow \R$ is $\mu$-strongly convex if $f(x) \geq f(y) + \left\langle \nabla f(y), x - y\right\rangle + \frac{\mu}{2}\left\| x - y \right\|^2$ for any $x, y \in \R^d$.
\end{defi}
\begin{defi}
We call the function $f: \R^d \rightarrow \R$ being $\mu$-strongly concave if $-f$ is $\mu$-strongly convex.
\end{defi}
\begin{defi}
A function $f: \R^d \rightarrow \R$ is $L$-Lipschitz continuous if  $ \left\| f(x) -  f(y) \right\| \leq  L\left\| x - y \right\|$ for any $x, y \in \R^d$.
\end{defi}
\begin{defi}
A function $f: \R^d \rightarrow \R$ is $\ell$-smooth if  $ \left\| \nabla f(x) - \nabla f(y) \right\| \leq  \ell\left\| x - y \right\|$ for any $x, y \in \R^d$.
\end{defi}
\begin{assumption}(Bounded variance)\label{assumpt: bounded_variance} Suppose for each $\xi_i \in \Xi$ and $\zeta_j \in Z$, the followings hold:
  \begin{itemize}
    \item[(i)] $\E[\left\| \nabla L_1(\omega, \lambda; \xi_i) - \nabla  L_1(\omega, \lambda) \right\|^2] \leq  \sigma_1^2$
       \item[(ii)] $\E[\left\| \nabla L_2(\omega, \lambda; \zeta_j) - \nabla L_2(\omega, \lambda)\right\|^2] \leq \sigma_2^2$.
\end{itemize}
\end{assumption}
\subsection{One-stage Gradient Descent Ascent Algorithm}\label{subsec:onestage}
We first focus on the analysis when the outer step $N$ is 1, and the multiplier $\alpha$ is fixed. Then Algorithm \ref{alg:minmax:general} is reduced to gradient descent ascent  (GDA) method of optimizing a fixed objective $L^{\alpha}(u, \omega, \lambda)$. Especially, we consider two time-scale Algorithm 1 with $\eta_{u} = \eta_{\omega} = \mathcal{O}(1/\ell)$ and $\eta_{\lambda}$ is in another scale and smaller than $\eta^u$. It reflects the non-symmetric nature of the objective function with $u, \omega$, and $\lambda$. In general, if $L_1$ and $L_2$ are convex, then $L^{\alpha_i}(u, \omega, \lambda)$ is convex with respect to $\omega$ and concave with respect to $u$. However, $L^{\alpha_i}(u, \omega, \lambda)$
with respect to $\lambda$ is a DC function (i.e., convex minus convex function), nor a convex or concave function.  \deleted{It is reasonable to use different scales step-sizes such that $\eta_{\lambda} \leq \eta_{u}, \eta_{\omega}$. 
The typical analysis of minimax optimization for convex-concave problems is not applied in this setting, which brings us the challenges to provide convergence for Algorithm~\ref{alg:minmax:general}. }


Recalling the definition of $L^{\alpha}$, the variables $u, \omega$ are independent. That is to say, $u$ does not affect the property of $L^{\alpha}$ with respect to $\omega$. This implies that
$\Phi(\omega, \lambda)$ is also $\mu$-strongly convex with $\omega$ and $u^{\ast}(\lambda)$ is independent on $\omega$. We provide a technical lemma that structures the functions $\Phi$ and $\Gamma$ in the (strongly-concave)-(strongly-convex)-nonconvex setting. 
\begin{restatable}[]{lem}{lemonestagevtwo} \label{lem:onestage:v2}
Under Assumption~\ref{assumpt:v2}, if $\alpha > 2\ell_{11}/\mu_2$, the followings hold:
\begin{itemize}
\item[(i)] $L^{\alpha}$ is $\ell_{L}$-smooth where $\ell_{L}=\frac{5}{2}\alpha \ell_{21}$; $\mu_2\alpha$-strongly concave w.r.t. $u$; $\frac{\mu_2\alpha}{2}$-strongly convex w.r.t. $\omega$
\item[(ii)]$\Phi^{\alpha}(\omega, \lambda)$ is $\ell_{\Phi,\lambda}$-smooth w.r.t. $\lambda$ where $\ell_{\Phi,\lambda}=(\kappa+1)\ell_{L}$; $\Phi^{\alpha}(\omega, \lambda)$ is $\ell_{L}$-smooth w.r.t. $\omega$; and $u^{\ast}(\lambda)$ is $\kappa$-Lipschitz continuous;
    \item[(iii)] $\Gamma^{\alpha}(\lambda)$ is $\ell_{\Gamma}$-smooth and $\omega_{\alpha}^{\ast}(\lambda)$ is $\ell_{\omega^{\ast}}$-Lipschitz continuous where $\ell_{\omega^{\ast}}=2\kappa+1$.
\end{itemize}
Here $\kappa=\max\left\lbrace \ell_{10}, \ell_{11}, \ell_{21}, \ell_{22}\right\rbrace/\mu_2$ and $\ell_{\Gamma}$ is a constant which is independent on $\alpha$.
\end{restatable}


\begin{restatable}[]{prop}{proponestagevtwo}\label{thm:onestage:prop}
 Under Assumptions~\ref{assumpt:v2} and \ref{assumpt: bounded_variance} and suppose choosing the step-size as:
\begin{align*}
  \eta^{\lambda}=\frac{1}{\ell_{\Gamma}}; \quad  \eta^{u}=\frac{2}{\alpha(\mu_2 + \ell_{21})}; \quad \eta^{\omega} = \frac{4}{\alpha(\mu_2 + 3\ell_{21})}
\end{align*}
we consider Algorithm \ref{alg:minmax:general} with one-stage $N=1$ and any fixed $\alpha > 2\ell_{21}/\mu_2$ and $\zeta >0$
\begin{subequations}
\begin{align}
T_k  & \geq \frac{3\kappa-1}{4}\ln\left(\frac{12(\ell_{11}^2 + \alpha^2\ell_{21}^2)\max \left\lbrace \E[\delta_k^2], \E[r_k^2]\right\rbrace}{\zeta^2}\right) \notag \\
B & =  \frac{12\kappa \left(\frac{1}{2\alpha} + 1\right)(\sigma_1^2 + \alpha^2\sigma_2^2)}{\zeta^2} \notag
\end{align}
\end{subequations}
where
$$\max \left\lbrace \E[\delta_k^2], \E[r_k^2]\right\rbrace \leq  \left\{ 
\begin{aligned}
  &\frac{\zeta^2}{3\alpha^2\ell_{21}^2} + 2\ell_{\omega^{\ast}}^2 (\eta^{\lambda})^2 \left(2\zeta^2 + 4\ell_{11}^2 + 4\kappa^2 + \frac{\sigma_1^2 + \alpha^2\sigma_2^2}{B} \right),  &  k \geq 1  \\
& \max \left\lbrace \left\|u_0 - u^{\ast}(\lambda_0)\right\|^2, 2\left\|\omega_0 - \omega^{\ast}(\lambda_0)\right\|^2 + \frac{2\kappa^2}{\alpha^2}\right\rbrace  & k=0
\end{aligned}
\right.
$$
We can achieve that $\E[\left\|u_{k+1} - u^{\ast}(\lambda_{k})\right\|^2] \leq \frac{\zeta^2}{6\alpha^2\ell_{21}^2}$ and $\E[\left\|\omega_{k+1} - \omega_{\alpha}^{\ast}(\lambda_{k})\right\|^2] \leq \frac{\zeta^2}{6(\ell_{11}^2 + \alpha^2\ell_{21}^2)}$ hold for all $k$ and further demonstrate that 
$\E[\left\| \nabla_{\lambda} L^{\alpha}(u_{k+1}, \omega_{k+1}, \lambda_{k}; D_{k}) - \nabla \Gamma^{\alpha}(\lambda_k) \right\|^2] \leq \zeta^2$ for all $k \leq K-1$. 
\end{restatable}

\begin{restatable}[]{thm}{thmonestagevtwo}\label{thm:onestage:v2}
Let $\alpha=\mathcal{O}(\kappa^3\epsilon^{-1})$ and  $\zeta=\mathcal{O}(\epsilon)$ and suppose all the conditions in Proportion~\ref{thm:onestage:prop} hold.
After $K=\mathcal{O}(\kappa^4\epsilon^{-2})$ steps, we can reach $ \frac{1}{K}\sum_{k=0}^{K-1}\E[\left\|\nabla \Gamma^{\alpha}(\lambda_k) \right\|^2] \leq \epsilon^2$.
\end{restatable}

\begin{cor}
 Under the conditions of Theorem~\ref{thm:onestage:v2}, the whole gradient oracle complexity of Algorithm \ref{alg:minmax:general} to achieve an $\epsilon$-solution is $\mathcal{O}(3 BK + 3B K T_k) = \mathcal{O}(\epsilon^{-6}\log(1/\epsilon))$.
\end{cor}

\begin{restatable}[]{thm}{thmonestagevtwo:sigma2zero}\label{thm:onestage:v2:sigmazero}
Suppose all the conditions in \ref{thm:onestage:prop} hold and consider Assumption~\ref{assumpt: bounded_variance} with $\sigma_2=0$. Then we let $\alpha = \mathcal{O}(\kappa^3\epsilon^{-1})$ and $\zeta = \mathcal{O}(\epsilon^{-1})$, then after $K = \kappa^4 \epsilon^{-2}$ steps and $B = \kappa \epsilon^{-2}$, we have 
the total gradient oracle complexity is $ \mathcal{O}(BKT_k) = \mathcal{O}(\kappa^6 \epsilon^{-4}\log(1/\epsilon))$.
\end{restatable}


\subsection{Multi-stage Gradient Descent Ascent Algorithm} \label{subsec:multistage}


\begin{restatable}[]{prop}{proponestagevtwo}\label{thm:multistage:prop}
 Under Assumptions~\ref{assumpt:v2} and \ref{assumpt: bounded_variance},
we consider Algorithm \ref{alg:minmax:general} with multi-stage $N>1$ with the multiplier $\alpha_i > 2\ell_{11}/\mu_2$, and select the step-sizes  as:
\begin{align*}
  \eta^{\lambda}=\frac{1}{\ell_{\Gamma}}; \quad  \eta^{u}=\frac{2}{\alpha_i(\mu_2 + \ell_{21})}; \quad \eta^{\omega} = \frac{4}{\alpha_i(\mu_2 + 3\ell_{21})}
\end{align*}
Then at each stage $i$, suppose that
\begin{subequations}
\begin{align}
T_k^i  & \geq \frac{3\kappa-1}{4}\ln\left(\frac{12(\ell_{11}^2 + \alpha_i^2\ell_{21}^2)\max \left\lbrace \E[(\delta_k^i)^2], \E[(r_k^i)^2]\right\rbrace}{\zeta_i^2}\right) \label{eq: Tk_max_multistage_final}\\
B_k^i & =  \frac{12\kappa \left(\frac{1}{2\alpha_i} + 1\right)(\sigma_1^2 + \alpha_i^2\sigma_2^2)}{\zeta_i^2} \label{eq: B_max_multistage_final}
\end{align}
\end{subequations}
where
$$\max \left\lbrace \E[(\delta_k^i)^2], \E[(r_k^i)^2]\right\rbrace \leq  \left\{ 
\begin{aligned}
  &\frac{\zeta_i^2}{3\alpha_i^2\ell_{21}^2} + 2\ell_{\omega^{\ast}}^2 (\eta_i^{\lambda})^2 \left(2\zeta_i^2 + 4\ell_{11}^2 + 4\kappa^2 + \frac{\sigma_1^2 + \alpha_i^2\sigma_2^2}{B_i} \right),  &  k \geq 1  \\
& \max \left\lbrace \delta_0^0, r_0^0, 2\Delta_0(\alpha_0, \zeta_0, B_0)   + \frac{8\ell_{10}^2}{\mu_2^2\alpha_0^2}\right\rbrace   & k=0
\end{aligned}
\right.
$$
we can achieve that $\left\| \nabla \Gamma^{\alpha_i}(\lambda_k^i) - \nabla_{\lambda} L^{\alpha}(u_{k+1}^i, \omega_{k+1}^i, \lambda_k^i)\right\|^2 \leq \zeta_i^2$ at each stage $i$. 
\end{restatable}

\begin{restatable}[]{thm}{thmonestagesto}\label{thm:onestage:sto}
Under Assumptions~\ref{assumpt:v2} and \ref{assumpt: bounded_variance}, we consider Algorithm \ref{alg:minmax:general} with multi-stage $N>1$ with the multiplier $\alpha_i > 2\ell_{11}/\mu_2$, then
we let $\alpha_i = \alpha_0 \tau^i, \zeta_i=\tau^{-i}$, $K_i=\tau^{2i}$, $N=\log_{\tau}(1/\epsilon)$, and $T_k^i =\mathcal{O}(\kappa i)$ and $B_k^i=\tau^{4i}$, then  to obtain an $\epsilon$-stationary point, we need at least
the total number of iterations $\Sigma = \sum_{i=0}^{N} \sum_{k=1}^{K_i} T_k^i = \mathcal{O}\left(\epsilon^{-2}\right)$ and the total gradient oracle complexity is $\mathcal{O}\left(\sum_{i=0}^N \sum_{k=0}^{K_i} B_k^i\right) = \mathcal{O}(\epsilon^{-6}\log(1/\epsilon)) $.
    
\end{restatable}

\section{Numerical Experiments}\label{sec:numerical}

We evaluate the practical performance of the proposed minimax framework for solving bilevel problems. We start with the experiments on a linear model with logistic regression for the deterministic version of \replaced{MinimaxOPT}{the minimax algorithms}, then further explore \replaced{its stochastic counterpart}{the stochastic version of the proposed minimax algorithm} in deep neural networks and a hyper-data cleaning task\deleted{ which are the representative instances of bilevel optimization in machine learning}.

\subsection{Hyper-parameter Optimization for Logistic Regression with $\ell_2$ Regularization} \label{sec:numerical:1}

 The first problem is estimating hyperparameters of $\ell_2$ regularized logistic regression problems, in other words, hyperparameter optimization for weight decay:
 
 \deleted{The loss function of the inner problem is the per-parameter regularized logistic loss function. The dataset is
partitioned into two sets: a train set $S_{\text{train}} = \left\lbrace a_i, b_i \right\rbrace_{i=1}^n$ and a validation set $S_{\text{val}} = \left\lbrace a_i, b_i \right\rbrace_{i=1}^m$ where $a_i, b_i$ are denoted the input feature and true label. In the setting of classification, we use logistic regression as the validation loss of the outer problem, formulated as }
\begin{align*}
    \min_{\lambda \in \R_{+}^{d}} \,\,\, & \quad L_1(u^{\ast}(\lambda), \lambda) = \sum_{i \in S_{\text{val}}} \log\left(1+\exp(-b_ia_i^{T} u^{\ast}(\lambda))\right) \\
  \text{s.t.} \,\,\, &  \quad u^{\ast}(\lambda) = \arg\min_{u \in \R^d}  L_2(u, \lambda):= \sum_{j \in S_{\text{train}}} \log\left(1+\exp(-b_ja_j^{T}u)\right) + \frac{1}{2} u^{T}\text{diag}(\lambda)u.
\end{align*}
Here $S_{\text{train}} = \left\lbrace a_i, b_i \right\rbrace_{i=1}^n$ and $S_{\text{val}} = \left\lbrace a_i, b_i \right\rbrace_{i=1}^m$ reprenset the training and validation set respectively, with $a_i, b_i$ being the features and labels.
This setting is similar to section 3.1 of \citep{grazzi2020iteration}, where we test MinimaxOPT in Algorithm \ref{alg:minmax:general} and compare with three popular bilevel methods: (1) \emph{Reverse}-mode differentiation, which computes the hyper-gradient by $K_0$-truncated back-propagation~\citep{iterative_der,shaban2019truncated}; (2) \emph{Fixed-point} method, which computes the hyper-gradient via the fixed point method~\citep{grazzi2020iteration}~\replaced{and}{;} (3) \emph{Conjugate gradient} (CG) method, which solves the implicit linear system by the conjugate gradient method~\citep{grazzi2020iteration}.

We first conduct the experiment on a {\bf synthesis} generated dataset ($n+m=1000, d=20$) where half of the dataset is used for training and the rest for validation. \deleted{For the test bilevel algorithms, gradient descent is used as the optimizer for inner and outer problems. The details of all test methods are given in Appendix~\ref{append:logistic}. The results are reported in Figure~\ref{fig:syn}. The gradient calls mean the number of times the entire data has been called. We observe that the proposed method is the fastest to reduce the validation loss at the lowest level. }
\added{To make the comparison fair, we use gradient descent for all methods, which requires one full pass over the dataset. As observed in Figure~\ref{fig:syn}, MinimaxOPT significantly outperforms other methods with a much smaller number of gradient calls.  More experimental details are available in Appendix~\ref{append:logistic}.}
\begin{figure}[t]
     \centering
     \vskip -0.1in
\includegraphics[width=0.3\textwidth,height=1.6in]{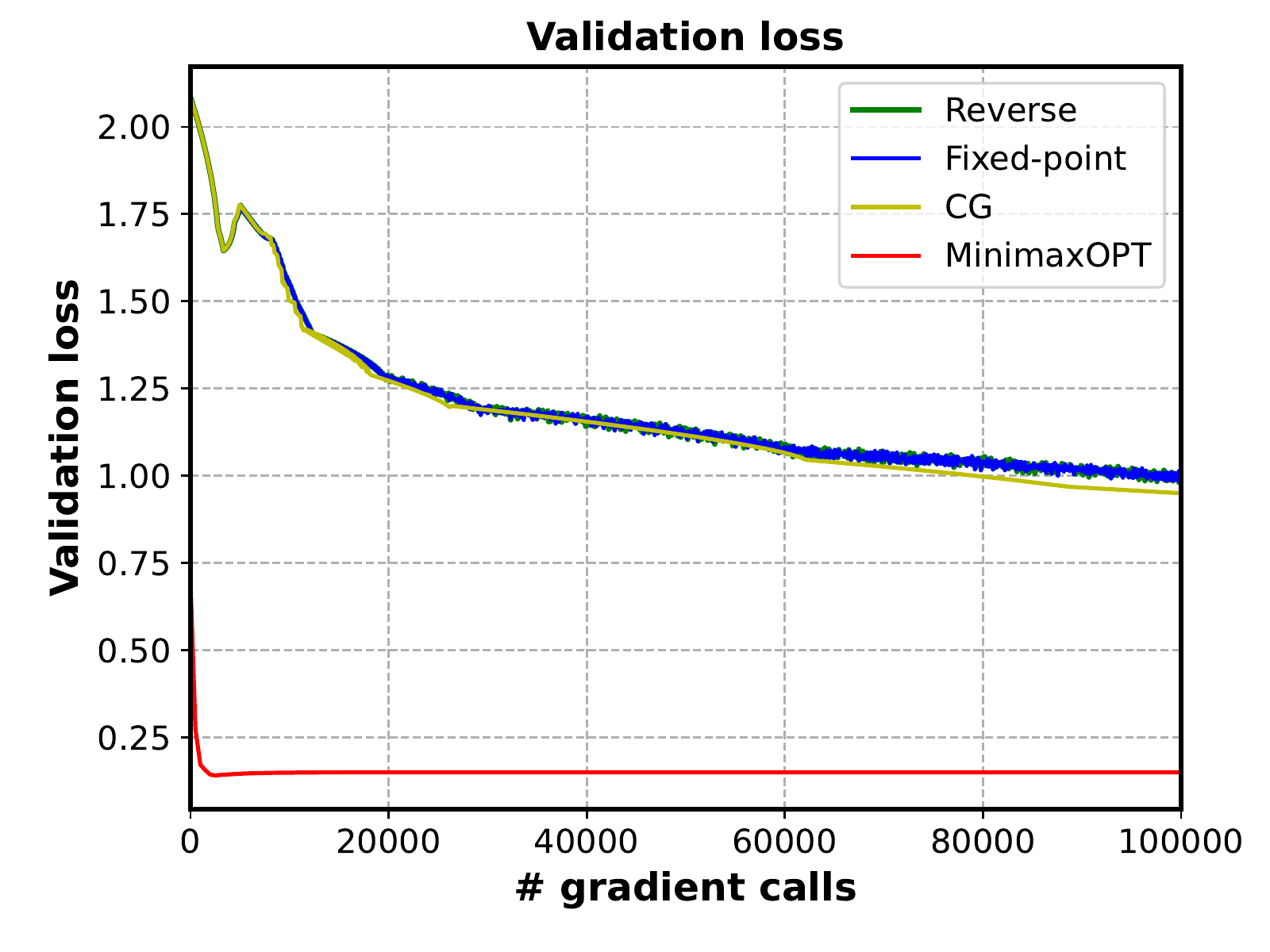}
\includegraphics[width=0.3\textwidth,height=1.6in]{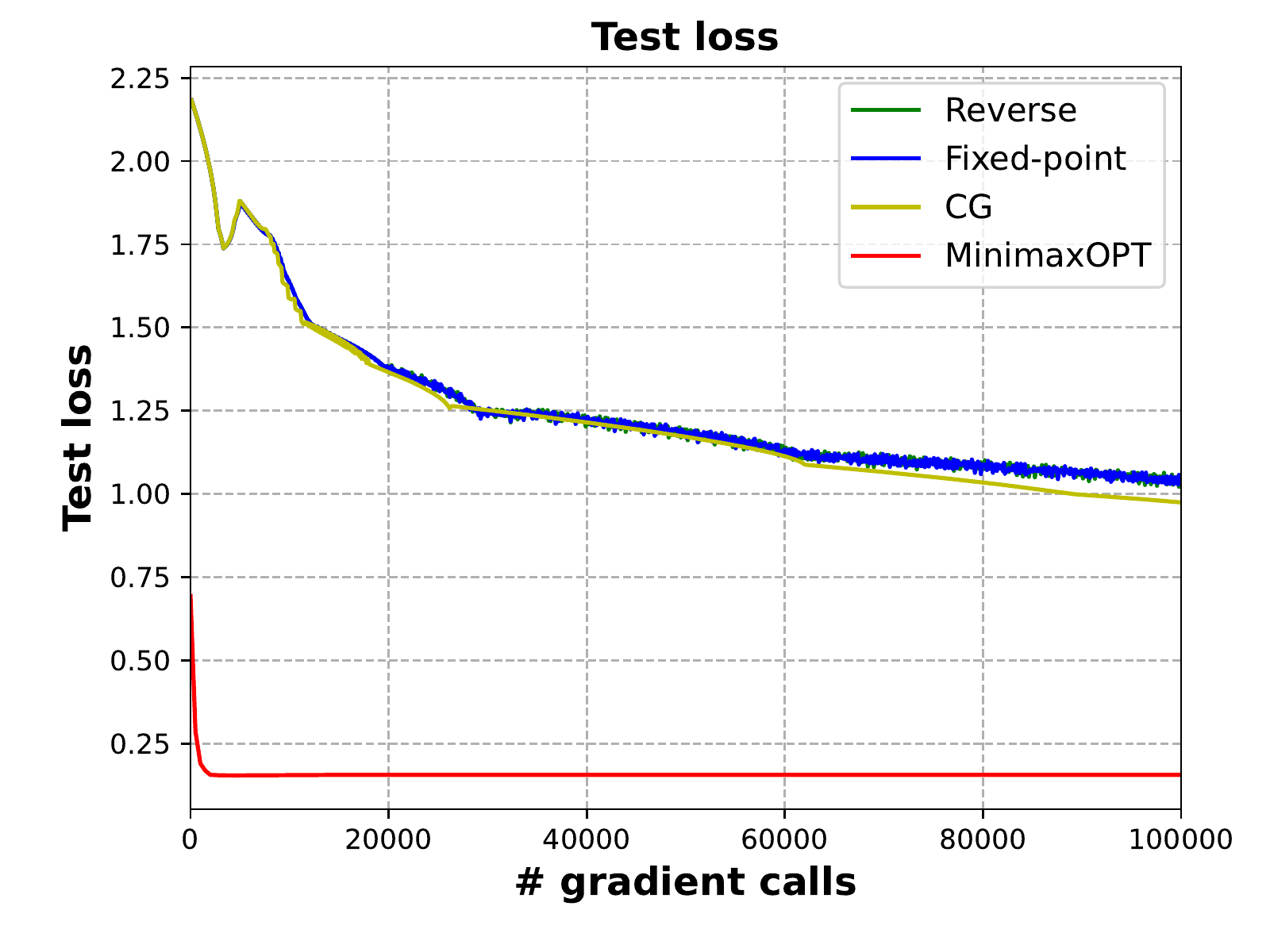}
\includegraphics[width=0.3\textwidth,height=1.6in]{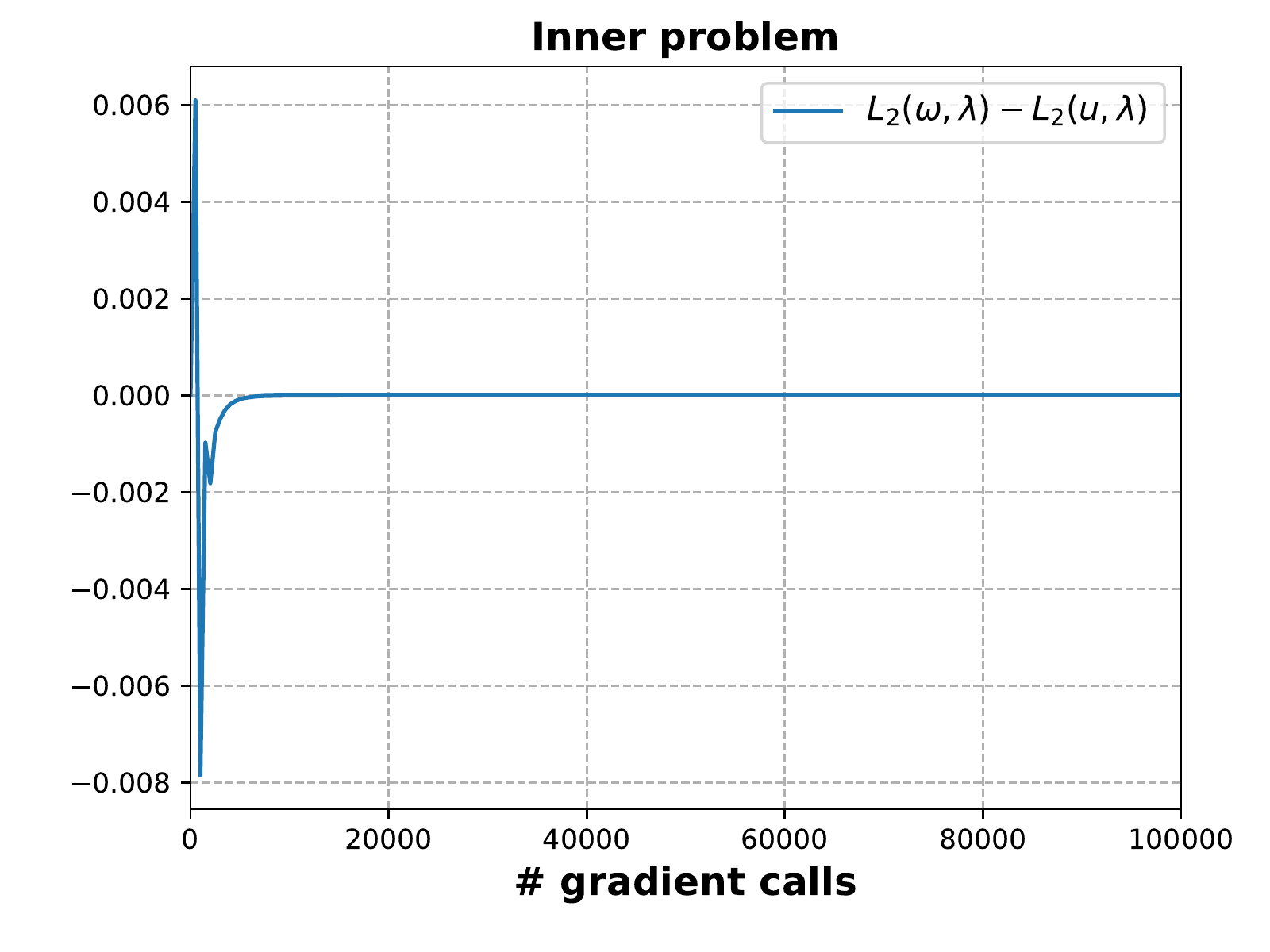}
\vskip -0.1in
 \caption{Hyper-parameter optimization results on a synthesis dataset}
\label{fig:syn}
\end{figure}

 

The next experiment is on a real dataset {\bf 20newsgroups}~\citep{Lang95}, which consists of 18846 news divided into 20 topics, and the features contain 130107 tf-idf sparse vectors. The full train dataset is equally split for training and validation. We follow the setting in \citep{grazzi2020iteration} and each feature is regularized by $\lambda_i \geq 0$, given $\lambda = [\lambda_i] \in \R^{d}$. To ensure this non-negativity, $\exp(\lambda_i) \geq 0$ is used in place of $\lambda_i$. Other settings remain similar to the previous synthesis dataset experiment, which is available in Appendix~\ref{append:logistic}.  The results in Figure~\ref{fig:news} show that the proposed minimax algorithm results in the best validation loss and achieves the highest test accuracy.

\deleted{In this experiment, the result of the reverse method is very close to the fixed-point method.}
\begin{figure}[t]
     \centering
     \vskip -0.1in
\includegraphics[width=0.3\textwidth,height=1.6in]{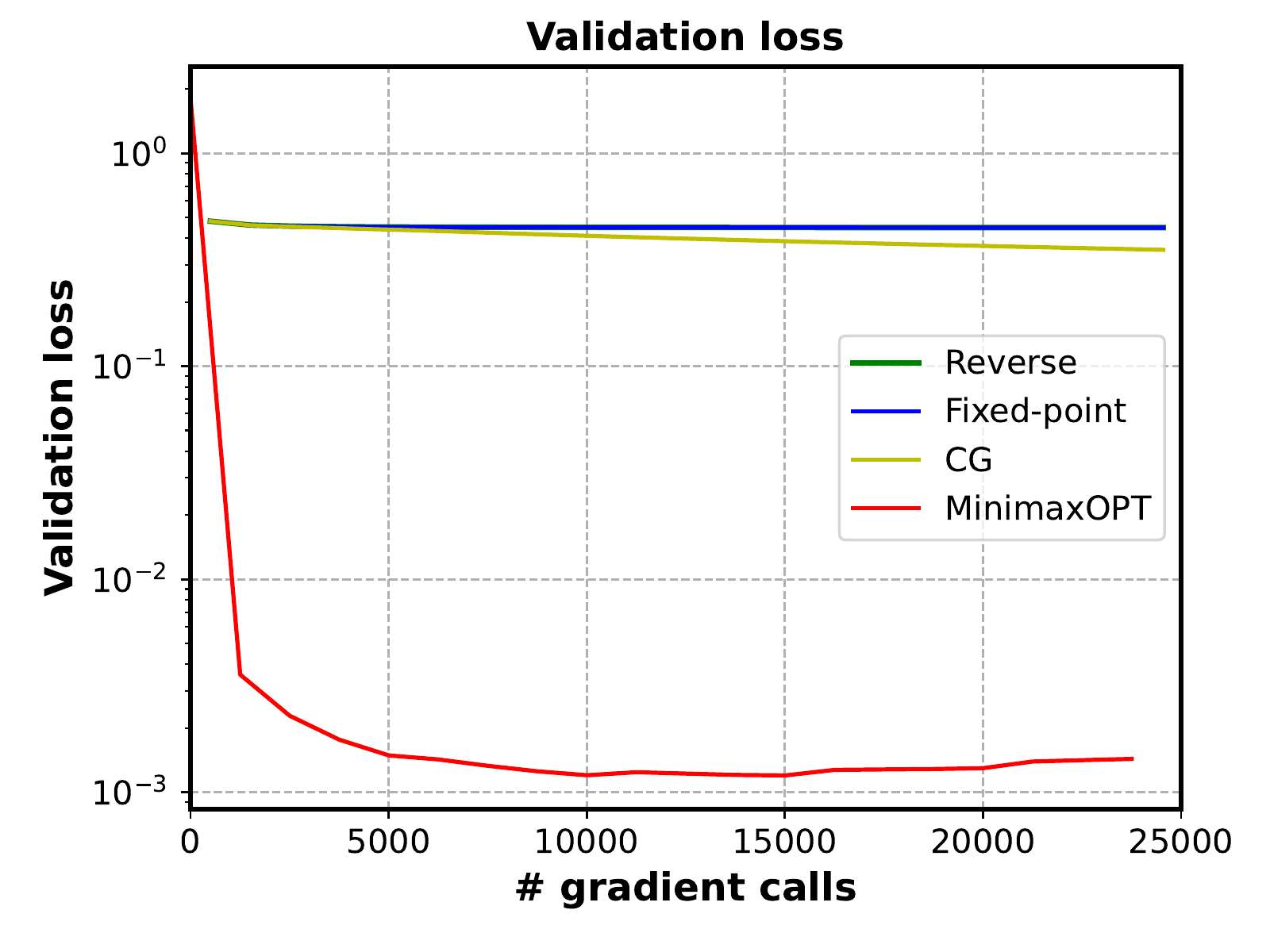}
\includegraphics[width=0.3\textwidth,height=1.6in]{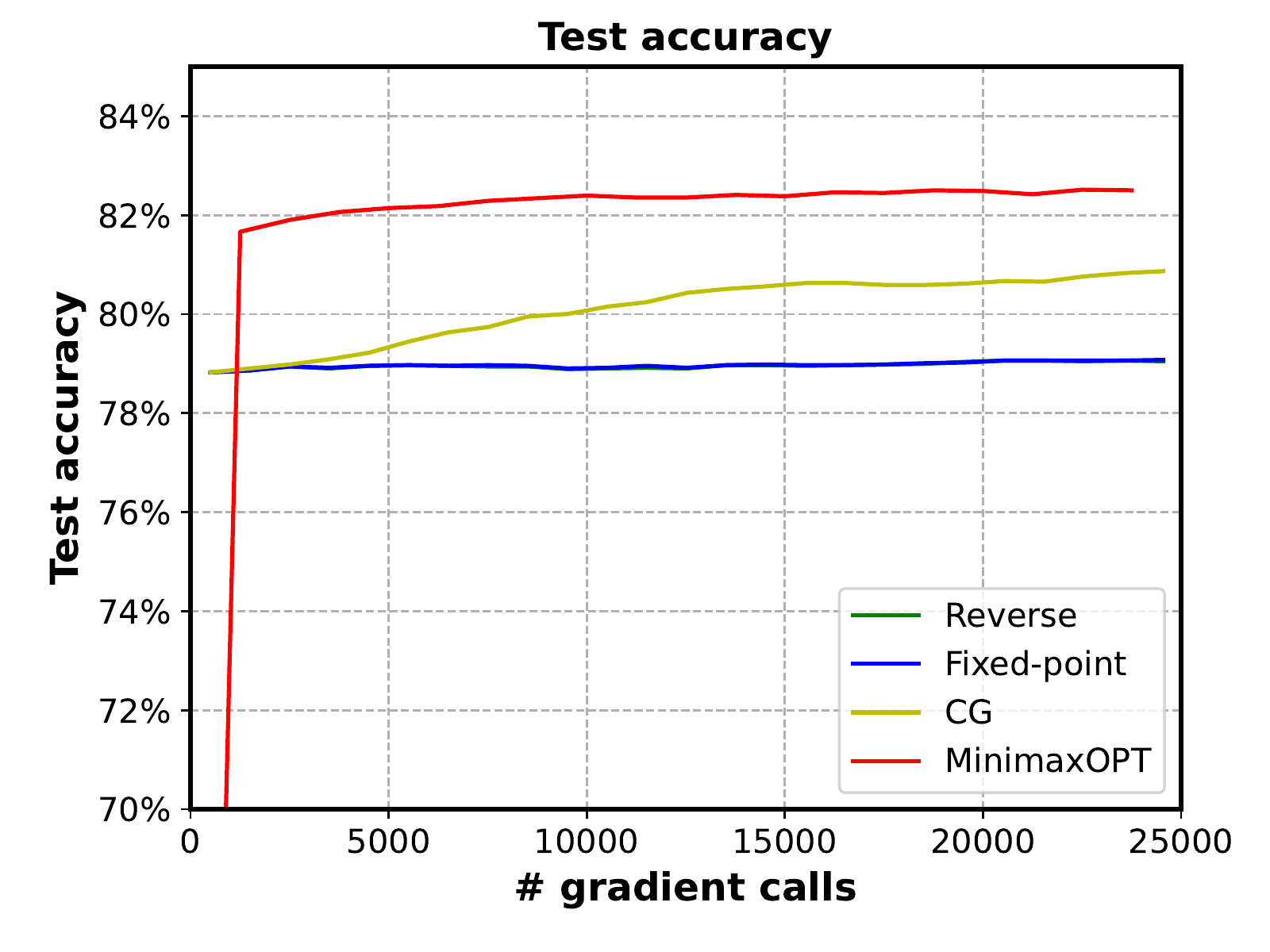}
\includegraphics[width=0.3\textwidth,height=1.6in]{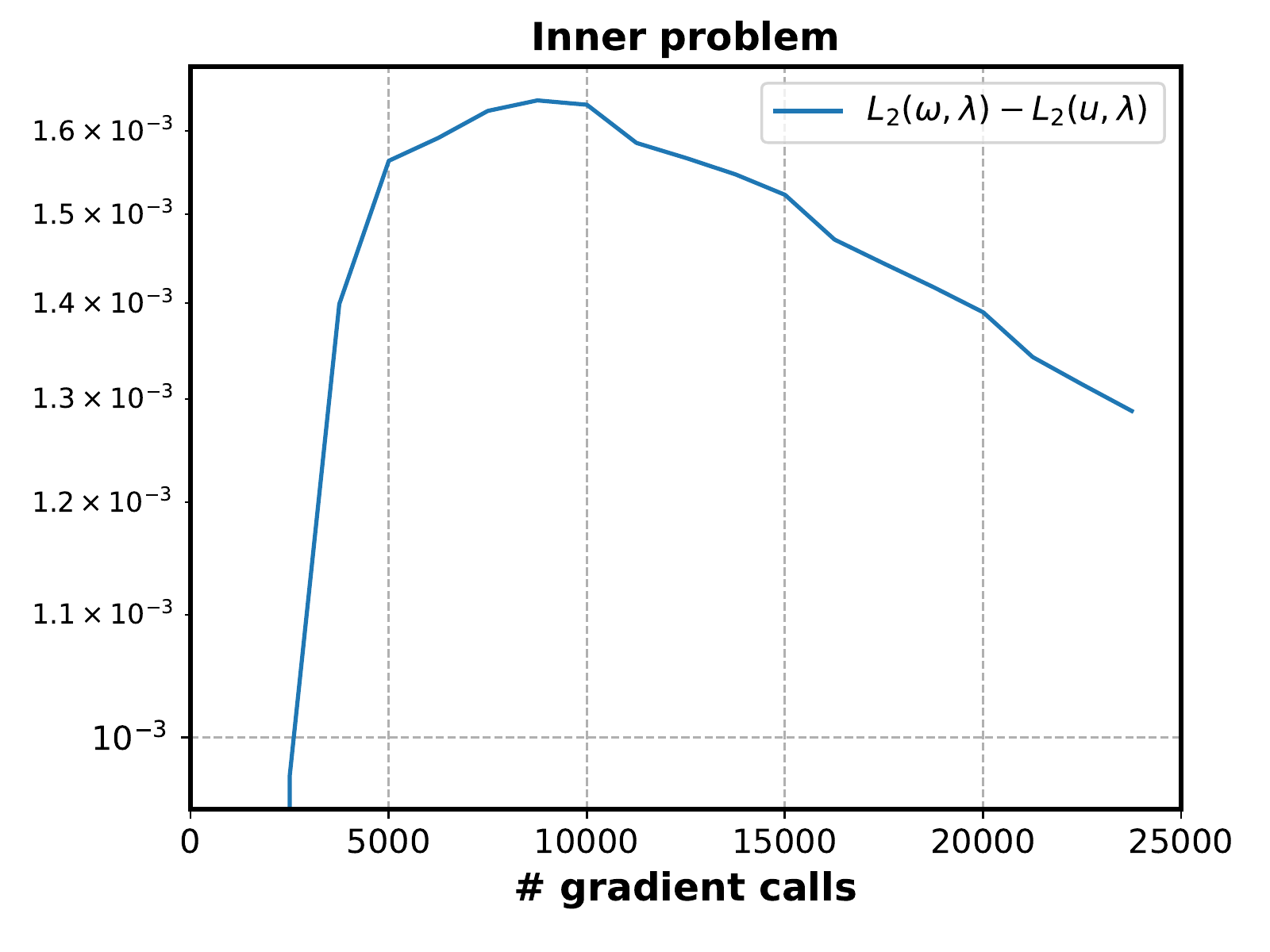}
\vskip -0.1in
 \caption{Hyper-parameter optimization results on 20newsgroups dataset}
\label{fig:news}
\end{figure}
\deleted{In Table~\ref{tab:logistic:time}, we report the total time of each algorithm. Our method saves around half the time compared to other bilevel methods that use almost the same number of gradients.}

\subsection{Deep Neural Networks with CIFAR10}
\label{sec:num:dnn}

\added{Next, we consider the task of training Resnet18~\citep{he2016deep} on CIFAR10~\citep{krizhevsky2009learning} for image classification.}
The entire training data is split by 0.9:0.1 for training and validation. We apply weight-decay per layer and initialize it with $10^{-10}$. The experiments are repeated three times under different seeds to eliminate randomness. \added{For more experimental details, please refer to Appendix~\ref{append:dnn}.}

\added{Two types of bilevel optimization baselines are adopted, one favors performance and runs 50 inner epochs for each outer iteration, while the other emphasizes efficiency and utilizes only 1 inner epoch. For each type, four different baselines are presented to compare with MinimaxOPT: (1) truncated reverse; (2) $T_1-T_2$ (also called one-step), which uses the identity matrix to approximate Hessian~\citep{luketina2016scalable}; (3) conjugate gradient (CG), a stochastic version which computes the Hessian matrix on a single minibatch and applies CG to compute the implicit linear system five times; (4) Neumann approximation: approximate the inverse of the second-order matrix with the Neumann series~\citep{lorraine2020optimizing}. Since those methods does not scale well, we shrink the model size of Resnet18 and use plane=16 instead of 64, so that all the aforementioned baselines can reach convergence within reasonable computational budget. }

\deleted{We implement the stochastic version of Algorithm \ref{alg:minmax:general} and use momentum to accelerate the convergence, shown in Algorithm \ref{alg:minmax:general} (in Appendix~\ref{append:dnn}). The cosine schedule is introduced into learning rate $\eta_k^i$: $\eta_k^i = \eta_0 /\tau^{i} \times 0.5 \times \left(\cos(\pi \times t/T) + 1\right)$ where $t = i \times K + k + 1$ and $T=KN$. We use the same batch size of 256 for both training and validation datasets. The parameter $\tau=1.5$. Other details about the experiments are given in Appendix~\ref{append:dnn}.}

\deleted{The test accuracy of each algorithm under different gradient oracles is reported in Figure~\ref{fig:cifar10} and Tables~\ref{tab:cifar10:1} and \ref{tab:cifar10:2}. Compared to the baselines where the inner problem converges, our method significantly reduces the computation cost and reaches a higher test accuracy. For the baselines that the inner problem is trained within one epoch, the proposed method can achieve higher test accuracy with the same gradient oracles.}

\added{As one can observe from Figure~\ref{fig:cifar10} and Table~\ref{tab:cifar10}, MinimaxOPT surpasses all baselines by a huge gap. It reaches the highest test accuracy with an order of magnitude speedup when compared with the second best method. This demonstrates that MinimaxOPT is not only algorithmically scalable but can also strike a good balance between performance and speed in medium-sized tasks.}

\begin{figure}[t]
\begin{minipage}{0.45\linewidth}
\begin{center}
\centerline{ \includegraphics[width=\textwidth,height=2in]{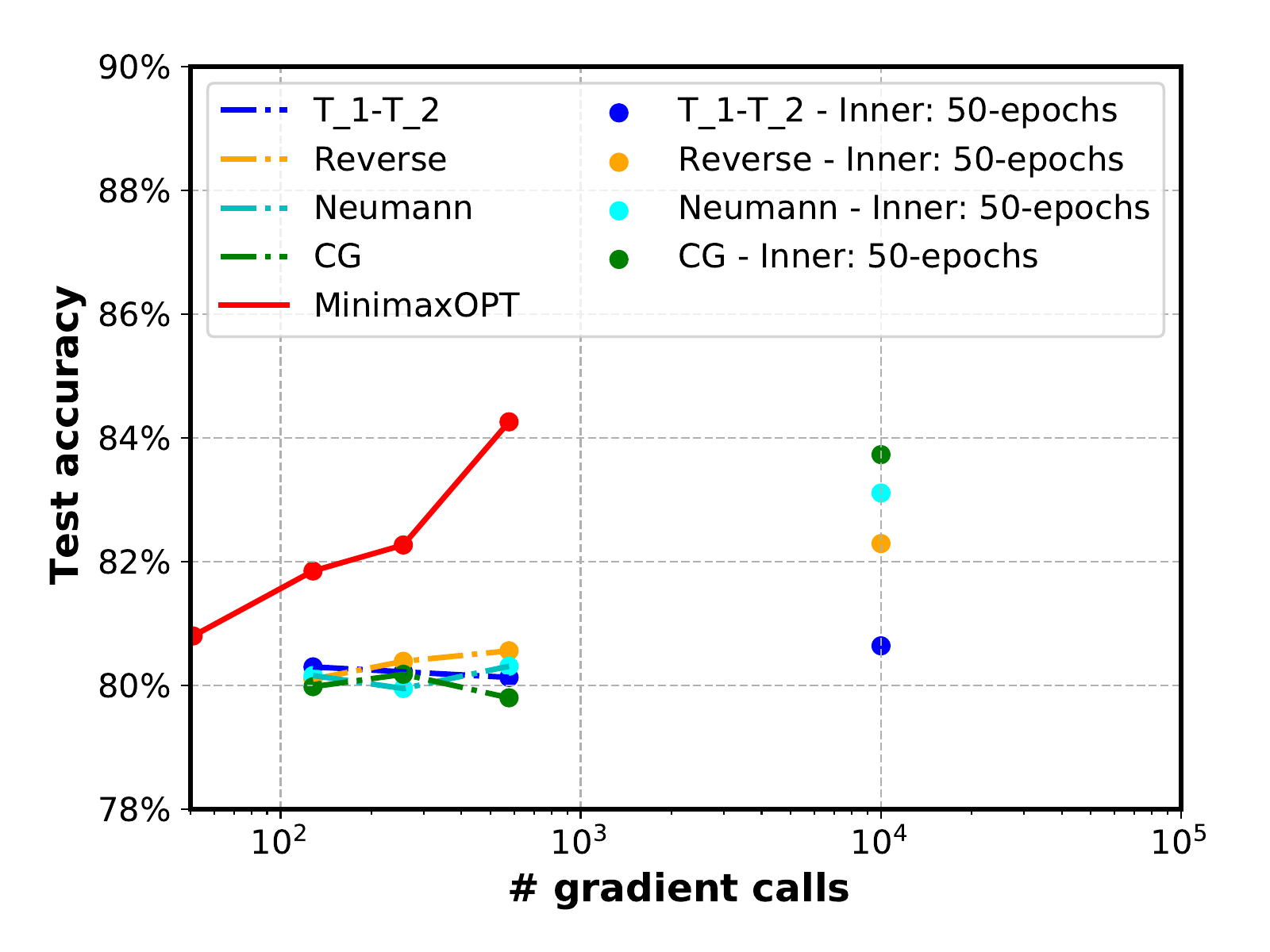}
}
\captionof{figure}{Test accuracy on CIFAR10}
\label{fig:cifar10} 
\end{center}
\end{minipage} 
\hfill
\begin{minipage}{0.48\linewidth}		
	\centering
		\captionof{table}{Test accuracy on CIFAR10.}
\label{tab:cifar10}
		\resizebox{\textwidth}{!}{%
		\begin{tabular}[width=0.9\linewidth]{@{}cccc@{}}
   \hline 
  Method &   \makecell{Inner\\epochs} &  Test accuracy & \# Gradient \\ \hline 
 $T_1-T_2$ & \multirow{4}{*}{1}  &   80.22 $\pm$ 0.97 & 256 \\
  Reverse &  &  80.39 $\pm$ 0.64  & 256 \\
   Neumann &      &   79.95 $\pm $ 1.16 & 256\\
CG &    &  80.18 $\pm$ 0.89 & 256\\ \midrule
 MinimaxOPT & - &  \textbf{82.27 }$\pm 1.416$  & 256 \\
    \bottomrule
  $T_1-T_2$ & \multirow{4}{*}{50}   & 80.67 $\pm$ 0.97 &  \multirow{4}{*}{10120}  \\ 
Reverse   &      &  82.29 $\pm$ 0.83 &  \\
 Neumann  &    &   83.11 $\pm$ 0.81 &   \\
   CG &    &  83.73 $\pm$ 0.68 \\  \midrule
 MinimaxOPT &  - &  \textbf{84.26} $\pm$ 1.56 & 576 \\
     \bottomrule
    \end{tabular}}
	\end{minipage}
\end{figure}

\subsection{Data Hyper-Cleaning on MNIST}\label{sec:num:hyper:clean}
\deleted{We evaluate the performance of the proposed minimax method on the data hyper-cleaning task. The task is to learn a classifier to clean the dataset with corrupted labels, formulated as the bilevel problem:}

\added{Data hyper-cleaning aims at cleaning the dataset with corrupted labels via reweighting,}
\begin{align*}
   \min_{\lambda} \quad  & \quad L_1(\lambda; S_{val}) = \sum_{j\in S_{val}}\ell(u^{\ast}(\lambda);\xi_j)\notag \\
  s.t., \quad  & \quad  u^{\ast}(\lambda) = \arg\min_{u} L_2(u, \lambda; S_{train}) =  \sum_{i \in S_{train}} \sigma(\lambda_i) \ell(u; \xi_{i}) + c \left\|u \right\|^2
\end{align*}
where $\sigma(\lambda_i):= \text{sigmoid}(\lambda_i)$, $\ell(x)$ is the loss function, $\lambda \in \R^m$ with $m = |S_{train}|$ and $c>0$ being a constant. We consider a subset of MNIST with 20000 examples for training, 5000 examples for validation, and 10000 examples for testing. The setting is similar to section 6 of \citep{ji2021bilevel} except for the model choice, where we use a non-linear model of two-layer neural networks with 0.2 dropouts instead of logistic regression. For the baseline, we compare three bilevel methods: (1) stocBiO~\citep{ji2021bilevel}, which is the stochastic bilevel method and uses Neumann series approximation to obtain sample-efficient hyper-gradient estimator; (2) truncated reverse method; (3) conjugate gradient (CG) method. The weight-decay parameter is fixed to be $c=0.001$. We sample both inner and outer problems by mini-batch for stocBiO and the stochastic version of MinimaxOPT. For reverse and CG methods, we employ gradient descent to optimize the inner and outer problems, as their stochastic counterparts are not easily accessible. The test accuracy of each method under different noise levels $p$ is presented in Figure~\ref{fig:hyper:noise} and the time of each algorithm to reach 90\% test accuracy is recorded in Table~\ref{tab:time:noise}. \deleted{Our minimax method is much faster than others to achieve the same relatively high accuracy.} \added{All those results demonstrate that MinimaxOPT is capable of reaching a relative high accuracy with much shorter time than others.} 

 \begin{figure}[t]
     \centering
     \vskip -0.1in
\includegraphics[width=0.3\textwidth,height=1.6in]{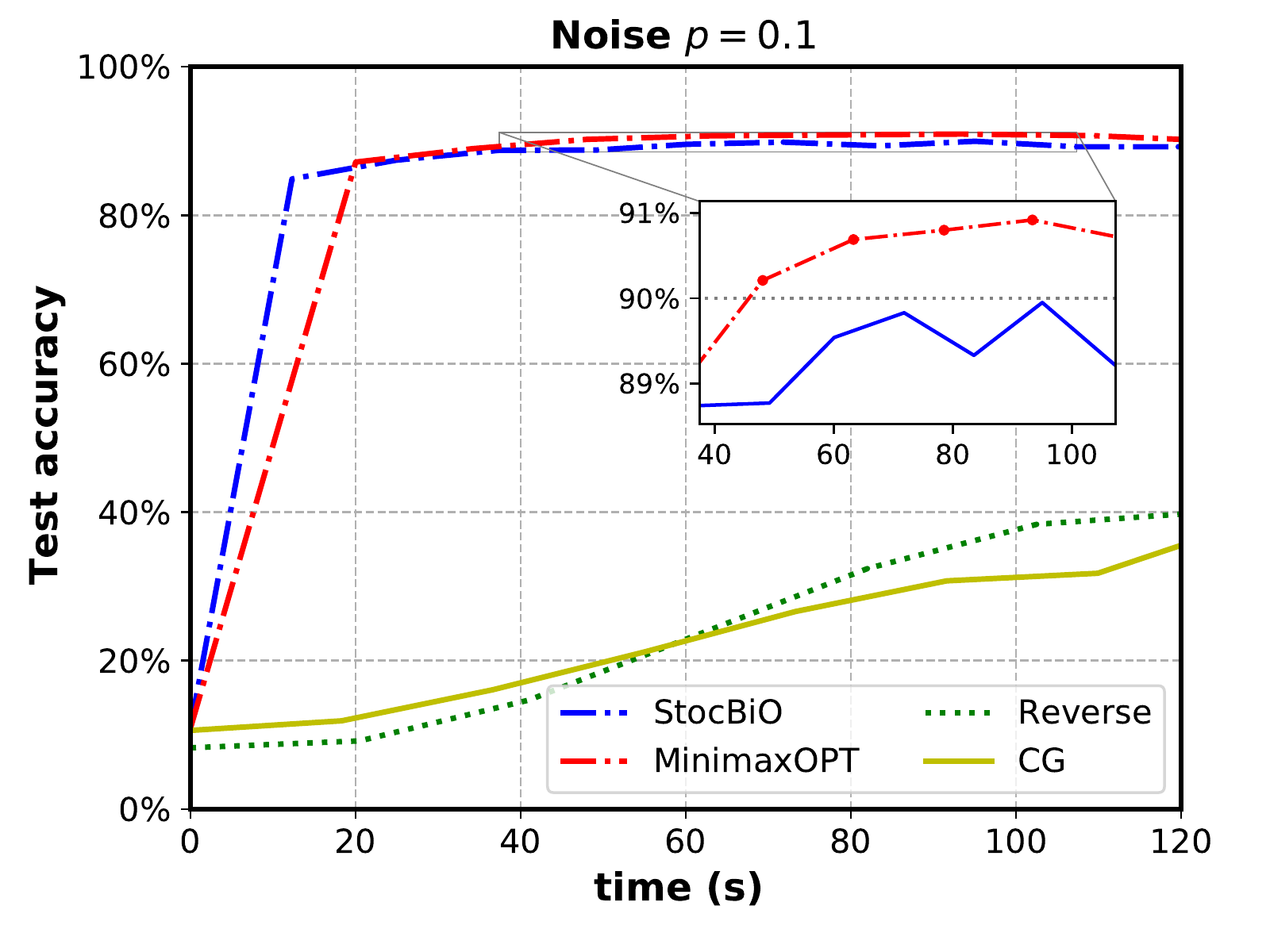}
\includegraphics[width=0.3\textwidth,height=1.6in]{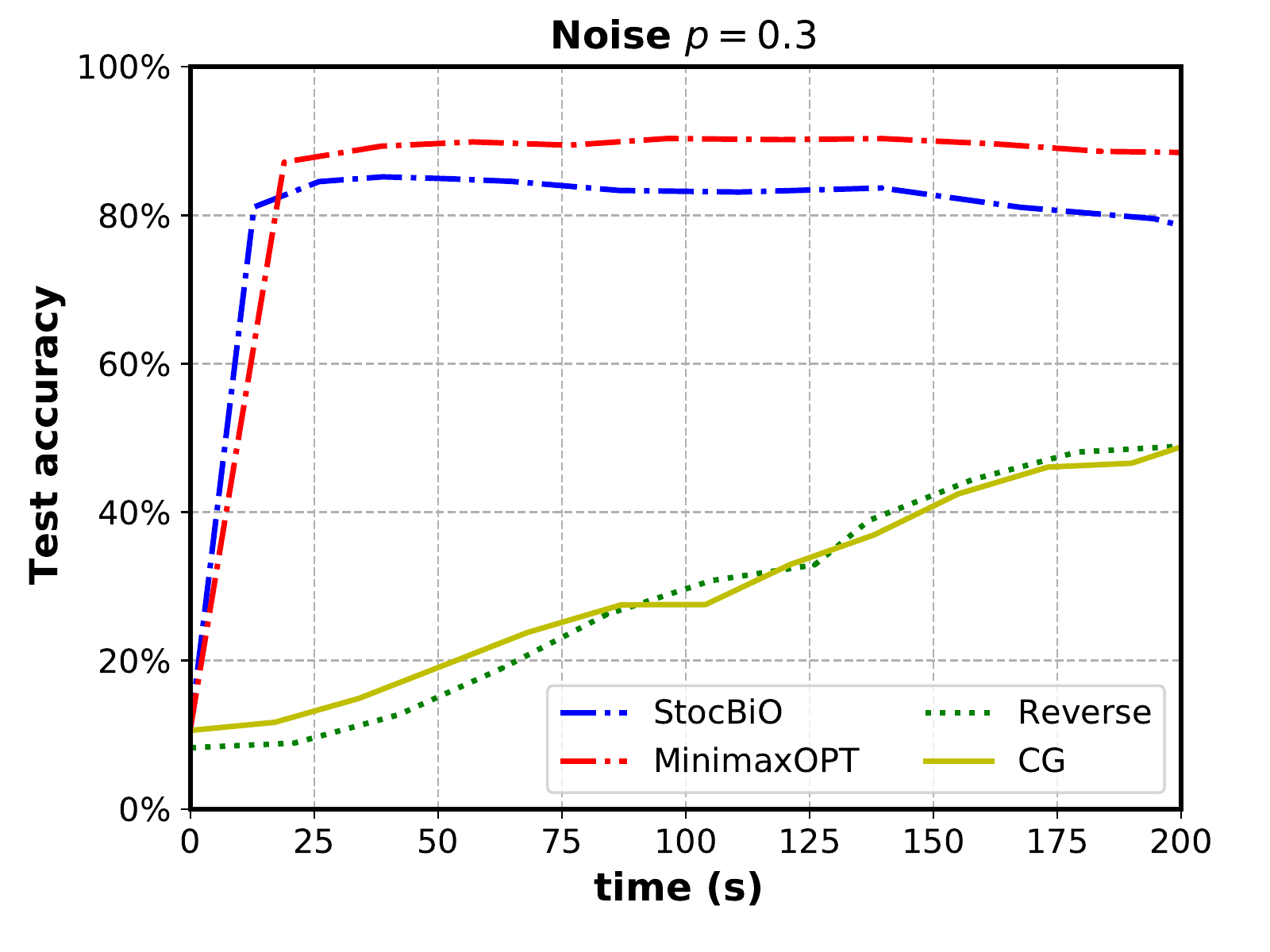}
\includegraphics[width=0.3\textwidth,height=1.6in]{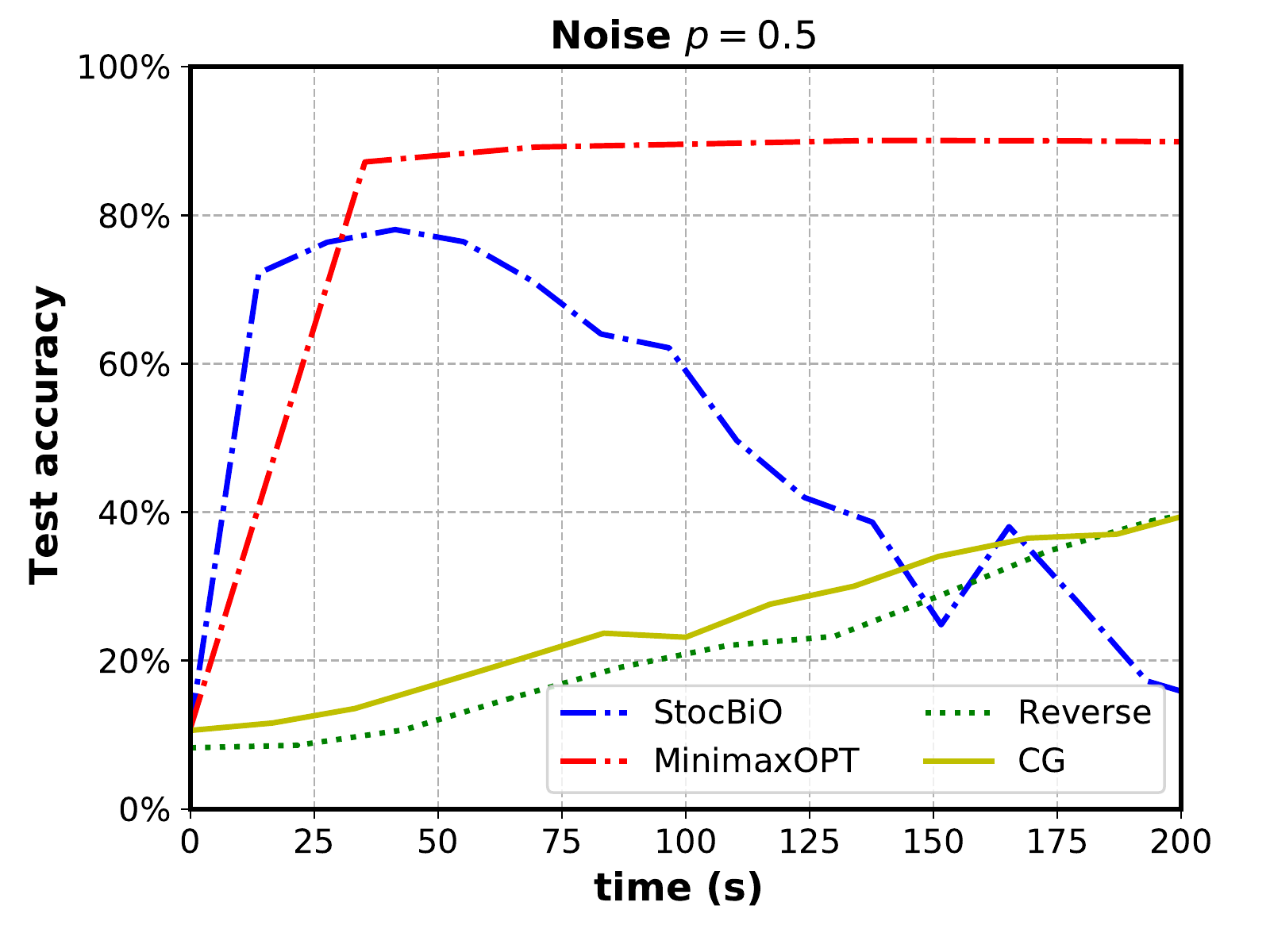}
\vskip -0.1in
 \caption{Data cleaning results on MNIST}
\label{fig:hyper:noise}
\end{figure}

\begin{table}[t]
  \centering
  \caption{Time of test accuracy reaching 90\%, ``-'' means the method fails to reach this accuracy} \label{tab:time:noise}
          \footnotesize
    \begin{tabular}{c|c|c|c|c}
     \toprule
        &  \multicolumn{3}{c}{Time (s)} \\ \midrule    
       Noise    & stocBiO &  MinimaxOPT & Reverse & CG\\ \midrule
  $p = 0.1$   &  95 &  48 & 3163 & 3428\\ \midrule 
      $p=0.3$  & --  & 96   & 4675      &  4113\\\midrule  
      $p=0.5$  &  -- &  137 & 7947 & 6084 \\ \midrule 
     \end{tabular}

\end{table}

\section{Conclusion}
\deleted{We proposed a novel prospect for approximating the bilevel problem via a minimax problem. We introduce an additional variable $\omega$ to decouple the dependence of the inner and outer problems. The gradient of the resulting minimax problem does not involve any second-order derivatives. We demonstrated that the proposed minimax problem is equivalent to the bilevel problem when the multiplier $\alpha$ goes to infinity. We proposed the multi-stage gradient descent and ascent (GDA) method to solve the minimax problem. The variables are updated synchronously, unlike the two-loop manner of the conventional bilevel methods. We provided an $\mathcal{O}(1/K)$ convergence rate for the one-stage GDA method and investigated the convergence properties for the multi-stage GDA method. Our numerical experiments demonstrated the superior performance of the proposed method compared to the baseline bilevel methods. We provided the empirical results of the stochastic version of the multi-stage GDA method. In the future, it will be very interesting to provide theoretical convergence guarantees for the stochastic multi-stage GDA method and explore the performance of the proposed problem and methods in other applications.}

\deleted{In this work, we propose a novel paradigm to efficiently solve general bilevel optimization problems. By converting bilevel optimization into equivalent minimax problems, we are able to decouple the infamous outer-inner dependency and open the door for more Hessian-free bilevel optimization algorithms to come. A multi-stage gradient descent algorithm named MinimaxOPT is proposed, which enjoys the same time/space complexity as gradient descent and can be painlessly equipped with SGD, SGD momentum, Adam, and many other first-order optimizers. Theoretically, besides the existing convergence results available in current minimax literature, an $\mathcal{O}(1/K)$ convergence rate is also provided for MinimaxOPT's one-stage setting, along with several additional convergence properties for its multi-stage setting. Empirically, MinimaxOPT surpasses existing bilevel bases by a huge gap and provides significant speedups. In the future, the theoretical guarantees for stochastic MinimaxOPT can be further investigated, along with its empirical performance being further verified in other settings.}

In this work, we propose a novel paradigm for efficiently solving general bilevel optimization problems. By converting bilevel optimization into equivalent minimax problems, we are capable of addressing the infamous outer-inner dependency issue, which opens up possibilities for more Hessian-free bilevel optimization algorithms. As a first step, we introduce MinimaxOPT, a multi-stage gradient descent ascent algorithm that shares the same time/space complexity as gradient descent. Algorithmically, MinimaxOPT can be easily equipped with first-order optimizers such as SGD, SGD with momentum, or Adam. Theoretically, MinimaxOPT enjoys convergence guarantees similar to those available in current minimax literature, and possesses further guarantees of $\mathcal{O}(\epsilon^{-6})$ convergence rate for its one-stage stochastic setting, along with other convergence properties for its multi-stage setting. Empirically, MinimaxOPT outperforms existing bilevel optimization baselines by a significant margin and provides substantial speedups. In the future, the theoretical guarantees for stochastic MinimaxOPT can be further investigated, along with its empirical performance being verified in other settings.


\bibliographystyle{plainnat}
\bibliography{ref_bilevel_minmax}

\begin{thebibliography}{45}
\providecommand{\natexlab}[1]{#1}
\providecommand{\url}[1]{\texttt{#1}}
\expandafter\ifx\csname urlstyle\endcsname\relax
  \providecommand{\doi}[1]{doi: #1}\else
  \providecommand{\doi}{doi: \begingroup \urlstyle{rm}\Url}\fi

\bibitem[Aghasi and Ghadimi(2024)]{aghasi2024fully}
Alireza Aghasi and Saeed Ghadimi.
\newblock Fully zeroth-order bilevel programming via gaussian smoothing.
\newblock \emph{arXiv preprint arXiv:2404.00158}, 2024.

\bibitem[Andrychowicz et~al.(2016)Andrychowicz, Denil, Gomez, Hoffman, Pfau, Schaul, Shillingford, and De~Freitas]{andrychowicz2016learning}
Marcin Andrychowicz, Misha Denil, Sergio Gomez, Matthew~W Hoffman, David Pfau, Tom Schaul, Brendan Shillingford, and Nando De~Freitas.
\newblock Learning to learn by gradient descent by gradient descent.
\newblock \emph{Advances in neural information processing systems}, 29, 2016.

\bibitem[Bengio(2000)]{bengio2000gradient}
Yoshua Bengio.
\newblock Gradient-based optimization of hyperparameters.
\newblock \emph{Neural computation}, 12\penalty0 (8):\penalty0 1889--1900, 2000.

\bibitem[Bergstra and Bengio(2012)]{bergstra2012random}
James Bergstra and Yoshua Bengio.
\newblock Random search for hyper-parameter optimization.
\newblock \emph{Journal of machine learning research}, 13\penalty0 (2), 2012.

\bibitem[Cai et~al.(2019)Cai, Zhu, and Han]{caiproxylessnas}
Han Cai, Ligeng Zhu, and Song Han.
\newblock {ProxylessNAS}: Direct neural architecture search on target task and hardware.
\newblock In \emph{International Conference on Learning Representations}, 2019.

\bibitem[Chen et~al.(2023)Chen, Ma, and Zhang]{chen2023near}
Lesi Chen, Yaohua Ma, and Jingzhao Zhang.
\newblock Near-optimal fully first-order algorithms for finding stationary points in bilevel optimization.
\newblock \emph{arXiv preprint arXiv:2306.14853}, 2023.

\bibitem[Chen et~al.(2022)Chen, Sun, Xiao, and Yin]{chen2022single}
Tianyi Chen, Yuejiao Sun, Quan Xiao, and Wotao Yin.
\newblock A single-timescale method for stochastic bilevel optimization.
\newblock In \emph{International Conference on Artificial Intelligence and Statistics}, pages 2466--2488. PMLR, 2022.

\bibitem[Domke(2012)]{pmlr-v22-domke12}
Justin Domke.
\newblock Generic methods for optimization-based modeling.
\newblock In Neil~D. Lawrence and Mark Girolami, editors, \emph{Proceedings of the Fifteenth International Conference on Artificial Intelligence and Statistics}, volume~22 of \emph{Proceedings of Machine Learning Research}, pages 318--326, La Palma, Canary Islands, 21--23 Apr 2012. PMLR.

\bibitem[Franceschi et~al.(2017)Franceschi, Donini, Frasconi, and Pontil]{iterative_der}
Luca Franceschi, Michele Donini, Paolo Frasconi, and Massimiliano Pontil.
\newblock Forward and reverse gradient-based hyperparameter optimization.
\newblock In Doina Precup and Yee~Whye Teh, editors, \emph{Proceedings of the 34th International Conference on Machine Learning}, volume~70 of \emph{Proceedings of Machine Learning Research}, pages 1165--1173. PMLR, 2017.

\bibitem[Franceschi et~al.(2018)Franceschi, Frasconi, Salzo, Grazzi, and Pontil]{franceschi2018bilevel}
Luca Franceschi, Paolo Frasconi, Saverio Salzo, Riccardo Grazzi, and Massimiliano Pontil.
\newblock Bilevel programming for hyperparameter optimization and meta-learning.
\newblock In \emph{International Conference on Machine Learning}, pages 1568--1577. PMLR, 2018.

\bibitem[Gao et~al.(2023)Gao, Pi, Yong, Xu, Ye, Wu, ZHANG, Liang, Li, and Kong]{gaoself}
Jiahui Gao, Renjie Pi, LIN Yong, Hang Xu, Jiacheng Ye, Zhiyong Wu, WEIZHONG ZHANG, Xiaodan Liang, Zhenguo Li, and Lingpeng Kong.
\newblock Self-guided noise-free data generation for efficient zero-shot learning.
\newblock In \emph{The Eleventh International Conference on Learning Representations}, 2023.

\bibitem[Ghadimi and Wang(2018)]{ghadimi2018approximation}
Saeed Ghadimi and Mengdi Wang.
\newblock Approximation methods for bilevel programming.
\newblock \emph{arXiv preprint arXiv:1802.02246}, 2018.

\bibitem[Grazzi et~al.(2020)Grazzi, Franceschi, Pontil, and Salzo]{grazzi2020iteration}
Riccardo Grazzi, Luca Franceschi, Massimiliano Pontil, and Saverio Salzo.
\newblock On the iteration complexity of hypergradient computation.
\newblock In \emph{International Conference on Machine Learning}, pages 3748--3758. PMLR, 2020.

\bibitem[He et~al.(2016)He, Zhang, Ren, and Sun]{he2016deep}
Kaiming He, Xiangyu Zhang, Shaoqing Ren, and Jian Sun.
\newblock Deep residual learning for image recognition.
\newblock In \emph{Proceedings of the IEEE conference on computer vision and pattern recognition}, pages 770--778, 2016.

\bibitem[Hong et~al.(2020)Hong, Wai, Wang, and Yang]{hong2020two}
Mingyi Hong, Hoi-To Wai, Zhaoran Wang, and Zhuoran Yang.
\newblock A two-timescale framework for bilevel optimization: Complexity analysis and application to actor-critic.
\newblock \emph{arXiv preprint arXiv:2007.05170}, 2020.

\bibitem[Ji et~al.(2021)Ji, Yang, and Liang]{ji2021bilevel}
Kaiyi Ji, Junjie Yang, and Yingbin Liang.
\newblock Bilevel optimization: Convergence analysis and enhanced design.
\newblock In \emph{International conference on machine learning}, pages 4882--4892. PMLR, 2021.

\bibitem[Khanduri et~al.(2021)Khanduri, Zeng, Hong, Wai, Wang, and Yang]{khanduri2021near}
Prashant Khanduri, Siliang Zeng, Mingyi Hong, Hoi-To Wai, Zhaoran Wang, and Zhuoran Yang.
\newblock A near-optimal algorithm for stochastic bilevel optimization via double-momentum.
\newblock \emph{Advances in neural information processing systems}, 34:\penalty0 30271--30283, 2021.

\bibitem[Kingma and Ba(2015)]{Adam}
Diederik~P Kingma and Jimmy~Lei Ba.
\newblock {ADAM}: A method for stochastic optimization.
\newblock In \emph{International Conference on Learning Representations}, 2015.

\bibitem[Konda and Tsitsiklis(1999)]{konda1999actor}
Vijay Konda and John Tsitsiklis.
\newblock Actor-critic algorithms.
\newblock \emph{Advances in neural information processing systems}, 12, 1999.

\bibitem[Krizhevsky et~al.(2009)Krizhevsky, Hinton, et~al.]{krizhevsky2009learning}
Alex Krizhevsky, Geoffrey Hinton, et~al.
\newblock Learning multiple layers of features from tiny images.
\newblock \emph{Technical Report}, 2009.

\bibitem[Kwon et~al.(2023{\natexlab{a}})Kwon, Kwon, Wright, and Nowak]{kwon2023fully}
Jeongyeol Kwon, Dohyun Kwon, Stephen Wright, and Robert~D Nowak.
\newblock A fully first-order method for stochastic bilevel optimization.
\newblock In \emph{International Conference on Machine Learning}, pages 18083--18113. PMLR, 2023{\natexlab{a}}.

\bibitem[Kwon et~al.(2023{\natexlab{b}})Kwon, Kwon, Wright, and Nowak]{pmlr-v202-kwon23c}
Jeongyeol Kwon, Dohyun Kwon, Stephen Wright, and Robert~D Nowak.
\newblock A fully first-order method for stochastic bilevel optimization.
\newblock In Andreas Krause, Emma Brunskill, Kyunghyun Cho, Barbara Engelhardt, Sivan Sabato, and Jonathan Scarlett, editors, \emph{Proceedings of the 40th International Conference on Machine Learning}, volume 202 of \emph{Proceedings of Machine Learning Research}, pages 18083--18113. PMLR, 23--29 Jul 2023{\natexlab{b}}.

\bibitem[Lang(1995)]{Lang95}
Ken Lang.
\newblock Newsweeder: Learning to filter netnews.
\newblock In \emph{Proceedings of the Twelfth International Conference on Machine Learning}, pages 331--339, 1995.

\bibitem[Liu et~al.(2019)Liu, Simonyan, and Yang]{liudarts}
Hanxiao Liu, Karen Simonyan, and Yiming Yang.
\newblock {Darts}: Differentiable architecture search.
\newblock In \emph{International Conference on Learning Representations}, 2019.

\bibitem[Lorraine et~al.(2020)Lorraine, Vicol, and Duvenaud]{lorraine2020optimizing}
Jonathan Lorraine, Paul Vicol, and David Duvenaud.
\newblock Optimizing millions of hyperparameters by implicit differentiation.
\newblock In \emph{International Conference on Artificial Intelligence and Statistics}, pages 1540--1552. PMLR, 2020.

\bibitem[Loshchilov and Hutter(2017)]{loshchilovsgdr}
Ilya Loshchilov and Frank Hutter.
\newblock {SGDR}: Stochastic gradient descent with warm restarts.
\newblock In \emph{International Conference on Learning Representations}, 2017.

\bibitem[Luketina et~al.(2016)Luketina, Berglund, Greff, and Raiko]{luketina2016scalable}
Jelena Luketina, Mathias Berglund, Klaus Greff, and Tapani Raiko.
\newblock Scalable gradient-based tuning of continuous regularization hyperparameters.
\newblock In \emph{International conference on machine learning}, pages 2952--2960. PMLR, 2016.

\bibitem[Mackay et~al.(2019)Mackay, Vicol, Lorraine, Duvenaud, and Grosse]{mackayself}
Matthew Mackay, Paul Vicol, Jonathan Lorraine, David Duvenaud, and Roger Grosse.
\newblock Self-tuning networks: Bilevel optimization of hyperparameters using structured best-response functions.
\newblock In \emph{International Conference on Learning Representations}, 2019.

\bibitem[Maclaurin et~al.(2015)Maclaurin, Duvenaud, and Adams]{pmlr-v37-maclaurin15}
Dougal Maclaurin, David Duvenaud, and Ryan Adams.
\newblock Gradient-based hyperparameter optimization through reversible learning.
\newblock In Francis Bach and David Blei, editors, \emph{Proceedings of the 32nd International Conference on Machine Learning}, volume~37 of \emph{Proceedings of Machine Learning Research}, pages 2113--2122, Lille, France, 07--09 Jul 2015. PMLR.

\bibitem[Mehra and Hamm(2021)]{mehra2021penalty}
Akshay Mehra and Jihun Hamm.
\newblock Penalty method for inversion-free deep bilevel optimization.
\newblock In \emph{Asian Conference on Machine Learning}, pages 347--362. PMLR, 2021.

\bibitem[Paszke et~al.(2019)Paszke, Gross, Massa, Lerer, Bradbury, Chanan, Killeen, Lin, Gimelshein, Antiga, et~al.]{paszke2019pytorch}
Adam Paszke, Sam Gross, Francisco Massa, Adam Lerer, James Bradbury, Gregory Chanan, Trevor Killeen, Zeming Lin, Natalia Gimelshein, Luca Antiga, et~al.
\newblock Pytorch: An imperative style, high-performance deep learning library.
\newblock \emph{Advances in neural information processing systems}, 32, 2019.

\bibitem[Pedregosa(2016)]{pmlr-v48-pedregosa16}
Fabian Pedregosa.
\newblock Hyperparameter optimization with approximate gradient.
\newblock In Maria~Florina Balcan and Kilian~Q. Weinberger, editors, \emph{Proceedings of The 33rd International Conference on Machine Learning}, volume~48 of \emph{Proceedings of Machine Learning Research}, pages 737--746, New York, New York, USA, 20--22 Jun 2016. PMLR.

\bibitem[Radford et~al.(2019)Radford, Wu, Child, Luan, Amodei, Sutskever, et~al.]{radford2019language}
Alec Radford, Jeffrey Wu, Rewon Child, David Luan, Dario Amodei, Ilya Sutskever, et~al.
\newblock Language models are unsupervised multitask learners.
\newblock \emph{OpenAI blog}, 1\penalty0 (8):\penalty0 9, 2019.

\bibitem[Rajeswaran et~al.(2019)Rajeswaran, Finn, Kakade, and Levine]{rajeswaran2019meta}
Aravind Rajeswaran, Chelsea Finn, Sham~M Kakade, and Sergey Levine.
\newblock Meta-learning with implicit gradients.
\newblock \emph{Advances in neural information processing systems}, 32, 2019.

\bibitem[Ren et~al.(2018)Ren, Zeng, Yang, and Urtasun]{ren2018learning}
Mengye Ren, Wenyuan Zeng, Bin Yang, and Raquel Urtasun.
\newblock Learning to reweight examples for robust deep learning.
\newblock In \emph{International conference on machine learning}, pages 4334--4343. PMLR, 2018.

\bibitem[Shaban et~al.(2019)Shaban, Cheng, Hatch, and Boots]{shaban2019truncated}
Amirreza Shaban, Ching-An Cheng, Nathan Hatch, and Byron Boots.
\newblock Truncated back-propagation for bilevel optimization.
\newblock In \emph{The 22nd International Conference on Artificial Intelligence and Statistics}, pages 1723--1732. PMLR, 2019.

\bibitem[Shen and Chen(2023)]{pmlr-v202-shen23c}
Han Shen and Tianyi Chen.
\newblock On penalty-based bilevel gradient descent method.
\newblock In \emph{Proceedings of the 40th International Conference on Machine Learning}, volume 202 of \emph{Proceedings of Machine Learning Research}, pages 30992--31015. PMLR, 23--29 Jul 2023.

\bibitem[Shi et~al.(2020)Shi, Pi, Xu, Li, Kwok, and Zhang]{shi2020bridging}
Han Shi, Renjie Pi, Hang Xu, Zhenguo Li, James Kwok, and Tong Zhang.
\newblock Bridging the gap between sample-based and one-shot neural architecture search with bonas.
\newblock \emph{Advances in Neural Information Processing Systems}, 33:\penalty0 1808--1819, 2020.

\bibitem[Shu et~al.(2019)Shu, Xie, Yi, Zhao, Zhou, Xu, and Meng]{shu2019meta}
Jun Shu, Qi~Xie, Lixuan Yi, Qian Zhao, Sanping Zhou, Zongben Xu, and Deyu Meng.
\newblock Meta-weight-net: Learning an explicit mapping for sample weighting.
\newblock \emph{Advances in neural information processing systems}, 32, 2019.

\bibitem[Sow et~al.(2022)Sow, Ji, and Liang]{NEURIPS2022_1a82986c}
Daouda Sow, Kaiyi Ji, and Yingbin Liang.
\newblock On the convergence theory for hessian-free bilevel algorithms.
\newblock In S.~Koyejo, S.~Mohamed, A.~Agarwal, D.~Belgrave, K.~Cho, and A.~Oh, editors, \emph{Advances in Neural Information Processing Systems}, volume~35, pages 4136--4149. Curran Associates, Inc., 2022.

\bibitem[Touvron et~al.(2023)Touvron, Lavril, Izacard, Martinet, Lachaux, Lacroix, Rozi{\`e}re, Goyal, Hambro, Azhar, et~al.]{touvron2023llama}
Hugo Touvron, Thibaut Lavril, Gautier Izacard, Xavier Martinet, Marie-Anne Lachaux, Timoth{\'e}e Lacroix, Baptiste Rozi{\`e}re, Naman Goyal, Eric Hambro, Faisal Azhar, et~al.
\newblock {Llama}: Open and efficient foundation language models.
\newblock \emph{arXiv preprint arXiv:2302.13971}, 2023.

\bibitem[White et~al.(2021)White, Neiswanger, and Savani]{white2021bananas}
Colin White, Willie Neiswanger, and Yash Savani.
\newblock Bananas: Bayesian optimization with neural architectures for neural architecture search.
\newblock In \emph{Proceedings of the AAAI Conference on Artificial Intelligence}, volume~35, pages 10293--10301, 2021.

\bibitem[Xu et~al.(2019)Xu, Xie, Zhang, Chen, Qi, Tian, and Xiong]{xu2019pc}
Yuhui Xu, Lingxi Xie, Xiaopeng Zhang, Xin Chen, Guo-Jun Qi, Qi~Tian, and Hongkai Xiong.
\newblock {PC-DARTS}: Partial channel connections for memory-efficient architecture search.
\newblock In \emph{International Conference on Learning Representations}, 2019.

\bibitem[Yang et~al.(2023)Yang, Xiao, and Ji]{yang2023achieving}
Yifan Yang, Peiyao Xiao, and Kaiyi Ji.
\newblock Achieving $o(\epsilon^{-1.5})$ complexity in hessian/jacobian-free stochastic bilevel optimization.
\newblock \emph{arXiv preprint arXiv:2312.03807}, 2023.

\bibitem[Yong et~al.(2023)Yong, Pi, Zhang, Xia, Gao, Zhou, Liu, and Han]{yongholistic}
LIN Yong, Renjie Pi, Weizhong Zhang, Xiaobo Xia, Jiahui Gao, Xiao Zhou, Tongliang Liu, and Bo~Han.
\newblock A holistic view of label noise transition matrix in deep learning and beyond.
\newblock In \emph{The Eleventh International Conference on Learning Representations}, 2023.

\end{thebibliography}

\newpage

\appendix


\section{Proof of Theorem~\ref{thm:equivalent:b:m}}
\thmequiv*
\begin{proof} \deleted{(of Theorem \ref{thm:equivalent:b:m})}
\deleted{For any fixed $\alpha $, in order to minimize $L^{\alpha}(u, \omega, \lambda)$, then there must exist a constant $M_2 > 0$ such that $L_2(\hat{\omega}, \hat{\lambda}) - L_2(\hat{u}, \hat{\lambda}) \leq \frac{M_2}{\alpha}$. Otherwise, $\alpha \left(L_2(\hat{\omega}, \hat{\lambda}) -L_2(\hat{u}, \hat{\lambda}) \right) \rightarrow \infty$ when $ \alpha \rightarrow \infty$.}

\textbf{For claim (1)}:
Since $\hw, \hlambda$ is the minimizer of 
\begin{align*}
  L'(\omega, \lambda)
  &\triangleq \max_u L^{\alpha}(u, \omega, \lambda)
  \\
  &= \max_u \left( L_1(\omega, \lambda) + \alpha \left(L_2(\omega, \lambda) - L_2(u, \lambda)\right) \right)
  \\
  &= L_1(\omega, \lambda) + \alpha \left(L_2(\omega, \lambda) - \min_u L_2(u, \lambda)\right)
\end{align*}
it satisfies
\begin{align*}
  L^{\alpha}(\hu, \hw, \hlambda)
  =
  L'(\hw, \hlambda)
  &\le
  L'(\ustar, \lambdastar)
  \\
  &=
  L_1(\ustar, \lambdastar) + \alpha \left(L_2(\ustar, \lambdastar) - \min_u L_2(u, \lambdastar)\right)
  \\
  &= L_1(\ustar, \lambdastar) = L_1^{\ast}.
\end{align*}
Thus
\begin{align*}
 &L_1(\hw, \hlambda) + \alpha (L_2(\hw, \hlambda) - L_2(\hu, \hlambda)) = L^{\alpha}(\hu, \hw, \hlambda) \le L_1^{\ast}
 \\
 \Rightarrow\quad &
 L_2(\hw, \hlambda) - L_2(\hu, \hlambda)
 \le \frac{L_1^{\ast} - L_1(\hw, \hlambda)}{\alpha}
 \le \frac{L_1^{\ast}}{\alpha}
\end{align*}
On the other hand, according to $\hu$'s optimality in minimax, we have
\begin{equation}
\label{eq:u:minmax:optimality}
\begin{aligned}
\hu &= \arg\max_u L^{\alpha}(u, \hw, \hlambda) = \arg\max_u L_1(\hw, \hlambda) + \alpha (L_2(\hw, \hlambda) - L_2(u, \hlambda))
\\
&= \arg\min L_2(u, \hlambda)
\end{aligned}
\end{equation}
Hence,
\begin{align*}
  L_2(\hw, \hlambda) - L_2(\hu, \hlambda)
  = L_2(\hw, \hlambda) - \min_u L_2(u, \hlambda)
  \ge 0.
\end{align*}

\textbf{For claim (2)}: by Lipschtiz continuity of $L_1(\cdot)$ we have
\begin{align}\label{inequ:lip:l1}
|L_1(\omega, \lambda) - L_1(u, \lambda)| \leq \hat{L}_1 \left\| \omega - u \right\|.
\end{align}
Due to (1) that $0 \le L_2(\hat{\omega}, \hat{\lambda}) - L_2(\hat{u}, \hat{\lambda}) \leq L_1^{\ast}/\alpha$ and the assumption $| L_2(\hat{\omega}, \hat{\lambda}) - L_2(\hat{u}, \hat{\lambda})| \leq \Delta \Rightarrow \left\| \hat{\omega} - \hat{u} \right\|^r \leq M \Delta$, we have $\left\| \hat{\omega} - \hat{u} \right\| \leq \left(\frac{M L_1^{\ast}}{\alpha}\right)^{1/r}$. Then using (\ref{inequ:lip:l1}), we have
\allowdisplaybreaks
\begin{align}
L_1(\hat{\omega}, \hat{\lambda}) & \geq L_1(\hat{u}, \hat{\lambda}) -  \hat{L}_1 \cdot \left(\frac{ M L_1^{\ast}}{\alpha}\right)^{1/r} 
  \\
  &\mathop{=}^{(a)} L_1(u(\hat{\lambda}), \hat{\lambda})
  -  \hat{L}_1 \cdot \left(\frac{ M L_1^{\ast}}{\alpha}\right)^{1/r}   \notag
  \\
  & \mathop{\geq}^{(b)} L_1(u(\lambda^{\ast}), \lambda^{\ast})) -  \hat{L}_1 \cdot \left(\frac{ M L_1^{\ast}}{\alpha}\right)^{1/r} 
  \\
  &\mathop{=}^{(c)} L_1(u^{\ast}, \lambda^{\ast}) -  \hat{L}_1 \cdot \left(\frac{ M L_1^{\ast}}{\alpha}\right)^{1/r} 
\end{align}
where (a) uses the fact that $\hat{u} = \arg\min L_2(u, \hat{\lambda})$ and (b) (c) use the definition of $\lambda^{\ast} = \arg\min L_1(u(\lambda), \lambda) $. Furthermore by the optimality of $(\hat{u}, \hat{\omega}, \hat{\lambda})$ of $L^{\alpha}(u, \omega, \lambda)$, we have
\begin{align*}
&L_1(\hat{\omega}, \hat{\lambda}) + \alpha \left(L_2(\hat{\omega}, \hat{\lambda}) - L_2(\hat{u}, \hat{\lambda}) \right)
\\
=& L'(\hw, \hlambda)
\leq
L'(\ustar, \lambdastar)
\\
=& L_1(u^{\ast}, \lambda^{\ast}) + \alpha \left(L_2(u^{\ast}, \lambda^{\ast}) - L_2(u^{\ast}, \lambda^{\ast}) \right) = L_1(u^{\ast}, \lambda^{\ast}).
\end{align*}
Therefore, the inequalities in claim (2) hold. 
\end{proof}

\subsection{Exact Bilevel-Minimax Equivalence when $\alpha\to\infty$}
\label{appendix:exact-equivalence-when-alpha-goes-to-inf}
\begin{cor}
  \label{cor:exact-equivalence-when-alpha-goes-to-inf}
  With the same settings in Theorem~\ref{thm:equivalent:b:m}, suppose
  $\omega, u \in \Omega, \lambda \in \Lambda$ are all compact sets and $L_2(u, \lambda)$ being continuous. Furthermore, assume the inner problem of~\eqref{P:bilevel} admits unique solutions. Denote $\hw_\alpha$, $\hu_\alpha$ and $\hlambda_\alpha$ as the minimax optimum for any fixed $\alpha$, then for any sequence $\left\{ \left(\alpha, \hw_{\alpha}, \hlambda_{\alpha}\right) \right\}$ satisfying $\alpha \to \infty$, there exists a subsequence $\left\{ \left(\alpha_n, \hw_{\alpha_n}, \hlambda_{\alpha_n}\right) \right\}_n$, s.t.
  \begin{align*}
    \lim_{n\to\infty} \hw_{\alpha_n} = \ustar \text{ and }
    \lim_{n\to\infty} \hlambda_{\alpha_n} = \lambdastar
  \end{align*}
\end{cor}
\begin{proof}
  According to Bolzano-Weierstrass theorem and the compactness of $\Omega, \Lambda$, for any sequences of $\left\{ \left(\alpha, \hw_\alpha, \hu_\alpha, \hlambda_\alpha\right)\right\}$ satisfying $\alpha \to \infty$, there exists a subsequence $\left\{ \left(\alpha_n, \hw_{\alpha_n}, \hu_{\alpha_n}, \hlambda_{\alpha_n}\right) \right\}_n$ that converges
  \begin{align*}
    \hw_\infty \triangleq \lim_{n\to\infty} \hw_{\alpha_n}, \quad \hu_\infty \triangleq \lim_{n\to\infty} \hu_{\alpha_n}, \quad
    \hlambda_\infty \triangleq \lim_{n\to\infty} \hlambda_{\alpha_n}
     \text{ and } \lim_{n\to\infty} \alpha = \infty, 
  \end{align*}
  First, we show that $\hu_\infty = u(\hlambda_\infty)$. For $\forall \epsilon > 0$, based on the definition of limits and continuity of $L_2(u, \lambda)$, there exists $N > 0$, s.t. for $\forall n \ge N$,
  \begin{align*}
    &\left\| L_2\left(\hu_\infty, \hlambda_\infty\right) - L_2\left(\hu_{\alpha_n}, \hlambda_{\alpha_n}\right) \right\| < \epsilon / 2
    \\
    &
    \left\|L_2\left(u(\hlambda_\infty), \hlambda_{\alpha_n}\right) - L_2\left(u(\hlambda_\infty), \hlambda_\infty\right) \right\| < \epsilon / 2
  \end{align*}
  It follows
  \begin{align*}
     L_2\left(\hu_\infty, \hlambda_\infty\right)
     &\le
     L_2\left(\hu_{\alpha_n}, \hlambda_{\alpha_n}\right) + \frac{\epsilon}{2}
     \\
     &=
     L_2\left(u(\hlambda_{\alpha_n}), \hlambda_{\alpha_n}\right) + \frac{\epsilon}{2}
     \\
     &\le
     L_2\left(u(\hlambda_\infty), \hlambda_{\alpha_n}\right) + \frac{\epsilon}{2}
     \\
     &\le
     L_2\left(u(\hlambda_\infty), \hlambda_{\infty}\right) + \frac{\epsilon}{2} + \frac{\epsilon}{2}
  \end{align*}
  Here the equality and the second inequality are entailed by the optimality of $\hu_{\alpha_n}$ in $L^{\alpha_n}(u, \omega, \lambda)$, as proved in Equation~\eqref{eq:u:minmax:optimality}. Furthermore, since $u(\hlambda_\infty)$ is the minimizer of $L_2(u, \hlambda_\infty)$, we have
  \begin{align*}
    L_2\left(u(\hlambda_\infty), \hlambda_{\infty}\right)
    \le L_2\left(\hu_\infty, \hlambda_\infty\right)
  \end{align*}
  Hence
  \begin{align*}
     L_2\left(u(\hlambda_\infty), \hlambda_{\infty}\right)
     \le L_2\left(\hu_\infty, \hlambda_\infty\right)
     \le
     L_2\left(u(\hlambda_\infty), \hlambda_{\infty}\right) + \epsilon
  \end{align*}
  Given the arbitrarity of $\epsilon$, we have
  \begin{align*}
    L_2\left(u(\hlambda_\infty), \hlambda_{\infty}\right) = L_2\left(\hu_\infty, \hlambda_\infty\right)
  \end{align*}
  Since $L_2(u, \hlambda_\infty)$ admits a unique solution, we have
  \begin{align*}
    u(\hlambda_\infty) = \hu_\infty
  \end{align*}

  Second, we show that $\hw_\infty = u(\hlambda_\infty)$. According to (1) in Theorem~\ref{thm:equivalent:b:m},
  \begin{align*}
    & 0 \le L_2(\hw_{\alpha_n}, \hlambda_{\alpha_n}) - L_2(\hu_{\alpha_n}, \hlambda_{\alpha_n})
    \leq \frac{L_1^{\ast}}{\alpha_n}
    \\
    \overset{\text{(ii) in Theorem~\ref{thm:equivalent:b:m}}}{\Longrightarrow} \quad&
     \left\| \hw_{\alpha_n} - \hu_{\alpha_n} \right\|^r \leq \frac{M L_1^{\ast}}{\alpha_n}
    \\
    \overset{n \to \infty}{\Longrightarrow} \quad&
    \left\| \hw_{\infty} - \hu_{\infty} \right\|^r  = 0
  \end{align*}
  thus $\hw_{\infty} = \hu_{\infty} = u(\hlambda_\infty)$.

  Finally, we show that $\hlambda_\infty = \lambdastar$, which makes $\hw_\infty = u(\hlambda_\infty) = u(\lambdastar) = \ustar$. According to (2) in Theorem~\ref{thm:equivalent:b:m}, we have
  \begin{align*}
    &L_1(u^{\ast}, \lambda^{\ast}) - \hat{L}_1 \cdot \left(\frac{ M L_1^{\ast}}{\alpha_n}\right)^{1/r}   \leq L_1(\hw_{\alpha_n}, \hlambda_{\alpha_n}) \leq L_1(u^{\ast}, \lambda^{\ast}).
    \\
    \overset{n\to\infty}{\Longrightarrow}\quad&
      L_1(\hw_{\infty}, \hlambda_{\infty}) = L_1(u^{\ast}, \lambda^{\ast}).
    \\
    \overset{\hw_{\infty} = u(\hlambda_\infty)}{\Longrightarrow}\quad&
      L_1(u(\hlambda_\infty), \hlambda_{\infty}) = L_1(u(\lambda^{\ast}), \lambda^{\ast})
    \\
    \Longrightarrow \quad&
      \hlambda_\infty = \lambdastar,
  \end{align*}
  where the last step is entailed by the uniqueness of the bilevel problem's optimum, given $\lambdastar$
 being well-defined.

 Therefore, we can obtain
 \begin{align*}
   &\lim_{n\to\infty} \hlambda_{\alpha_n} = \hlambda_\infty = \lambdastar
   \\
   &\lim_{n\to\infty} \hw_{\alpha_n} = \hw_\infty = u(\hlambda_\infty) = u(\lambdastar) = \ustar
 \end{align*}
\end{proof}

\subsection{An Example of Bilevel-Minimax Equivalence}\label{appendix:bilevel-example-one-dim}

\begin{exmp} \label{example:convergence_one_dim}
We consider the bilevel problem with one-dimensional least-square functions 
\begin{align*}
    \min_{\lambda \in \Lambda} \,\, & L_1(\lambda)  := \frac{\mu_1}{2} (u(\lambda)- \tilde{\omega}_1)^2  \notag \\
  \text{s.t.}\,\, &  u(\lambda) =  \arg\min_{u} L_2(u, \lambda)  := \frac{\mu_2}{2}(u - \tilde{\omega}_2)^2 + \lambda u^2
\end{align*}
where $u \in \R^1$ and $\lambda \in \Lambda := [0, \lambda_{\max}]$. Then $\hlambda = \lambdastar$.
\end{exmp}
 The solution to the inner problem is $u(\lambda)  =  \frac{\mu_2  }{\mu_2 + 2\lambda}\tilde{\omega}_2  
$. Incorporating this inner solution to the outer problem, 
  then
the solution of the outer problem is 
  \begin{align}
      \lambda^{\ast} = \text{Proj}_{\Lambda}\left(\frac{\mu_2}{2} \left(\frac{\tilde{\omega}_2}{\tilde{\omega}_1} -1\right) \right)
  \end{align}
where the inner minimzer $u^{\ast} = u(\lambda^{\ast}) = \frac{\mu_2 \tilde{\omega}_2}{\mu_2 + 2\lambda^{\ast}} $.
In this case, the minimax formulation (\ref{P:min:max}) is
  \begin{align}
    \min_{\omega, \lambda \in \Lambda } \max_{u} L(\omega, \lambda, u) = \frac{\mu_1}{2} (\omega - \tilde{\omega}_1)^2 + \alpha \left(\frac{\mu_2}{2}(\omega - \tilde{\omega}_2)^2 + \lambda \omega^2 - \frac{\mu_2}{2}(u - \tilde{\omega}_2)^2 - \lambda u^2 \right)
  \end{align}
  where $\alpha > 0$.
 We follow the procedure that we first maximize the minimax problem on $u$, then minimize the problem on $\omega$, and next minimize the problem on $\lambda$.
  For any fixed $\alpha > 0$,  the solution of the minimax problem $(\hat{u}, \hat{\omega}, \hat{\lambda})$ is
 \begin{align*}
        \hat{\lambda} & =  \text{Proj}_{\Lambda} \left(\frac{\mu_2}{2}\left( \frac{\tilde{\omega}_2}{\tilde{\omega}_1} - 1 \right) \right);\quad 
        \hat{\omega}  = \frac{\mu_1 \tilde{\omega}_1 + \alpha \mu_2 \tilde{\omega}_2}{\mu_1 + \alpha (\mu_2 + 2\hat{\lambda})}; \quad 
         \hat{u}  = \frac{\mu_2}{\mu_2 + 2\hat{\lambda}} \tilde{\omega}_2 
    \end{align*}
The first observation is that $\hat{\lambda} = \lambda^{\ast}$.     If $\lambda^{\ast} = \frac{\mu_2}{2} (\frac{\tilde{\omega}_2}{\tilde{\omega}_1} -1)  \in \Lambda $, we have $\hat{u} = \tilde{\omega}_1$. In this case, the minimizer of the bilevel problem $\omega^{\ast} = \hat{u}$. The minimax problem has the same solution as the bilevel problem. Else, if $\lambda^{\ast} = \frac{\mu_2}{2}(\frac{\tilde{\omega}_2}{\tilde{\omega}_1} -1)  \notin \Lambda $, we get that $\lambda^{\ast} = 0$ or $\lambda^{\ast} = \lambda_{\max}$. Whatever we always have $\hat{\lambda} = \lambda^{\ast}$ and $\hat{u} = \omega^{\ast}$. If $\alpha \rightarrow \infty$, we have $\hat{\omega} \rightarrow \frac{\mu_2}{\mu_2 + 2\lambda} \tilde{\omega}_2 = \hat{u}$. Thus, when $\alpha \rightarrow \infty$, the minimax problem formulated in (\ref{P:min:max}) is equivalent to the bilevel problem (\ref{P:bilevel}).

We might as well set $\mu_1=1$, $\mu_2=0.1$, $\tilde{\omega}_1=0.1$, $\tilde{\omega}_2 = 1$, then the optimal hyper-parameter of the bilevel problem 
is $\lambda^{\ast} = 0.45$ and the corresponding $\omega^{\ast} = \omega(\lambda^{\ast})=0.1$.  For Algorithm \ref{alg:minmax:general}, we set $u_0^0=\omega_0^0=0$, $\alpha_0=\eta_0=\eta_0^{\lambda}=1$, $\lambda_0^0=1$, after $K=100$ and $N=5$ steps, the output of Algorithm \ref{alg:minmax:general} is $(u_{K+1}^N, \omega_{K+1}^N, \lambda_{K+1}^N) = (0.10015, 0.10014, 0.44925)$. Therefore,  Algorithm \ref{alg:minmax:general} produces the hyper-parameter $\lambda_{K+1}^{N}$, which is a relatively high-accuracy solution of the bilevel problem with only $0.001$ noise error.

\section{Proofs and Useful Lemmas of Theorem~\ref{thm:equiv}}

\begin{lem}[\cite{ghadimi2018approximation}]
 Under Assumption~\ref{assumpt:v2}, we have
 $\mathcal{L}(\lambda)$ is $\ell$-smooth where $\ell = \mathcal{O}\left(\kappa^3\right)$ and $\kappa = \max\left\lbrace \ell_{10}, \ell_{11}, \ell_{21}\right\rbrace / \mu_2$.
\end{lem}
\begin{lem}\label{lem:omega:ast}
Under Assumption~\ref{assumpt:v2}, we have
\begin{align*}
 \left\| \omega_{\alpha}^{\ast}(\lambda) - \omega^{\ast}(\lambda)\right\| \leq \frac{C_0}{\alpha}
\end{align*}
where $C_0 = \ell_{10}/\mu_2$.
\end{lem}
\begin{proof}
 (1) By the optimality of $\omega_{\alpha}^{\ast}$ in $\Phi^{\alpha}(\omega, \lambda)$, we have
 \begin{align}
     \nabla_{\omega} \Phi^{\alpha}(\omega_{\alpha}^{\ast}(\lambda), \lambda) = \nabla_{\omega} L_1(\omega_{\alpha}^{\ast}(\lambda), \lambda) + \alpha \nabla_{\omega} L_2(\omega_{\alpha}^{\ast}(\lambda), \lambda) = 0. \notag 
 \end{align}
 For the strongly convexity of $L_2$ w.r.t. $\omega$ implies that 
 \begin{align}\label{sc:inequ}
     \left\| \nabla_{\omega} L_2(\omega, \lambda) - \nabla_{\omega} L_2(\omega^{'}, \lambda) \right\| \geq \mu_2 \left\| \omega - \omega^{'} \right\|, \quad \forall \, \omega, \omega^{'}. 
 \end{align}
 Recalling the definition $\omega^{\ast}(\lambda) = \arg\min_{\omega} L_2(\omega, \lambda)$, we achieve that
 \begin{align}
     \mu_2 \left\| \omega_{\alpha}^{\ast}(\lambda) - \omega^{\ast}(\lambda) \right\| & \mathop{\leq}^{(a)} \left\| \nabla_{\omega} L_2(\omega_{\alpha}^{\ast}(\lambda), \lambda) - \nabla_{\omega} L_2(\omega^{\ast}(\lambda), \lambda) \right\| \mathop{=}^{(b)} \left\| \nabla_{\omega} L_2(\omega_{\alpha}^{\ast}(\lambda), \lambda) \right\| \notag \\
     &  \mathop{=}^{(c)} \frac{1}{\alpha} \left\| \nabla_{\omega} L_1(\omega_{\alpha}^{\ast}(\lambda), \lambda) \right\|  \mathop{\leq}^{(d)} \frac{\ell_{10}}{\alpha} \notag 
 \end{align}
 where $(a)$ uses the property of (\ref{sc:inequ}) which is implied from the strongly convexity of $L_2$ (w.r.t. $\omega$), $(b)$ and $(c)$ are obtained from the optimality of $\omega^{\ast}(\lambda) $ and $\omega_{\alpha}^{\ast}(\lambda)$ resepctively, and $(d)$ follows from the Lipschtiz continuity of $L_1$.
\end{proof}
\begin{lem}\label{lem:grad:omega:bound}
Under Assumption~\ref{assumpt:v2}, if $\alpha > 2\ell_{11}/\mu_2$, we have $\left\|\nabla \omega_{\alpha}^{\ast}(\lambda)\right\| \leq 3\ell_{21} /\mu_2$.
\end{lem}
\begin{proof}
Under Assumption~\ref{assumpt:v2}, if $\alpha \geq 2\ell_{11}/\mu_2$, then 
\begin{align}
\lambda_{\min}\left(\nabla_{\omega}^2 L^{\alpha}(u, \omega, \lambda) \right) & = \lambda_{\min} \left(\nabla_{\omega}^2 L_1(\omega, \lambda) + \alpha \nabla_{\omega}^2 L_2(\omega, \lambda) \right)  \notag \\
& =\lambda_{\min} \left(\nabla_{\omega}^2 L_1(\omega, \lambda) \right) + \alpha \lambda_{\min} \left(\nabla_{\omega}^2 L_2(\omega, \lambda) \right) \notag \\
& = - \ell_{11} + \alpha \mu_2  \geq \frac{\alpha \mu_2}{2}. \notag 
\end{align}
That is: $L^{\alpha}(u, \omega, \lambda)$ is $\alpha \mu_2 /2$-strongly convex in $\omega$. The definition of $u^{\ast}(\lambda) = \arg\min L_2(u,\lambda)$,  implies that $\nabla_{u} L_2(u^{\ast}(\lambda), \lambda) = 0$. Taking derivative w.r.t. $\lambda$ on the both sides of $\nabla_{u} L_2(u^{\ast}(\lambda), \lambda) = 0$ yields
\begin{align}\label{inequ:grad:u:ast}
\nabla_{u}^2 L_2(u^{\ast}(\lambda), \lambda) \nabla_{\lambda} u^{\ast}(\lambda) + \nabla_{u \lambda}^2 L_2(u^{\ast}(\lambda), \lambda) = 0.
\end{align}
By the optimality of $\omega_{\alpha}^{\ast}(\lambda)$ such that $\omega_{\alpha}^{\ast}(\lambda) = \arg\min_{\omega} \Phi^{\alpha}(\omega, \lambda)$, we have $\nabla_{\omega} \Phi^{\alpha}(\omega_{\alpha}^{\ast}(\lambda), \lambda) = \nabla_{\omega} L^{\alpha}(u^{\ast}(\lambda), \omega_{\alpha}^{\ast}(\lambda), \lambda)=0$. The derivative of $L^{\alpha}$ with respect to $\omega$ is not affected by $u^{\ast}(\lambda)$. Then $\nabla_{\omega} L^{\alpha}(u, \omega_{\alpha}^{\ast}(\lambda), \lambda)=0$ holds for any $u$.  Taking derivative w.r.t. $\lambda$ on both sides gives that
\begin{align*}
  \nabla_{\omega}^2 L^{\alpha}(u, \omega_{\alpha}^{\ast}(\lambda), \lambda)  \nabla \omega_{\alpha}^{\ast} (\lambda) + \nabla_{\omega\lambda}^2 L^{\alpha}(u, \omega_{\alpha}^{\ast}(\lambda), \lambda) =0.
\end{align*}
Then
\begin{align*}
\left\| \nabla \omega_{\alpha}^{\ast} (\lambda) \right\| & = \left\|- \nabla_{\omega\lambda}^2 L^{\alpha}(u, \omega_{\alpha}^{\ast}(\lambda), \lambda) \left[ \nabla_{\omega}^2 L^{\alpha}(u, \omega_{\alpha}^{\ast}(\lambda), \lambda)\right]^{-1} \right\| \notag \\
& \leq \left\| \nabla_{\omega\lambda}^2 L^{\alpha}(u, \omega_{\alpha}^{\ast}(\lambda), \lambda) \right\| \left\|\nabla_{\omega}^2 L^{\alpha}(u, \omega_{\alpha}^{\ast}(\lambda), \lambda)^{-1}  \right\|
\leq 3\ell_{21} /\mu_2
\end{align*}
where $\left\|\nabla_{\omega\lambda}^2 L^{\alpha}(u, \omega_{\alpha}^{\ast}(\lambda), \lambda) \right\|\leq \lambda_{\max}\left( \nabla_{\omega\lambda}^2 L^{\alpha}(u, \omega_{\alpha}^{\ast}(\lambda), \lambda)\right) \leq \ell_{11} + \alpha \ell_{21} \leq \alpha \left(\frac{\mu_2}{2}  + \ell_{21} \right) \leq \frac{3}{2} \alpha \ell_{21}$ with $\mu_2 \leq \ell_{21}$.
\end{proof}
\begin{lem}\label{grad:inequ:omega}
Under Assumption~\ref{assumpt:v2}, if $\alpha \geq 2\ell_{11}/\mu_2$, then 
\begin{align}
    \left\| \nabla \omega_{\alpha}^{\ast}(\lambda) - \nabla \omega^{\ast}(\lambda) \right\| \leq \frac{C_1}{\alpha} \notag 
\end{align}
where $C_1 = \mathcal{O}(\kappa^3)$.
\end{lem}
\begin{proof}(of Lemma~\ref{grad:inequ:omega})
The proof is similar to Lemma B.5. in Lesi Chen's paper. We omit it here.
\end{proof}

\subsection{Proof of Theorem~\ref{thm:equiv}}
\thmequivstronger*
\begin{proof}
(1): To demonstrate (\ref{inequ:zero:dif}) of Theorem~\ref{thm:equiv}. By definitions of $\omega^{\ast}(\lambda) = u^{\ast}(\lambda)= \arg\min_{\omega} L_2(\omega, \lambda)$, we have $u^{\ast}(\lambda) = \omega^{\ast}(\lambda)$, then
\begin{align}
|\mathcal{L}(\lambda) - \Gamma^{\alpha}(\lambda)| &  \leq |L_1(\omega_{\alpha}^{\ast}(\lambda), \lambda) + \alpha \left( L_2(\omega_{\alpha}^{\ast}(\lambda), \lambda) - L_2(u^{\ast}(\lambda), \lambda) \right) - L_1(\omega^{\ast}(\lambda), \lambda)| \notag \\
& \leq |L_1(\omega_{\alpha}^{\ast}(\lambda), \lambda)-L_1(\omega^{\ast}(\lambda), \lambda)| + \alpha |L_2(\omega_{\alpha}^{\ast}(\lambda), \lambda) - L_2(u^{\ast}(\lambda), \lambda)| \notag \\
& \mathop{\leq}^{(a)} \ell_{10}\left\|\omega_{\alpha}^{\ast}(\lambda) - \omega^{\ast}(\lambda) \right\| + \frac{\alpha \ell_{21}}{2} \left\| \omega_{\alpha}^{\ast}(\lambda) - u^{\ast}(\lambda) \right\|^2 \notag \\
& \mathop{\leq}^{(b)} \ell_{10} \left(1+ \frac{\ell_{21}}{2\mu_2} \right) \frac{\ell_{10}}{\mu_2 \alpha}
\end{align}
where $(a)$ uses the Lipschtiz continuity of $L_1$ and gradient Lipschitz property of $L_2$ and $(b)$ uses the result of Lemma \ref{lem:omega:ast}.

(2): To prove (\ref{inequ:one:dif}) of Theorem~\ref{thm:equiv}. The gradient of the minimax problem $L^{\alpha}$ in (\ref{P:min:max}) is computed by:
\begin{subequations}
\begin{align}
  \nabla_{\lambda} L^{\alpha}(u, \omega, \lambda) & = \nabla_{\lambda} L_1(\omega, \lambda) + \alpha \left( \nabla_{\lambda} L_2(\omega, \lambda) - \nabla_{\lambda} L_2(u, \lambda)\right)  \label{minimax:grad:lambda}\\
 \nabla_{u} L^{\alpha}(u, \omega, \lambda) & =  - \nabla_{u} L_2(u, \lambda) \label{minimax:grad:u} \\
  \nabla_{\omega} L^{\alpha}(u, \omega, \lambda) & = \nabla_{\omega} L_1(\omega, \lambda) + \alpha \nabla_{\omega} L_2(\omega, \lambda).  \label{minimax:grad:omega}
\end{align}
\end{subequations}
By the optimality of $u^{\ast}(\lambda)$ such that $\nabla_{u} L^{\alpha}(u^{\ast}(\lambda), \omega, \lambda) = 0$ for any $\omega$, then we have $\nabla_{u} L^{\alpha}(u^{\ast}(\lambda), \omega_{\alpha}^{\ast}(\lambda), \lambda) = 0$. By the optimality of $\omega_{\alpha}^{\ast}$, then $\nabla_{\omega} L^{\alpha}(u^{\ast}(\lambda), \omega_{\alpha}^{\ast}(\lambda), \lambda) =0$, we thus have
   \begin{align}
       \nabla \Gamma^{\alpha}(\lambda) & = \nabla_{\lambda} \Phi(\omega_{\alpha}^{\ast}(\lambda), \lambda) = \nabla_{\lambda} L^{\alpha}(u^{\ast}(\lambda), \omega_{\alpha}^{\ast}(\lambda), \lambda)  \notag \\
       & = \nabla_{\lambda} L^{\alpha}(u^{\ast}(\lambda), \omega_{\alpha}^{\ast}(\lambda), \lambda) +  (\nabla_{\lambda} u^{\ast}(\lambda))^{T}\nabla_{u} L^{\alpha}(u^{\ast}(\lambda), \omega_{\alpha}^{\ast}(\lambda), \lambda)  \notag \\
       & \quad +  (\nabla_{\lambda} \omega_{\alpha}^{\ast}(\lambda))^{T}\nabla_{\omega} L^{\alpha}(u^{\ast}(\lambda), \omega_{\alpha}^{\ast}(\lambda), \lambda) \notag \\
       & = \nabla_{\lambda} L^{\alpha}(u^{\ast}(\lambda), \omega_{\alpha}^{\ast}(\lambda), \lambda).
   \end{align} 
For bilevel problem, the hyper-gradient can be estimated as:
\begin{align}\label{inequ:grad:bilevel}
 \nabla \mathcal{L}(\lambda) & = \nabla_{\lambda} L_1(u^{\ast}(\lambda), \lambda)   = \nabla_{\lambda} L_1(u^{\ast}(\lambda), \lambda) + (\nabla_{\lambda} u^{\ast}(\lambda))^T \nabla_{u} L_1(u^{\ast}(\lambda), \lambda) \notag \\
 & \mathop{=}^{(a)} \nabla_{\lambda} L_1(u^{\ast}(\lambda), \lambda) - \nabla_{\lambda u}^2L_2(u^{\ast}(\lambda), \lambda)\left[\nabla_{u}^2L_2(u^{\ast}(\lambda), \lambda)\right]^{-1} \nabla_{u} L_1(u^{\ast}(\lambda), \lambda)
\end{align}
where $(a)$ uses the fact derived from the equality (\ref{inequ:grad:u:ast}). By using the definition of $\nabla_{\lambda} L^{\alpha}(u^{\ast}(\lambda), \omega, \lambda)$ and applying (\ref{inequ:grad:bilevel}), we have
\begin{align}\label{inequ:L:Lalpha}
 \nabla \mathcal{L}(\lambda) - \nabla_{\lambda} L^{\alpha}(u^{\ast}(\lambda), \omega, \lambda) & =  \nabla_{\lambda} L_1(u^{\ast}(\lambda), \lambda) -  \nabla_{\lambda} L_1(\omega, \lambda) - \alpha \left( \nabla_{\lambda} L_2(\omega, \lambda) - \nabla_{\lambda} L_2(u^{\ast}(\lambda), \lambda)\right)  \notag \\
& - \nabla_{\lambda u}^2L_2(u^{\ast}(\lambda), \lambda)\left[\nabla_{u}^2L_2(u^{\ast}(\lambda), \lambda)\right]^{-1} \nabla_{u} L_1(u^{\ast}(\lambda), \lambda)
\end{align}
We then turn to estimate the difference of $\nabla_{\lambda} L_2(\omega, \lambda)$ and $\nabla_{\lambda} L_2(u^{\ast}(\lambda), \lambda)$ below:
\begin{align}\label{inequ:grad:lambda:L2}
\nabla_{\lambda} L_2(\omega, \lambda) - \nabla_{\lambda} L_2(u^{\ast}(\lambda), \lambda) &  =  \nabla_{\lambda} L_2(\omega, \lambda) - \nabla_{\lambda} L_2(u^{\ast}(\lambda), \lambda)- \nabla_{\lambda u}^2 L_2(u^{\ast}(\lambda), \lambda)^{T}(\omega - u^{\ast}(\lambda)) \notag \\
& + \nabla_{\lambda u}^2 L_2(u^{\ast}(\lambda), \lambda)^{T}(\omega - u^{\ast}(\lambda))
\end{align}
By the definition of $u^{\ast}(\lambda)$ such that $u^{\ast}(\lambda)= \arg\min L_2(u, \lambda)$, we have $\nabla_{u} L_2(u^{\ast}(\lambda),\lambda)=0$. Note that the equation $ \nabla_{\omega} L_1(\omega, \lambda) + \alpha \nabla_{\omega} L_2(\omega, \lambda) = \nabla_{\omega} L^{\alpha}(u, \omega, \lambda)$  holds for any $u$, then $\omega - u^{\ast}(\lambda)$  is reformulated as 
\begin{align}\label{inequ:omega:uast}
\omega - u^{\ast}(\lambda) & = - \nabla_{u}^2 L_2(u^{\ast}(\lambda), \lambda)^{-1} \left(\nabla_{\omega} L_2(\omega, \lambda) - \nabla_{u} L_2(u^{\ast}(\lambda),\lambda \right)  - \nabla_{u}^2 L_2(u^{\ast}(\lambda), \lambda)(\omega - u^{\ast}(\lambda)) \notag \\
 & + \frac{1}{\alpha} \nabla_{u}^2 L_2(u^{\ast}(\lambda), \lambda)^{-1} \left(\nabla_{\omega} L^{\alpha}(u^{\ast}(\lambda), \omega, \lambda) - \nabla_{\omega} L_1(\omega, \lambda)\right).
\end{align}
Incorporating (\ref{inequ:omega:uast}) into (\ref{inequ:grad:lambda:L2}) and then applying the result into (\ref{inequ:L:Lalpha}) gives that
\begin{align}\label{inequ:grad:L:core}
& \nabla \mathcal{L}(\lambda) -   \nabla_{\lambda} L^{\alpha}(u^{\ast}(\lambda), \omega, \lambda)\notag \\
& =  \nabla_{\lambda} L_1(u^{\ast}(\lambda), \lambda) - \nabla_{\lambda u}^2 L_2(u^{\ast}(\lambda), \lambda)\left[\nabla_{u}^2L_2(u^{\ast}(\lambda), \lambda)\right]^{-1} \nabla_{u} L_1(u^{\ast}(\lambda), \lambda) -   \nabla_{\lambda} L_1(\omega, \lambda) \notag \\
& - \alpha \left( \nabla_{\lambda} L_2(\omega, \lambda) - \nabla_{\lambda} L_2(u^{\ast}(\lambda), \lambda)\right) \notag \\
& =  \nabla_{\lambda} L_1(u^{\ast}(\lambda), \lambda)  - \nabla_{\lambda u}^2L_2(u^{\ast}(\lambda), \lambda)\left[\nabla_{u}^2L_2(u^{\ast}(\lambda), \lambda)\right]^{-1} \left(\nabla_{u} L_1(u^{\ast}(\lambda), \lambda) - \nabla_{\omega} L_1(\omega, \lambda) \right) \notag \\ 
& - \nabla_{\lambda} L_1(\omega, \lambda) + \alpha \left(\nabla_{\lambda} L_2(\omega, \lambda) - \nabla_{\lambda} L_2(u^{\ast}(\lambda), \lambda)- \nabla_{\lambda u}^2 L_2(u^{\ast}(\lambda), \lambda)^{T}(\omega - u^{\ast}(\lambda)) \right) \notag \\
& - \alpha \nabla_{\lambda u}^2L_2(u^{\ast}(\lambda), \lambda)^{T}\nabla_{u}^2 L_2(u^{\ast}(\lambda), \lambda)^{-1} \left(\nabla_{\omega} L_2(\omega, \lambda) - \nabla_{u} L_2(u^{\ast}(\lambda),\lambda)  - \nabla_{u}^2 L_2(u^{\ast}(\lambda),\lambda)^{T}(\omega - u^{\ast}(\lambda)\right) \notag \\
& - \nabla_{\lambda u}^2 L_2(u^{\ast}(\lambda), \lambda)^{T}\nabla_{u}^2 L_2(u^{\ast}(\lambda), \lambda)^{-1}\nabla_{\omega} L^{\alpha}(u^{\ast}(\lambda), \omega, \lambda) 
\end{align}
(i): By the Hessian-Lipschitz of $L_2$, the third term of (\ref{inequ:grad:L:core}) can be estimated as:
\begin{align}
  \left\|\nabla_{\lambda} L_2(\omega, \lambda) - \nabla_{\lambda} L_2(u^{\ast}(\lambda), \lambda)- \nabla_{\lambda u}^2 L_2(u^{\ast}(\lambda), \lambda)^{T}(\omega - u^{\ast}(\lambda))  \right\| \leq \frac{\ell_{22}}{2} \left\| \omega - u^{\ast}(\lambda) \right\|^2.
\end{align}
(ii): Similarly, we use the Hessian-Lipschitz of $L_2$ and estimate the fourth term of (\ref{inequ:grad:L:core}) below:
\begin{align}
\left\| \nabla_{\omega} L_2(\omega, \lambda) - \nabla_{u} L_2(u^{\ast}(\lambda),\lambda)  - \nabla_{u}^2 L_2(u^{\ast}(\lambda),\lambda)^{T}(\omega - u^{\ast}(\lambda))\right\| \leq \frac{\ell_{22}}{2} \left\| \omega - u^{\ast}(\lambda) \right\|^2.
\end{align}
(iii): by the smoothness of $L_1$, we have
\begin{subequations}
\begin{align}
 \left\| \nabla_{\lambda} L_1(u^{\ast}(\lambda), \lambda) - \nabla_{\lambda} L_1(\omega, \lambda) \right\| & \leq \ell_{11} \left\| \omega - u^{\ast}(\lambda) \right\| \\
\left\| \nabla_{u} L_1(u^{\ast}(\lambda), \lambda) - \nabla_{\omega} L_1(\omega, \lambda)  \right\| & \leq \ell_{11} \left\| \omega - u^{\ast}(\lambda) \right\|.
\end{align}
\end{subequations}
Based on the above results, we can conclude that
\begin{align}
 &\left\| \nabla \mathcal{L}(\lambda) -   \nabla_{\lambda} L^{\alpha}(u^{\ast}(\lambda), \omega, \lambda) + \nabla_{\lambda u}^2L_2(u^{\ast}(\lambda), \lambda)^{T}\nabla_{u}^2 L_2(u^{\ast}(\lambda), \lambda)^{-1}\nabla_{\omega} L^{\alpha}(u^{\ast}(\lambda), \omega, \lambda)  \right\|^2  \notag \\
 & \leq \ell_{11}\left( 1+ \ell_{21}/\mu_2\right)\left\|\omega- u^{\ast}(\lambda) \right\|  + \frac{\alpha \ell_{22}}{2} \left(1 +  \ell_{21}/\mu_2\right)\left\|\omega- u^{\ast}(\lambda) \right\|^2 
\end{align}
Let $\omega = \omega_{\alpha}^{\ast}(\lambda)$, then $\nabla_{\omega} L^{\alpha}(u^{\ast}(\lambda), \omega_{\alpha}^{\ast}(\lambda), \lambda) = 0$ by the optimality of $\omega_{\alpha}^{\ast}(\lambda)$, we can achieve that 
\begin{align}
  \left\| \nabla \mathcal{L}(\lambda) -  \nabla \Gamma^{\alpha}(\lambda) \right\| & =     \left\| \nabla \mathcal{L}(\lambda) -  \nabla_{\lambda} L^{\alpha}(u^{\ast}(\lambda), \omega_{\alpha}^{\ast}(\lambda), \lambda) \right\| \notag \\
  & \leq \ell_{11}\left( 1+ \ell_{21}/\mu_2\right)\left\|\omega_{\alpha}^{\ast}(\lambda)- u^{\ast}(\lambda) \right\|  + \frac{\alpha \ell_{22}}{2} \left(1 +  \ell_{21}/\mu_2\right)\left\|\omega_{\alpha}^{\ast}(\lambda)- u^{\ast}(\lambda) \right\|^2  \notag \\
  & \leq \ell_{11}\left( 1+ \ell_{21}/\mu_2\right)\left\|\omega_{\alpha}^{\ast}(\lambda)- u^{\ast}(\lambda) \right\|  + \frac{\ell_{10} \ell_{22}}{2 \mu_2} \left(1 +  \ell_{21}/\mu_2\right)\left\|\omega_{\alpha}^{\ast}(\lambda)- u^{\ast}(\lambda) \right\| \notag \\
  & \leq \left(\ell_{11} +  \frac{\ell_{10} \ell_{22}}{2 \mu_2} \right) \left(1 +  \ell_{21}/\mu_2\right) \frac{\ell_{10}}{\mu_2\alpha}.
\end{align}

(3): We turn to prove (\ref{inequ:second:dif}) of Theorem~\ref{thm:equiv}.
\begin{align}
    \nabla^2 \Gamma_{\alpha}(\lambda) & = \nabla_{\lambda} \left(\nabla_{\lambda} L_1(\omega_{\alpha}^{\ast}(\lambda), \lambda) +  \alpha \left( \nabla_{\lambda}L_2(\omega_{\alpha}^{\ast}(\lambda), \lambda) - \nabla_{\lambda} L_2(u^{\ast}(\lambda), \lambda) \right)\right) \notag \\
    & = \nabla_{\lambda}^2 L_1(\omega_{\alpha}^{\ast}(\lambda), \lambda) + \nabla \omega_{\alpha}^{\ast}(\lambda)^{T} \nabla_{\omega \lambda}^2 L_1(\omega_{\alpha}^{\ast}(\lambda), \lambda) \notag \\
    & + \alpha \nabla_{\lambda}^2 L_2(\omega_{\alpha}^{\ast}(\lambda), \lambda) + \alpha \nabla \omega_{\alpha}^{\ast}(\lambda)^{T} \nabla_{\omega \lambda}^2 L_2(\omega_{\alpha}^{\ast}(\lambda), \lambda) \notag \\
    & -  \alpha  \nabla_{\lambda}^2 L_2(u^{\ast}(\lambda), \lambda) - \alpha \nabla u^{\ast}(\lambda)^{T} \nabla_{u\lambda}^2 L_2(u^{\ast}(\lambda), \lambda) \notag \\
    & = \nabla_{\lambda}^2 L_1(\omega_{\alpha}^{\ast}(\lambda), \lambda) + \nabla \omega_{\alpha}^{\ast}(\lambda)^{T} \nabla_{\omega \lambda}^2 L_1(\omega_{\alpha}^{\ast}(\lambda), \lambda) + \alpha \left(\nabla_{\lambda}^2 L_2(\omega_{\alpha}^{\ast}(\lambda), \lambda) - \nabla_{\lambda}^2 L_2(u^{\ast}(\lambda), \lambda)\right)\notag \\
    &  + \alpha \left( \nabla \omega_{\alpha}^{\ast}(\lambda)^{T} \nabla_{\omega \lambda}^2 L_2(\omega_{\alpha}^{\ast}(\lambda), \lambda) - \nabla u^{\ast}(\lambda)^{T} \nabla_{u\lambda}^2 L_2(u^{\ast}(\lambda), \lambda)\right)
\end{align}
\begin{align}
\left\| \nabla^2 \Gamma_{\alpha}(\lambda)\right\| & \leq \left\| \nabla_{\lambda}^2 L_1(\omega_{\alpha}^{\ast}(\lambda), \lambda)\right\| + \left\|\nabla \omega_{\alpha}^{\ast}(\lambda) \right\|\left\| \nabla_{\omega \lambda}^2 L_1(\omega_{\alpha}^{\ast}(\lambda), \lambda)\right\| \notag \\
& + \alpha \left\|\nabla_{\lambda}^2 L_2(\omega_{\alpha}^{\ast}(\lambda), \lambda) - \nabla_{\lambda}^2 L_2(u^{\ast}(\lambda), \lambda) \right\|  \notag \\
& + \alpha \left\|\nabla \omega_{\alpha}^{\ast}(\lambda)^{T} \nabla_{\omega \lambda}^2 L_2(\omega_{\alpha}^{\ast}(\lambda), \lambda) - \nabla u^{\ast}(\lambda)^{T} \nabla_{u\lambda}^2 L_2(u^{\ast}(\lambda), \lambda) \right\| \notag \\
& \mathop{\leq}^{(a)} \ell_{11} \left(1 + 3\ell_{21}/\mu_2\right)  + \alpha \ell_{22}\left\| \omega_{\alpha}^{\ast}(\lambda) - u^{\ast}(\lambda) \right\| \notag \\
& + \alpha \left\|\nabla \omega_{\alpha}^{\ast}(\lambda)^{T} \nabla_{\omega \lambda}^2 L_2(\omega_{\alpha}^{\ast}(\lambda), \lambda) -\nabla \omega_{\alpha}^{\ast}(\lambda)^{T} \nabla_{u \lambda}^2 L_2(u^{\ast}(\lambda), \lambda) \right\| \notag \\
& + \alpha \left\|\nabla \omega_{\alpha}^{\ast}(\lambda)^{T} \nabla_{u \lambda}^2 L_2(u^{\ast}(\lambda), \lambda) - \nabla u^{\ast}(\lambda)^{T} \nabla_{u\lambda}^2 L_2(u^{\ast}(\lambda), \lambda) \right\| \notag\\
& \mathop{\leq}^{(b)} \ell_{11} \left(1 + 3\ell_{21}/\mu_2\right)  + \alpha \ell_{22} \left(1 + 3\ell_{21}/\mu_2 \right)\left\| \omega_{\alpha}^{\ast}(\lambda) - u^{\ast}(\lambda) \right\| + \alpha\ell_{21} \left\|\nabla \omega_{\alpha}^{\ast}(\lambda) - \nabla u^{\ast}(\lambda) \right\| \notag \\
& \mathop{\leq}^{(c)} \ell_{11} \left(1 + 3\ell_{21}/\mu_2\right)  + \ell_{22} \left(1 + 3\ell_{21}/\mu_2 \right) \frac{\ell_{10}}{\mu_2} + \ell_{21} C_1
\end{align}
where $(a)$ uses the facts that $\left\|\nabla_{\lambda}^2 L_1(\omega_{\alpha}^{\ast}(\lambda), \lambda) \right\| \leq \ell_{11}$ and  $\left\|\nabla_{\omega\lambda}^2 L_1(\omega_{\alpha}^{\ast}(\lambda), \lambda) \right\| \leq \ell_{11}$, and applies the result of Lemma~\ref{lem:grad:omega:bound} and Hessian-Lipschitz of $L_2$; and (b) follows the Cauchy-Schwarz inequality and the property of Hessian-Lipschitz of $L_2$; (c) uses the results of Lemmas~\ref{lem:omega:ast} and \ref{grad:inequ:omega}.
\end{proof}

\section{Proofs of Lemmas and Theorems in Subsection~\ref{subsec:onestage}}
We denote:
\begin{align}
    L^{\alpha}(u, \omega, \lambda) = L_1(\omega, \lambda) + \alpha \left( L_2(\omega, \lambda) - L_2(u, \lambda) \right) \notag 
\end{align}
\lemonestagevtwo*
\begin{proof}\deleted{(of Lemma \ref{lem:onestage:v2})}
\textbf{For Claim (i)}: Under Assumption~\ref{assumpt:v2}, we have 
\begin{align}
    \lambda_{\max}(\nabla^2 L^{\alpha}(u, \omega, \lambda)) & =  \lambda_{\max}(\nabla^2 L_1(\omega, \lambda)) + \alpha  \lambda_{\max}(\nabla^2 L_2(\omega, \lambda) - \nabla^2 L_2(u, \lambda)) \notag \\
    & \leq  \ell_{11} + \alpha \ell_{21} + \alpha \ell_{21} \leq  \frac{\mu_2\alpha}{2} + 2\alpha \ell_{21}  \leq \frac{5}{2}\alpha \ell_{21}. \notag 
\end{align}
Because $L_2$ is $\mu_2$ strongly convex, then we have $L^{\alpha}$ is $\mu_2\alpha$-strongly concave w.r.t. $u$. Besides, $L_1$ is $\ell_{11}$ gradient-Lipschitz, then
\begin{align}
    \lambda_{\min}(\nabla_{\omega}^2 L^{\alpha}(u,\omega, \lambda)) = \lambda_{\min}(\nabla_{\omega}^2 L_1(\omega, \lambda) + \alpha \nabla_{\omega}^2 L_2(\omega, \lambda)) = - \ell_{11} + \alpha \mu_2 \geq \frac{\ell_{11}\alpha}{2} \notag 
\end{align}
where $\alpha \geq 2\ell_{11}/\mu_2$.

{\bf For Claim (ii)}:
Since $L_2$ is $\mu_2$-strongly convex in $u$ for any $(\omega, \lambda)$,  then the function $L^{\alpha}$ is $\mu_2\alpha$-strongly concave in $u$. Then the function $u^{\ast}(\lambda)$ is unique and well-defined. 
Let $x = (\omega, \lambda)$ and we choose $x_1 =(\omega, \lambda_1)$ and $x_2=(\omega, \lambda_2)$. By the optimality of $u^{\ast}(\lambda_1)$ and $u^{\ast}(\lambda_2)$: for any $u \in \R^d$, we have
\begin{subequations}
\begin{align}
\left\langle u - u^{\ast}(\lambda_1), \nabla_{u} L^{\alpha} (u^{\ast}(\lambda_1), x_1)\right\rangle & \leq 0,  \label{inequ:u:1}\\
\left\langle u - u^{\ast}(\lambda_2), \nabla_{u} L^{\alpha} (u^{\ast}(\lambda_2), x_2)\right\rangle & \leq 0. \label{inequ:u:2}
\end{align}
\end{subequations}
Let $u = u^{\ast}(\lambda_2)$ in (\ref{inequ:u:1}) and $u = u^{\ast}(\lambda_1)$ in (\ref{inequ:u:2}) and then sum the two inequalities, we get
\begin{align}\label{inequ:uast:1}
\left\langle u^{\ast}(\lambda_2) - u^{\ast}(\lambda_1), \nabla_{u} L^{\alpha} (u^{\ast}(\lambda_1), x_1) - \added{\nabla_u} L^{\alpha} (u^{\ast}(\lambda_2), x_2)\right\rangle \leq 0.
\end{align}
Recalling the strongly-concavity of $L^{\alpha}(u, x_1)$ with respect to $u$, we have
\begin{align}\label{inequ:uast:2}
\left\langle u^{\ast}(\lambda_2) -u^{\ast}(\lambda_1), \nabla_u L^{\alpha}(u^{\ast}(\lambda_2) , x_1) - \nabla_u L^{\alpha}(u^{\ast}(\lambda_1), x_1) \right\rangle + \mu_2\alpha \left\|u^{\ast}(\lambda_2) -u^{\ast}(\lambda_1) \right\|^2 \leq 0.
\end{align}
Plugging the two inequalities (\ref{inequ:uast:1}) and (\ref{inequ:uast:2}) gives that
\begin{align}\label{inequ:u:3}
 \mu_2\alpha\left\|u^{\ast}(\lambda_2) -u^{\ast}(\lambda_1) \right\|^2  & \leq     \left\langle u^{\ast}(\lambda_2) -u^{\ast}(\lambda_1),  \nabla_u L^{\alpha}(u^{\ast}(\lambda_2), x_2) - \nabla_u L^{\alpha}(u^{\ast}(\lambda_2) , x_1)\right\rangle \notag \\
 & \mathop{\leq}^{(a)} \left\| u^{\ast}(\lambda_2) -u^{\ast}(\lambda_1)\right\| \left\| \nabla_u L^{\alpha}(u^{\ast}(\lambda_2), x_2) - \nabla_u L^{\alpha}(u^{\ast}(\lambda_2) , x_1)\right\| \notag \\
 & \mathop{\leq}^{(b)} \alpha\ell_{21} \left\| u^{\ast}(\lambda_2) -u^{\ast}(\lambda_1)\right\| \left\|x_2 -x_1 \right\| \notag\\
 & = \alpha\ell_{21}\left\| u^{\ast}(\lambda_2) -u^{\ast}(\lambda_1)\right\| \left\| \lambda_1 -\lambda_2 \right\|
\end{align}
where $(a)$ uses the Cauchy-Schwartz inequality and $(b)$ follows the fact that $L^{\alpha}$ is $\alpha \ell_{21}$ gradient Lipschitz in $u$.
Thus 
\begin{align}\label{inequ:u:4}
   \left\|u^{\ast}(\lambda_2) -u^{\ast}(\lambda_1) \right\| \leq   \kappa \left\| \lambda_1 -\lambda_2 \right\|.
\end{align}
That is $u^{\ast}(\lambda)$ is $\kappa$-Lipschitz continuous with $\kappa  = \max \left\lbrace \ell_{22}, \ell_{21}, \ell_{10}, \ell_{11}\right\rbrace/\mu_2$.
Since $u^{\ast}(\lambda)$ is unique and, from Danskin's theorem that $\Phi^{\alpha}$ is differentiable with 
$$\nabla_{\lambda} \Phi^{\alpha}(\omega, \lambda) = \nabla_{\lambda} L^{\alpha}(u^{\ast}(\lambda), \omega, \lambda)$$ and $$\nabla_{\omega} \Phi^{\alpha}(\omega, \lambda) = \nabla_{\omega} L^{\alpha}(u^{\ast}(\lambda), \omega, \lambda) = \nabla_{\omega} L^{\alpha}(u, \omega, \lambda), \quad \text{for any} \,\, u, $$ then for any $x = (\omega, \lambda)$ and $x'=(\omega', \lambda')$
\begin{align}
\left\|\nabla_{\lambda} \Phi^{\alpha}(x) - \nabla_{\lambda}  \Phi^{\alpha}(x') \right\| & = \left\|\nabla_{\lambda} L^{\alpha}(u^{\ast}({\lambda}), x) - \nabla_{\lambda} L^{\alpha}( u^{\ast}({\lambda}'), x') \right\|  \notag \\
& \mathop{\leq}^{(a)} \ell_{L} \left(\left\|x-x' \right\| + \left\|u^{\ast}(\lambda) - u^{\ast}({\lambda}') \right\| \right) \mathop{\leq}^{(b)} (\kappa+1)\ell_{L} \left\|x-x' \right\|,
\end{align}
where $\ell_{L} = \frac{5}{2}\alpha \ell_{21}$, $(a)$ uses the smoothness of $L^{\alpha}$ and $(b)$ uses the $\kappa$-Lipschitz continunity of $u^{\ast}(\lambda)$ and $\left\|\lambda - \lambda' \right\| \leq \left\| x -x' \right\|$.  We thus conclude that $\Phi^{\alpha}(\omega, \lambda)$ is $(\kappa+1)\ell_{L} $-smooth w.r.t. $\lambda$. Because $L^{\alpha}$ is $\ell_{L} $-smooth, we can conclude that $\Phi^{\alpha}(\omega, \lambda)$ is $\ell_{L}$-smooth w.r.t. $\omega$.  

{\bf For Claim (iii)}: Since $\Phi^{\alpha}(\omega, \lambda)$ is $\mu_2\alpha/2$-strongly convex with respect to $\omega$, similar to (\ref{inequ:u:1}), (\ref{inequ:u:2}), we have
\begin{subequations}
\begin{align}
\left\langle \omega - \omega_{\alpha}^{\ast}(\lambda_1), \nabla_{\omega} \Phi^{\alpha} (\omega_{\alpha}^{\ast}(\lambda_1), \lambda_1)\right\rangle & \geq 0,  \label{inequ:omega:1} \\
\left\langle \omega - \omega_{\alpha}^{\ast}(\lambda_2), \nabla_{\omega}\Phi^{\alpha}(\omega_{\alpha}^{\ast}(\lambda_2), \lambda_2)\right\rangle & \geq 0. \label{inequ:omega:2}
\end{align}
\end{subequations}
Let $\omega = \omega_{\alpha}^{\ast}(\lambda_2)$ in (\ref{inequ:omega:1}) and $\omega = \omega_{\alpha}^{\ast}(\lambda_1)$ in (\ref{inequ:omega:2}) and then sum the two inequalities, we get
\begin{align}
 \left\langle \omega_{\alpha}^{\ast}(\lambda_2) - \omega_{\alpha}^{\ast}(\lambda_1), \nabla_{\omega} \Phi^{\alpha} (\omega_{\alpha}^{\ast}(\lambda_1), \lambda_1) - \nabla_{\omega}\Phi^{\alpha}(\omega_{\alpha}^{\ast}(\lambda_2), \lambda_2) \right\rangle & \geq 0.   
\end{align}
By the strongly convexity of $\Phi^{\alpha}$ w.r.t. $\omega$ we have
\begin{align}
 \left\langle \omega_{\alpha}^{\ast}(\lambda_2) - \omega_{\alpha}^{\ast}(\lambda_1), \nabla_{\omega} \Phi^{\alpha} (\omega_{\alpha}^{\ast}(\lambda_2), \lambda_1) - \nabla_{\omega}\Phi^{\alpha}(\omega_{\alpha}^{\ast}(\lambda_1), \lambda_1) \right\rangle 
 \deleted{+}{\added{-}} \frac{\mu_2\alpha}{2} \left\|\omega_{\alpha}^{\ast}(\lambda_1)  - \omega_{\alpha}^{\ast}(\lambda_2)\right\|^2 \geq 0.
\end{align}
Summing the above two inequalities gives that 
\begin{align}
\frac{\mu_2\alpha}{2} \left\| \omega_{\alpha}^{\ast}(\lambda_1) - \omega_{\alpha}^{\ast}(\lambda_2)\right\|^2  & \leq \left\langle \omega_{\alpha}^{\ast}(\lambda_2) - \omega_{\alpha}^{\ast}(\lambda_1), \nabla_{\omega} \Phi^{\alpha}(\omega_{\alpha}^{\ast}(\lambda_2), \lambda_1) - \nabla_{\omega}\Phi^{\alpha}(\omega_{\alpha}^{\ast}(\lambda_2), \lambda_2)\right\rangle   \notag \\
& \leq \left\|\omega_{\alpha}^{\ast}(\lambda_1) - \omega_{\alpha}^{\ast}(\lambda_2) \right\| \left\|\nabla_{\omega} \Phi^{\alpha}(\omega_{\alpha}^{\ast}(\lambda_2), \lambda_1) - \nabla_{\omega}\Phi^{\alpha}(\omega_{\alpha}^{\ast}(\lambda_2), \lambda_2) \right\| \notag \\
& \leq \left(\ell_{11} + \alpha \ell_{21}\right) \left\|\omega_{\alpha}^{\ast}(\lambda_1) - \omega_{\alpha}^{\ast}(\lambda_2) \right\| \left\| \lambda_{1} - \lambda_2 \right\|.
\end{align}
Then we can conclude that $\omega_{\alpha}^{\ast}(\lambda)$ is $(2\kappa+1)$-Lipschitz continuous suppose that $\alpha \geq 2\ell_{11}/\mu_2$. Since $\omega_{\alpha}^{\ast}(\lambda)$ is unique and  from Danskin's theorem that $\Gamma^{\alpha}$ is differentiable with $\nabla \Gamma^{\alpha}(\lambda) = \nabla_{\lambda} \Phi^{\alpha}(\omega_{\alpha}^{\ast}(\lambda), \lambda)$. From Theorem~\ref{thm:equiv} (iii), there exists a constant $\ell_{\Gamma} > 0$ such that
$\left\|\nabla^2 \Gamma^{\alpha}(\lambda) \right\| \leq \ell_{\Gamma}$, that is $\Gamma^{\alpha}(\lambda)$ is $\ell_{\Gamma}$ gradient Lipschitz. We complete the proof.
\end{proof}

\begin{proof}({\bf One-Stage of Algorithm~\ref{alg:minmax:general})})

In this setting, we will prove the convergence of the one-stage Algorithm~\ref{alg:minmax:general} where $N=0$ and the penalty $\alpha$ is fixed.
 We apply $B$-batch SGD to update $\lambda$ that  $\lambda_{k+1} - \lambda_k =  - \eta^{\lambda} \nabla_{\lambda} L^{\alpha}(u_{k+1}, \omega_{k+1}, \lambda_k; D_k)$ where the batch size $|D_k|=B$ for training and validation and the stochastic gradient satisfies that \begin{align}
      \E[\nabla_{\lambda} L^{\alpha}(u_k, \omega_k, \lambda_k; D_k) \mid \mathcal{F}_k] = \nabla_{\lambda} L^{\alpha}(u_{k+1}, \omega_{k+1}, \lambda_k).
 \end{align}
By the smoothness of $\Gamma$, we have
\begin{align}\label{inequ:gamma:sto}
\Gamma^{\alpha}(\lambda_{k+1}) & \leq \Gamma^{\alpha}(\lambda_k) + \left\langle \nabla \Gamma^{\alpha}(\lambda_k), \lambda_{k+1} - \lambda_{k}\right\rangle + \frac{\ell_{\Gamma}}{2}\left\| \lambda_{k+1} - \lambda_k \right\|^2 \notag \\
& = \Gamma^{\alpha}(\lambda_k) - \eta^{\lambda} \left\langle \nabla \Gamma^{\alpha}(\lambda_k), \nabla_{\lambda} L^{\alpha}(u_{k+1}, \omega_{k+1}, \lambda_k; D_k)\right\rangle + \frac{\ell_{\Gamma}  (\eta^{\lambda})^2}{2}\left\| \nabla_{\lambda} L^{\alpha}(u_{k+1}, \omega_{k+1}, \lambda_k; D_k)\right\|^2
\end{align} 
Taking conditional expectation w.r.t. $\mathcal{F}_k$ on the above inequality and using the fact that $L_{D_k}^{\alpha}$ is an unbiased estimation of $L^{\alpha}$, we have
\begin{align}\label{eq: Gamma_lambda}
\E[\Gamma^{\alpha}(\lambda_{k+1}) \mid \mathcal{F}_k] 
& \leq \Gamma^{\alpha}(\lambda_k) - \eta_{\lambda} \left\langle \nabla \Gamma(\lambda_k), \nabla_{\lambda} L^{\alpha}(u_{k+1}, \omega_{k+1}, \lambda_k)\right\rangle \notag \\
& + \frac{\ell_{\Gamma}  (\eta^{\lambda})^2}{2}\E[\left\| \nabla_{\lambda} L^{\alpha}(u_{k+1}, \omega_{k+1}, \lambda_k; D_k)\right\|^2 \mid \mathcal{F}_k] 
\end{align}
First we turn to estimate the last term of (\ref{eq: Gamma_lambda}). 
\begin{align}\label{eq: sto_gradient:bound}
&\E[\left\| \nabla_{\lambda} L^{\alpha}(u_{k+1}, \omega_{k+1}, \lambda_k; D_k)\right\|^2 \mid \mathcal{F}_k]  \notag \\
& =     \E[\left\| \nabla_{\lambda} L^{\alpha}(u_{k+1}, \omega_{k+1}, \lambda_k; D_k) - \nabla_{\lambda} L^{\alpha}(u_{k+1}, \omega_{k+1}, \lambda_k) + \nabla_{\lambda} L^{\alpha}(u_{k+1}, \omega_{k+1}, \lambda_k) \right\|^2\mid \mathcal{F}_k ] \notag \\
& \mathop{=}^{(a)}  \E[\left\| \nabla_{\lambda} L^{\alpha}(u_{k+1}, \omega_{k+1}, \lambda_k; D_k) - \nabla_{\lambda} L^{\alpha}(u_{k+1}, \omega_{k+1}, \lambda_k) \right\|^2 \mid \mathcal{F}_k ] + \left\|\nabla_{\lambda} L^{\alpha}(u_{k+1}, \omega_{k+1}, \lambda_k) \right\|^2 \notag \\
& \mathop{\leq}^{(b)} \frac{\sigma_1^2 + 2\alpha^2\sigma_2^2}{B} + \left\|\nabla_{\lambda} L^{\alpha}(u_{k+1}, \omega_{k+1}, \lambda_k) \right\|^2
\end{align}
where $(a)$ follows from the fact  that $\E[\left\|X - \E[X] + \E[X] \right\|^2] = \E[\left\|X - \E[X]\right\|^2] + \left\|\E[X]\right\|^2$ and $(b)$ uses the following estimation of the variance term under Assumption~\ref{assumpt: bounded_variance} that
\begin{align}\label{eq: grad_lambda}
  & \E[\left\| \nabla_{\lambda} L^{\alpha}(u_{k+1}, \omega_{k+1}, \lambda_k; D_k) - \nabla_{\lambda} L^{\alpha}(u_{k+1}, \omega_{k+1}, \lambda_k) \right\|^2 \mid \mathcal{F}_k ] \notag \\
  & =   \E[\left\| \nabla_{\lambda} L_1(\omega_{k+1}, \lambda_k; S_{val}^k) - \nabla_{\lambda} L_1(\omega_{k+1}, \lambda_k) \right\|^2 \mid \mathcal{F}_k ]   +  \alpha^2 \E[ \left\| \nabla_{\lambda} L_2(\omega_{k+1}, \lambda_k; S_{train}^k) - \nabla_{\lambda} L_2(\omega_{k+1}, \lambda_k) \right\|^2\mid \mathcal{F}_k ] \notag \\
  & \quad  + \alpha^2 \E[ \left\| \nabla_{\lambda} L_2(u_{k+1}, \lambda_{k}; S_{train}^k) - \nabla_{\lambda} L_2(u_{k+1}, \lambda_k) \right\|^2\mid \mathcal{F}_k ] \notag \\
  & \leq \frac{\sigma_1^2 + 2\alpha^2\sigma_2^2}{B}.
\end{align}
Applying (\ref{eq: sto_gradient:bound}) into (\ref{eq: Gamma_lambda}), we have
\begin{align}\label{eq: Gamma_lambda_2}
\E[\Gamma^{\alpha}(\lambda_{k+1}) \mid \mathcal{F}_k] 
& \leq \Gamma^{\alpha}(\lambda_k) -  \eta^{\lambda} \left\langle \nabla \Gamma^{\alpha}(\lambda_k), \nabla_{\lambda} L^{\alpha}(u_{k+1}, \omega_{k+1}, \lambda_k)\right\rangle  + \frac{\ell_{\Gamma}  (\eta^{\lambda})^2}{2}\left\|\nabla_{\lambda} L^{\alpha}(u_{k+1}, \omega_{k+1}, \lambda_k) \right\|^2 \notag \\
& + \frac{\ell_{\Gamma}  (\eta^{\lambda})^2}{2B}\left(\sigma_1^2 + 2\alpha^2\sigma_2^2\right) \notag \\
&= \Gamma^{\alpha}(\lambda_k) - \frac{\eta^{\lambda}}{2} \left\|\nabla \Gamma^{\alpha}(\lambda_k) \right\|^2 - \left(\frac{\eta^{\lambda}}{2} - \frac{\ell_{\Gamma}  (\eta^{\lambda})^2}{2}\right)\left\| \nabla_{\lambda} L^{\alpha}(u_{k+1}, \omega_{k+1}, \lambda_k)\right\|^2 \notag \\
& + \frac{\eta^{\lambda}}{2}\left\| \nabla \Gamma^{\alpha}(\lambda_k) - \nabla_{\lambda} L^{\alpha}(u_{k+1}, \omega_{k+1}, \lambda_k)\right\|^2 +  \frac{\ell_{\Gamma}  (\eta^{\lambda})^2}{2B}\left(\sigma_1^2 + 2\alpha^2\sigma_2^2\right).
\end{align}
If $\eta^{\lambda} \leq \frac{1}{\ell_{\Gamma}}$, then
\begin{align}\label{eq: grad_Gamma_core_bound}
    \frac{1}{K}\sum_{k=0}^{K-1}\E[\left\|\nabla \Gamma^{\alpha}(\lambda_k) \right\|^2] & \leq \frac{2\E[\Gamma^{\alpha}(\lambda_0)] -2\Gamma_{\min}^{\alpha} }{K\eta^{\lambda}} +  + \frac{\ell_{\Gamma}  (\eta^{\lambda})}{B}\left(\sigma_1^2 + 2\alpha^2\sigma_2^2\right) \notag \\
    & + \frac{1}{K}\sum_{k=0}^{K-1}\left\| \nabla \Gamma^{\alpha}(\lambda_k) - \nabla_{\lambda} L^{\alpha}(u_{k+1}, \omega_{k+1}, \lambda_k)\right\|^2
\end{align}
where $\Gamma_{\min}^{\alpha} = \min_{\lambda} \Gamma^{\alpha}(\lambda)$. First, we show that
the difference between $\Gamma^{\alpha}(\lambda_0) $ and $ \Gamma_{\min}^{\alpha}$ can be controlled by a constant which is independent with $\alpha$ as below:
\begin{align}\label{eq: bound_Gamma_init_min}
    \Gamma^{\alpha}(\lambda_{0}) - \Gamma_{\min}^{\alpha} & = L^{\alpha}(u^{\ast}(\lambda_0), \omega_{\alpha}^{\ast}(\lambda), \lambda_0) - L^{\alpha}(u^{\ast}(\lambda^{\ast}), \omega_{\alpha}^{\ast}(\lambda^{\ast}), \lambda^{\ast}) \notag \\
    & = L_1(\omega_{\alpha}^{\ast}(\lambda_0), \lambda_0) - L_1(\omega_{\alpha}^{\ast}(\lambda^{\ast}), \lambda^{\ast}) + \alpha \left(L_2(\omega_{\alpha}^{\ast}(\lambda_0), \lambda_0) - L_2(u^{\ast}(\lambda_0), \lambda_0) \right) \notag \\
    & + \alpha \left(L_2(\omega_{\alpha}^{\ast}(\lambda^{\ast}), \lambda^{\ast}) - L_2(u^{\ast}(\lambda^{\ast}), \lambda^{\ast}) \right) \notag \\
    & \leq L_1(\omega_{\alpha}^{\ast}(\lambda_0), \lambda_0) - L_1(\omega_{\alpha}^{\ast}(\lambda^{\ast}), \lambda^{\ast}) + \alpha \frac{\ell_{21}}{2}\left\|\omega_{\alpha}^{\ast}(\lambda_0) -  u^{\ast}(\lambda_0) \right\|^2 \notag \\
    & + \alpha \frac{\ell_{21}}{2}\left\|\omega_{\alpha}^{\ast}(\lambda^{\ast}) -  u^{\ast}(\lambda^{\ast}) \right\|^2 \notag \\
    & \leq L_1(\omega_{\alpha}^{\ast}(\lambda_0), \lambda_0) - L_1(\omega_{\alpha}^{\ast}(\lambda^{\ast}), \lambda^{\ast}) + \frac{\ell_{21}\ell_{10}^2}{2\mu_2^2\alpha} + \frac{\ell_{21}\ell_{10}^2}{2\mu_2^2\alpha} \notag \\
    & \leq \ell_{10}\left(\left\|\omega_{\alpha}^{\ast}(\lambda_0) - \omega_{\alpha}^{\ast}(\lambda^{\ast}) \right\| + \left\|\lambda_0 - \lambda^{\ast}\right\|\right) + \frac{2\ell_{21}\ell_{10}^2}{2\mu_2^2\alpha} \notag \\
    & \leq \ell_{10}\left(1 + \frac{3\ell_{21}}{\mu_2} \right)\left\|\lambda_0 - \lambda^{\ast}\right\| + \frac{2\ell_{21}\ell_{10}^2}{2\mu_2^2\alpha}  \leq \mathcal{O}\left(\kappa \left\|\lambda_0 - \lambda^{\ast}\right\| + \frac{\kappa}{2}\right) 
\end{align}
We next turn to estimate the approximation of $\nabla_{\lambda} L^{\alpha}(u_{k+1}, \omega_{k+1}, \lambda_k)$ to $\nabla \Gamma^{\alpha}(\lambda_k)$:
\begin{align}\label{eq: grad_L_Gamma}
  & \left\| \nabla \Gamma^{\alpha}(\lambda_k) - \nabla_{\lambda} L^{\alpha}(u_{k+1}, \omega_{k+1}, \lambda_k)\right\|^2 =   \left\| \nabla_{\lambda} L^{\alpha}(u^{\ast}(\lambda_k), \omega_{\alpha}^{\ast}(\lambda_k), \lambda_k) - \nabla_{\lambda} L^{\alpha}(u_{k+1}, \omega_{k+1}, \lambda_k)\right\|^2 \notag \\
  & \mathop{\leq}^{(a)} 3\left\|\nabla_{\lambda} L_1(\omega_{\alpha}^{\ast}(\lambda_k), \lambda_k) - \nabla_{\lambda} L_1(\omega_{k+1}, \lambda_k)\right\|^2 + 3\alpha^2\left\|\nabla_{\lambda} L_2(\omega_{\alpha}^{\ast}(\lambda_k), \lambda_k)  - \nabla_{\lambda} L_2(\omega_{k+1}, \lambda_k) \right\|^2 \notag \\
  & \quad + 3\alpha^2 \left\|\nabla_{\lambda} L_2(u^{\ast}(\lambda_k), \lambda_k) - \nabla_{\lambda} L_2(u_{k+1}, \lambda_k) \right\|^2 \notag \\
  & \mathop{\leq}^{(b)} 3(\ell_{11}^2 + \alpha^2 \ell_{21}^2) \left\| \omega_{\alpha}^{\ast}(\lambda_k) - \omega_{k+1}\right\|^2 + 3\alpha^2 \ell_{21}^2 \left\| u^{\ast}(\lambda_k) - u_{k+1} \right\|^2
\end{align}
where $(a)$ uses the Cauchy-Schwartz inequality that $(a+b+c)^2 \leq 3(a^2 + b^2 +c^2)$ and $(b)$ uses the gradient-Lipschitz properties of $L_1$ and $L_2$.

 Then, the focus is to estimate $\left\|u_{k+1} - u^{\ast}(\lambda_k) \right\|^2$ and $\left\| \omega_{k+1} - \omega_{\alpha}^{\ast}(\lambda_k) \right\|^2$.  For $u$, we use $B$-batch SGD for running $T_k$ iterations. By the strongly concavity of $L^{\alpha}$ with respect to $u$, if $\eta^u \leq \frac{2}{\alpha(\ell_{21} + \mu_2)}$, then
\begin{align}\label{inequ:uk:stoc}
 & \E[\left\|  u^{\ast}( \lambda_{k}) - \tilde{u}_{t+1}  \right\|^2 \mid \mathcal{F}_{k,t}] \notag \\
 &= \E[\left\|u^{\ast}(\lambda_{k}) - \tilde{u}_t - \eta^{u} \nabla_{u} L^{\alpha}(\tilde{u}_t, \tilde{\omega}_t, \lambda_k; D_{k,t}) \right\|^2  \mid \mathcal{F}_{k,t}] \notag \\
  & \leq \left\|u^{\ast}( \lambda_{k}) - \tilde{u}_t \right\|^2 - 2\eta^{u}\left\langle u^{\ast}( \lambda_{k}) - \tilde{u}_t,  \E[\nabla_{u} L^{\alpha}(\tilde{u}_t, \tilde{\omega}_t, \lambda_k; D_{k,t}) \mid \mathcal{F}_{k,t}]\right\rangle \notag \\
    & \quad + (\eta^{u})^2 \E[\left\| \nabla_{u} L^{\alpha}(\tilde{u}_t, \tilde{\omega}_t, \lambda_k; D_{k,t}) \right\|^2 \mid \mathcal{F}_{k,t}] \notag \\
    & \leq \left\|u^{\ast}( \lambda_{k}) - \tilde{u}_t \right\|^2 - 2\eta^{u}\left\langle u^{\ast}( \lambda_{k}) - \tilde{u}_t,  \nabla_{u} L^{\alpha}(\tilde{u}_t, \tilde{\omega}_t, \lambda_k)\right\rangle  + (\eta^{u})^2\E\left[\left\| \nabla_{u} L^{\alpha}(\tilde{u}_t, \omega_k^t, \lambda_k; D_{k,t}) \right\|^2 \mid \mathcal{F}_{k,t}\right] \notag \\
        & \leq \left\|u^{\ast}( \lambda_{k}) - \tilde{u}_t \right\|^2 - 2\eta^{u} \alpha \left\langle u^{\ast}( \lambda_{k}) - \tilde{u}_t,  -\nabla_{u} L_2(\tilde{u}_t, \lambda_k)\right\rangle + (\eta^{u})^2\alpha^2\E\left[\left\| \nabla_{u} L_2(\tilde{u}_t, \lambda_k; S_{train}^{k,t}) \right\|^2 \mid \mathcal{F}_{k,t}\right]  \notag \\
  & \mathop{\leq}^{(a)} \left\|u^{\ast}( \lambda_{k}) - \tilde{u}_t \right\|^2 - 2\alpha \eta^u\left(\frac{\ell_{21}\mu_2}{\mu_2+\ell_{21}} \left\|\tilde{u}_t - u^{\ast}(\lambda_k) \right\|^2 + \frac{1}{\mu_2+\ell_{21}}\left\| \nabla_{u} L_2(\tilde{u}_t, \lambda_k)\right\|^2\right) \notag \\ 
  & \quad + (\eta^{u})^2\alpha^2\E\left[\left\| \nabla_{u} L_2(\tilde{u}_t, \lambda_k; S_{train}^{k,t}) \right\|^2 \mid \mathcal{F}_{k,t}\right]  \notag \\
  & \mathop{=}^{(b)}  \left\|u^{\ast}( \lambda_{k}) - \tilde{u}_t \right\|^2 - 2\alpha \eta^u\left(\frac{\ell_{21}\mu_2}{\mu_2+\ell_{21}} \left\|\tilde{u}_t - u^{\ast}(\lambda_k) \right\|^2 + \frac{1}{\mu_2+\ell_{21}}\left\| \nabla_{u} L_2(\tilde{u}_t, \lambda_k)\right\|^2\right) \notag \\ 
  & \quad + (\eta^u)^2\alpha^2\E[\left\|\nabla_{u} L_2(\tilde{u}_t, \lambda_k; S_{train}^k) - \nabla_{u} L_2(\tilde{u}_t, \lambda_k) \right\|^2 \mid \mathcal{F}_{k,t}] +(\eta^u)^2\alpha^2  \left\| \nabla_{u} L_2(\tilde{u}_t, \lambda_k) \right\|^2 \notag \\
& \mathop{\leq}^{(c)}   \left\|u^{\ast}( \lambda_{k}) - \tilde{u}_t \right\|^2 - 2\alpha \eta^u\left(\frac{\ell_{21}\mu_2}{\mu_2+\ell_{21}} \left\|\tilde{u}_t - u^{\ast}(\lambda_k) \right\|^2 + \frac{1}{\mu_2+\ell_{21}}\left\| \nabla_{u} L_2(\tilde{u}_t, \lambda_k)\right\|^2\right) \notag \\ 
  & \quad +  \frac{(\eta^u)^2 \alpha^2\sigma_2^2}{B}+  (\eta^u)^2\alpha^2  \left\| \nabla_{u} L_2(\tilde{u}_t, \lambda_k) \right\|^2\notag \\
  & \mathop{\leq} \left(1-\mu_2\alpha \eta^u \right)^2 \left\|u^{\ast}(\lambda_k) - \tilde{u}_t \right\|^2 + \frac{(\eta^u)^2 \alpha^2\sigma_2^2}{B}.
\end{align}
where $(a)$ follows from the property of any $\gamma_1$-strong convexity and $\gamma_2$-smoothness function $f$ which implies that
\begin{align}
\left\langle \nabla f(x_k),  x_k - x^{\ast}\right\rangle \geq \frac{\gamma_1\gamma_2}{\gamma_1 + \gamma_2}\left\| x_k - x^{\ast}\right\|^2 + \frac{1}{\gamma_1+\gamma_2}\left\| \nabla f(x_k)\right\|^2  \notag,
\end{align}
with $x^{\ast}=\arg\min f(x)$;
(b) uses the relationship $\E[\nabla_u L_2(u_k, \lambda_k; S_{train}^k)\mid \mathcal{F}_k] = \nabla_u L_2(u_k, \lambda_k)$ which induces that 
\begin{align}
  &  \E[\left\|\nabla_{u} L_2(u_k, \lambda_k; S_{train}^k)\right\|^2  \mid \mathcal{F}_k] \notag \\
  & =   \E[\left\|\nabla_{u} L_2(u_k, \lambda_k; S_{train}^k) - \nabla_{u} L^{\alpha}(u_k,\lambda_k) \right\|^2 \mid \mathcal{F}_k] + \left\| \nabla_{u} L^{\alpha}(u_k,  \lambda_k) \right\|^2
\end{align}
and $(c)$ uses Assumption~\ref{assumpt: bounded_variance} that $\E[\left\|\nabla_{u} L_2(u_k,\lambda_k; S_{train}^k) - \nabla_{u} L_2(u_k, \lambda_k) \right\|^2 \mid \mathcal{F}_k] \leq \sigma_2^2$. For $t=0,1,\cdots, T_k-1$, we have
\begin{align}\label{eq: uast_k1_lambda}
 \E[\left\|  u^{\ast}( \lambda_{k}) - u_{k+1}  \right\|^2] & := \E[\left\|  u^{\ast}( \lambda_{k}) - \tilde{u}_{T_k}  \right\|^2]  \leq \left(1-\mu_2\alpha \eta^u \right)^2 \left\|u^{\ast}(\lambda_k) - \tilde{u}_{T_k-1} \right\|^2 + \frac{\alpha^2(\eta^u)^2\sigma_2^2}{B} \notag\\
 & \leq \left(1-\mu_2\alpha \eta^u \right)^{2T_k} \left\|u^{\ast}(\lambda_k) - \tilde{u}_{0} \right\|^2 + \frac{\alpha^2(\eta^u)^2\sigma_2^2}{B} \sum_{t=0}^{T_k-1}\left(1-\mu_2\alpha \eta^u \right)^{2t} \notag \\
 & \leq \left(1-\mu_2\alpha \eta^u \right)^{2T_k} \left\|u^{\ast}(\lambda_k) - u_k \right\|^2 + \frac{\alpha \eta^u\sigma_2^2}{\mu_2 B}.
\end{align}
where $\tilde{u}_0 = u_k$.

In order to achieve that $
\E[\left\|\nabla_{\lambda} L^{\alpha}(u_{k+1}, \omega_{k+1}, \lambda_k) - \nabla \Gamma^{\alpha}(\lambda_k) \right\|^2 ] \leq \zeta^2 $ for all $k \leq K-1$. According to~(\ref{eq: grad_L_Gamma}), we can prove that $\E[\left\|u_{k+1} - u^{\ast}(\lambda_k)\right\|^2] \leq \frac{\zeta^2}{6\alpha^2\ell_{21}^2}$ and $\E[\left\|\omega_{k+1} - \omega_{\alpha}^{\ast}(\lambda_k)\right\|^2] \leq \frac{\zeta^2}{6(\ell_{11}^2 + \alpha^2\ell_{21}^2)}$ for all $k \leq K-1$.
First we show how to control $\E[\left\| u_{k+1} - u^{\ast}(\lambda_k) \right\|]$ for all $k$. By (\ref{eq: uast_k1_lambda}), we have
\begin{align}
   \E[\left\|  u^{\ast}( \lambda_{k}) - u_{k+1}  \right\|^2] \leq   \left(1-\mu_2\alpha \eta^u \right)^{2T_k} \left\|u^{\ast}(\lambda_k) - u_k \right\|^2 + \frac{\alpha \eta^u\sigma_2^2}{\mu_2 B}
\end{align}
Suppose that we set 
\begin{align}\label{eq: bound_Tk_u}
 T_k & \geq \frac{\kappa-1}{4}\ln\left(\frac{12\alpha^2 \ell_{21}^2\left\|u^{\ast}(\lambda_k) - u_{k+1} \right\|^2}{\zeta^2 }\right) \geq \frac{\ln\left(\frac{12\alpha^2 \ell_{21}^2\E[\left\|u^{\ast}(\lambda_k) - u_k \right\|^2]}{\zeta^2 }\right)}{2\ln\left(\frac{\kappa+1}{\kappa-1}\right)}\notag \\
 B & =   \frac{12\kappa(\sigma_1^2 + \alpha^2\sigma_2^2)}{\zeta^2}
\end{align}
then
\begin{align}
   \E[\left\| u_{k+1} - u^{\ast}( \lambda_{k})  \right\|^2] \leq  \frac{\zeta^2}{6\alpha^2\ell_{21}^2}. 
\end{align}
Next we turn to bound $\left\|u^{\ast}(\lambda_k) - u_k \right\|$ for all $t$. Suppose that 
\begin{align}
\E[\left\|u_k - u^{\ast}(\lambda_{k-1})\right\|^2] \leq \frac{\zeta^2}{6\alpha^2\ell_{21}^2}
\end{align}
For $t \geq 1$, we have
\begin{align}
\E[\left\|u_k - u^{\ast}(\lambda_k)\right\|^2 ] & =   \E[\left\|u_k - u^{\ast}(\lambda_{k-1}) +  u^{\ast}(\lambda_{k-1}) - u^{\ast}(\lambda_{k})\right\|^2] \notag \\
& \leq 2 \E[\left\|u_k - u^{\ast}(\lambda_{k-1})\right\|^2] + 2 \E[\left\| u^{\ast}(\lambda_{k-1}) - u^{\ast}(\lambda_{k})\right\|^2] \notag \\
& \leq \frac{\zeta}{3\alpha^2\ell_{21}^2} + 2 \kappa^2 \E[\left\| \lambda_k - \lambda_{k-1} \right\|^2] \notag \\
& = \frac{\zeta}{3\alpha^2\ell_{21}^2} + 2 \kappa^2 (\eta^{\lambda})^2\E[\left\| \nabla_{\lambda} L^{\alpha}(u_{k}, \omega_{k}, \lambda_{k-1}; D_{k-1}) \right\|^2 ] \notag \\
& \leq \frac{\zeta}{3\alpha^2\ell_{21}^2} + 2 \kappa^2 (\eta^{\lambda})^2 \left(\left\| \nabla_{\lambda} L^{\alpha}(u_{k}, \omega_{k}, \lambda_{k-1}) \right\|^2 + \frac{\sigma_1^2 + \alpha^2\sigma_2^2}{B} \right) \notag \\
& \leq \frac{\zeta}{3\alpha^2\ell_{21}^2} + 2 \kappa^2 (\eta^{\lambda})^2 \left(2\zeta^2 + 2\left\| \nabla \Gamma^{\alpha}(\lambda_{k-1}) \right\|^2 + \frac{\sigma_1^2 + \alpha^2\sigma_2^2}{B} \right) 
\end{align}
where the first inequality uses the Cauchy-Schwartz inequality; the second inequality use the Lipschitz continuity of $u^{\ast}(\lambda)$; and the third inequality uses the inequality (\ref{eq: sto_gradient:bound}); and the last inequaty uses the fact that $\left\|\nabla_{\lambda} L^{\alpha}(u_{k}, \omega_{k}, \lambda_{k-1}) - \nabla \Gamma^{\alpha}(\lambda_{k-1}) \right\| \leq \zeta$ which implies that $\left\|\nabla_{\lambda} L^{\alpha}(u_{k}, \omega_{k}, \lambda_{k-1})\right\|^2 \leq 2 \zeta^2 + 2\left\| \nabla \Gamma^{\alpha}(\lambda_{k-1}) \right\|^2$.

We then recall the definition of $\nabla \Gamma^{\alpha}(\lambda_{k-1})$ and make the following estimation
\begin{align}\label{eq: grad_Gamma_lambda-1}
  \left\| \nabla \Gamma^{\alpha}(\lambda_{k-1}) \right\|^2 & = \left\|\nabla_{\lambda} L^{\alpha}(u^{\ast}(\lambda_{k-1}), \omega_{\alpha}^{\ast}(\lambda_{k-1}), \lambda_{k-1}) \right\|^2 \notag \\
  & \mathop{\leq}^{(a)} 2\ell_{11}^2 + 2\alpha^2\left\|u^{\ast}(\lambda_{k-1}) - \omega_{\alpha}^{\ast}(\lambda_k) \right\|^2  \notag \\
  & \leq  2\ell_{11}^2 + 2\alpha^2\left\|\omega^{\ast}(\lambda_{k-1}) - \omega_{\alpha}^{\ast}(\lambda_k) \right\|^2 \leq 2\ell_{11}^2 + \frac{2\alpha^2\kappa^2}{\alpha^2} = 2\ell_{11}^2 + 2\kappa^2
\end{align}
where $(a)$ uses the Cauchy-Schwartz inequality and $(b)$ follows from the result of Lemma \ref{lem:omega:ast}.
Thus, for $t \geq 1$, we have
\begin{align}
\E[\left\|u_k - u^{\ast}(\lambda_k)\right\|^2 ] & \leq  \frac{\zeta}{3\alpha^2\ell_{21}^2} + 2 \kappa^2 (\eta^{\lambda})^2 \left(2\zeta^2 + 4\ell_{11}^2 + 4\kappa^2 + \frac{\sigma_1^2 + \alpha^2\sigma_2^2}{B} \right) 
\end{align}
In order to achieve that $\E[\left\| u_{k+1} - u^{\ast}( \lambda_{k})  \right\|^2] \leq  \frac{\zeta^2}{6\alpha^2\ell_{21}^2}$, we set
\begin{align}
 T_k & \geq \frac{\kappa-1}{4}\ln\left(\frac{12\alpha^2 \ell_{21}^2\E[\left\|\omega_{\alpha}^{\ast}(\lambda_k) - \omega_{k} \right\|^2]}{\zeta^2 }\right) \geq \frac{\ln\left(\frac{12\alpha^2 \ell_{21}^2\E[\left\|u^{\ast}(\lambda_k) - u_k \right\|^2]}{\zeta^2 }\right)}{2\ln\left(\frac{\kappa+1}{\kappa-1}\right)}\notag \\
 B & =   \frac{12\kappa(\sigma_1^2 + \alpha^2\sigma_2^2)}{\zeta^2}
\end{align}
where 
$$\E[\left\|u_k - u^{\ast}(\lambda_k)\right\|^2 ] \leq   \left\{ 
\begin{aligned}
  &\frac{\zeta}{3\alpha^2\ell_{21}^2} + 2 \kappa^2 (\eta^{\lambda})^2 \left(2\zeta^2 + 4\ell_{11}^2 + 4\kappa^2 + \frac{\sigma_1^2 + \alpha^2\sigma_2^2}{B} \right),  &  t \geq 1  \\
& \left\|u_0 - u^{\ast}(\lambda_0)\right\|^2  & t=0
\end{aligned}
\right.
$$
Similarly, we make the estimation about $\left\|\omega_{k+1} - \omega_{\alpha}^{\ast}(\lambda_{k}) \right\|^2$. Due to that $\Phi^{\alpha}$ is $\frac{\mu_2\alpha}{2}$-strongly convex and $\frac{3}{2}\alpha\ell_{21}$-smooth with respect to $\omega$,  if $\eta^{\omega} \leq \frac{4}{\alpha(\mu_2+3\ell_{21})}$, we have
\begin{align}
    \left\|\omega_{k+1} -\omega_{\alpha}^{\ast}(\lambda_k)  \right\|^2 \leq \left( 1- \frac{1}{2}\mu_2\alpha \eta^{\omega}\right)^{2T_k} \E[\left\|\omega_{k} - \omega_{\alpha}^{\ast}(\lambda_k) \right\|^2]  + \frac{4\eta^{\omega}}{3\mu_2\alpha B}\left(\sigma_1^2 + \sigma_2^2\alpha^2 \right).
\end{align}
By properly choosing $T_k$ and $B$
\begin{align}\label{eq: bound_Tk_omega}
T_k & \geq  \frac{3\kappa-1}{4}\ln\left(\frac{12(\ell_{11}^2 + \alpha^2\ell_{21}^2)\E[\left\|\omega_k - \omega_{\alpha}^{\ast}(\lambda_k)\right\|^2]}{\zeta^2}\right)\notag \\
B & =  \frac{4\kappa \left(\frac{1}{2\alpha} + 1\right)(\sigma_1^2 + \alpha^2\sigma_2^2)}{\zeta^2}
\end{align}
then we can achieve $ \left\|\omega^{\ast}(\lambda_k) - \omega_{k+1} \right\|^2 \leq \frac{\zeta^2}{6(\ell_{11}^2 + \alpha^2\ell_{21}^2)}$. Next, we turn to estimate $\E[\left\|\omega_k - \omega_{\alpha}^{\ast}(\lambda_k)\right\|^2]$  for all $t \geq 1$ which is similar to bound $\E[\left\|u_k - u^{\ast}(\lambda_k)\right\|^2]$. For $t=0$,
\begin{align}
\E[\left\|\omega_0 - \omega_{\alpha}^{\ast}(\lambda_0)\right\|^2] & = \E[\left\|\omega_0 - \omega^{\ast}(\lambda_0) + \omega^{\ast}(\lambda_0)-\omega_{\alpha}^{\ast}(\lambda_0) \right\|^2] \notag \\
& \leq 2\E[\left\|\omega_0 - \omega^{\ast}(\lambda_0)\right\|^2] + 2\E[\left\| \omega^{\ast}(\lambda_0)-\omega_{\alpha}^{\ast}(\lambda_0) \right\|^2] \notag \\
& \leq  2\E[\left\|\omega_0 - \omega^{\ast}(\lambda_0)\right\|^2] + \frac{2\kappa^2}{\alpha^2}
\end{align}
where the last inequality applies the result of Lemma~\ref{lem:omega:ast}.
Suppose that 
\begin{align}
    \E[\left\|\omega_k - \omega_{\alpha}^{\ast}(\lambda_{k-1}) \right\|] \leq \frac{\zeta^2}{6(\ell_{11}^2 + \alpha^2\ell_{21}^2)},
\end{align}
holds at index $k-1$, then for $k\geq 1$, we have
\begin{align}
 \E[\left\| \omega_k - \omega_{\alpha}^{\ast}(\lambda_k) \right\|^2] & =   \E[\left\| \omega_k - \omega_{\alpha}^{\ast}(\lambda_{k-1}) + \omega_{\alpha}^{\ast}(\lambda_{k-1}) -  \omega_{\alpha}^{\ast}(\lambda_k)\right\|^2]   \notag \\
 & \leq 2\E[\left\| \omega_k - \omega_{\alpha}^{\ast}(\lambda_{k-1}) \right\|^2] + 2\E[\left\|\omega_{\alpha}^{\ast}(\lambda_{k-1}) -  \omega_{\alpha}^{\ast}(\lambda_k)\right\|^2]  \notag \\
 & \leq  \frac{\zeta^2}{3(\ell_{11}^2 + \alpha^2\ell_{21}^2)} + 2 \ell_{\omega^{\ast}}^2 \E[\left\|\lambda_k - \lambda_{k-1} \right\|^2] \notag \\
  & \leq  \frac{\zeta^2}{3(\ell_{11}^2 + \alpha^2\ell_{21}^2)} + 2 \ell_{\omega^{\ast}}^2(\eta^{\lambda})^2 \E[\left\|\nabla_{\lambda} L^{\alpha}(u_k, \omega_k, \lambda_{k-1}; D_{k-1}) \right\|^2] \notag \\
  & \leq  \frac{\zeta^2}{3(\ell_{11}^2 + \alpha^2\ell_{21}^2)} + 2 \ell_{\omega^{\ast}}^2(\eta^{\lambda})^2 \left(\left\|\nabla_{\lambda} L^{\alpha}(u_k, \omega_k, \lambda_{k-1}\right\|^2 + \frac{\sigma_1^2 + \alpha^2\sigma_2^2}{B}\right) \notag \\
  & \leq \frac{\zeta^2}{3(\ell_{11}^2 + \alpha^2\ell_{21}^2)} + 2 \ell_{\omega^{\ast}}^2(\eta^{\lambda})^2 \left(2\zeta^2 + 2\left\|\nabla \Gamma^{\alpha}(\lambda_{k-1})\right\|^2 + \frac{\sigma_1^2 + \alpha^2\sigma_2^2}{B}\right)  \notag \\
  & = \frac{\zeta^2}{3(\ell_{11}^2 + \alpha^2\ell_{21}^2)} + 2 \ell_{\omega^{\ast}}^2(\eta^{\lambda})^2 \left(2\zeta^2 + 4(\ell_{11}^2 + \kappa^2)+ \frac{\sigma_1^2 + \alpha^2\sigma_2^2}{B}\right)
\end{align}
where the last inequality follows from (\ref{eq: grad_Gamma_lambda-1}).

Overall, suppose that $\E[\left\|u_{k} - u^{\ast}(\lambda_{k-1})\right\|^2] \leq \frac{\zeta^2}{6\alpha^2\ell_{21}^2}$ and $\E[\left\|\omega_{k} - \omega_{\alpha}^{\ast}(\lambda_{k-1})\right\|^2] \leq \frac{\zeta^2}{6(\ell_{11}^2 + \alpha^2\ell_{21}^2)}$ hold at index $k-1$, combining the bounds of (\ref{eq: bound_Tk_u}) and (\ref{eq: bound_Tk_omega}) for $T_k$ and $B$ and properly choosing 
\begin{subequations}
\begin{align}
T_k  & \geq \frac{3\kappa-1}{4}\ln\left(\frac{12(\ell_{11}^2 + \alpha^2\ell_{21}^2)\max \left\lbrace \E[\delta_k^2], \E[r_k^2]\right\rbrace}{\zeta^2}\right) \label{eq: Tk_max}\\
B & =  \frac{12\kappa \left(\frac{1}{2\alpha} + 1\right)(\sigma_1^2 + \alpha^2\sigma_2^2)}{\zeta^2} \label{eq: B_max}
\end{align}
\end{subequations}
where
$$\max \left\lbrace \E[\delta_k^2], \E[r_k^2]\right\rbrace \leq  \left\{ 
\begin{aligned}
  &\frac{\zeta^2}{3\alpha^2\ell_{21}^2} + 2\ell_{\omega^{\ast}}^2 (\eta^{\lambda})^2 \left(2\zeta^2 + 4\ell_{11}^2 + 4\kappa^2 + \frac{\sigma_1^2 + \alpha^2\sigma_2^2}{B} \right),  &  k \geq 1  \\
& \max \left\lbrace \left\|u_0 - u^{\ast}(\lambda_0)\right\|^2, 2\left\|\omega_0 - \omega^{\ast}(\lambda_0)\right\|^2 + \frac{2\kappa^2}{\alpha^2}\right\rbrace  & k=0
\end{aligned}
\right.
$$
We can achieve that $\E[\left\|u_{k+1} - u^{\ast}(\lambda_{k})\right\|^2] \leq \frac{\zeta^2}{6\alpha^2\ell_{21}^2}$ and $\E[\left\|\omega_{k+1} - \omega_{\alpha}^{\ast}(\lambda_{k})\right\|^2] \leq \frac{\zeta^2}{6(\ell_{11}^2 + \alpha^2\ell_{21}^2)}$ hold at index $k$. Proof by induction, $\E[\left\|u_{k+1} - u^{\ast}(\lambda_{k})\right\|^2] \leq \frac{\zeta^2}{6\alpha^2\ell_{21}^2}$ and $\E[\left\|\omega_{k+1} - \omega_{\alpha}^{\ast}(\lambda_{k})\right\|^2] \leq \frac{\zeta^2}{6(\ell_{11}^2 + \alpha^2\ell_{21}^2)}$ hold for all $k \leq K-1$. Thus, By~(\ref{eq: grad_L_Gamma}), we demonstrate that 
$\E[\left\| \nabla_{\lambda} L^{\alpha}(u_{k+1}, \omega_{k+1}, \lambda_{k}; D_{k}) - \nabla \Gamma^{\alpha}(\lambda_k) \right\|^2] \leq \zeta^2$ for all $k \leq K-1$. 

Finally, we apply the bound for $\E[\left\| \nabla_{\lambda} L^{\alpha}(u_{k+1}, \omega_{k+1}, \lambda_{k}; D_{k}) - \nabla \Gamma^{\alpha}(\lambda_k) \right\|^2]$ and substitude (\ref{eq: bound_Gamma_init_min}) into (\ref{eq: grad_Gamma_core_bound}), we have 
\begin{align}
    \frac{1}{K}\sum_{k=0}^{K-1}\E[\left\|\nabla \Gamma^{\alpha}(\lambda_k) \right\|^2] & \leq \frac{2\E[\Gamma^{\alpha}(\lambda_0)] -2\Gamma_{\min}^{\alpha} }{K\eta^{\lambda}} +  + \frac{\ell_{\Gamma}  (\eta^{\lambda})}{B}\left(\sigma_1^2 + 2\alpha^2\sigma_2^2\right) \notag \\
    & + \frac{1}{K}\sum_{k=0}^{K-1}\left\| \nabla \Gamma^{\alpha}(\lambda_k) - \nabla_{\lambda} L^{\alpha}(u_{k+1}, \omega_{k+1}, \lambda_k)\right\|^2  \notag \\
    & \leq \frac{\mathcal{O}(\kappa \left\|\lambda_0 - \lambda^{\ast}\right\| + \kappa)}{K\eta^{\lambda}} + \zeta^2 +  \frac{\ell_{\Gamma}\eta^{\lambda}(\sigma_1^2 + \alpha^2\sigma_2^2)}{B}
\end{align}
That is to say, given $\alpha \geq 2\ell_{11}/\mu_2$ and the three step-sizes satisfy that
\begin{align}
  \eta^{\lambda}=\frac{1}{\ell_{\Gamma}}; \quad  \eta^{u}=\frac{2}{\alpha(\mu_2 + \ell_{21})}; \quad \eta^{\omega} = \frac{4}{\alpha(\mu_2 + 3\ell_{21})}
\end{align}
\begin{itemize}
    \item let $\alpha=\mathcal{O}(\kappa^3\epsilon^{-1})$ and  $\zeta=\mathcal{O}(\epsilon)$, by properly choosing $T_k$ and $B$ suggested by (\ref{eq: Tk_max}) and (\ref{eq: B_max}), after $K=\mathcal{O}(\kappa^4\epsilon^{-2})$ steps, we can reach $ \frac{1}{K}\sum_{k=0}^{K-1}\E[\left\|\nabla \Gamma^{\alpha}(\lambda_k) \right\|^2] \leq \epsilon^2$.
\begin{align}
& \frac{1}{K}\sum_{k=0}^{K-1}\E[\left\|\nabla \mathcal{L}(\lambda_k) \right\|^2]  =  \frac{1}{K}\sum_{k=0}^{K-1}\E[\left\|\nabla \mathcal{L}(\lambda_k) - \nabla \Gamma^{\alpha}(\lambda_k) + \nabla \Gamma^{\alpha}(\lambda_k)\right\|^2] \notag \\
& \leq  \frac{2}{K}\sum_{k=0}^{K-1}\E[\left\|\nabla \mathcal{L}(\lambda_k) - \nabla \Gamma^{\alpha}(\lambda_k) \right\|^2]+  \frac{2}{K}\sum_{k=0}^{K-1}\E[\left\|\nabla \Gamma^{\alpha}(\lambda_k)\right\|^2] \notag \\
& \leq \frac{2\kappa^6}{\alpha^2} + \frac{2}{K}\sum_{k=0}^{K-1}\E[\left\|\nabla \Gamma^{\alpha}(\lambda_k)\right\|^2] \leq \epsilon^2
\end{align}
The whole complexity of Algorithm \ref{alg:minmax:general} is $\mathcal{O}(3 BK + 3B K T_k) = \mathcal{O}(\epsilon^{-6}\log(1/\epsilon))$.
\item If $\sigma_2 =0$, we then properly select $T_k$ and $B$ as 
\begin{subequations}
\begin{align}
T_k  & \geq \frac{3\kappa-1}{4}\ln\left(\frac{12(\ell_{11}^2 + \alpha^2\ell_{21}^2)\max \left\lbrace \E[\delta_k^2], \E[r_k^2]\right\rbrace}{\zeta^2}\right) \label{eq: Tk_max_0sigma2}\\
B & =  \frac{12\kappa \left(\frac{1}{2\alpha} + 1\right)\sigma_1^2}{\zeta^2} \label{eq: B_max_0sigma2}
\end{align}
\end{subequations}
where
$$\max \left\lbrace \E[\delta_k^2], \E[r_k^2]\right\rbrace \leq  \left\{ 
\begin{aligned}
  &\frac{\zeta^2}{3\alpha^2\ell_{21}^2} + 2\ell_{\omega^{\ast}}^2 (\eta^{\lambda})^2 \left(2\zeta^2 + 4\ell_{11}^2 + 4\kappa^2 + \frac{\sigma_1^2}{B} \right),  &  k \geq 1  \\
& \max \left\lbrace \left\|u_0 - u^{\ast}(\lambda_0)\right\|^2, 2\left\|\omega_0 - \omega^{\ast}(\lambda_0)\right\|^2 + \frac{2\kappa^2}{\alpha^2}\right\rbrace  & k=0
\end{aligned}
\right.
$$
Let $\alpha = \mathcal{O}(\kappa^3\epsilon^{-1})$ and $\zeta = \mathcal{O}(\epsilon^{-1})$, then after $K = \kappa^4 \epsilon^{-2}$ steps and $B = \kappa \epsilon^{-2}$, we have 
the total complexity is $ BKT_k = \mathcal{O}(\kappa^6 \epsilon^{-4}\log(1/\epsilon))$
\end{itemize}

\end{proof}

\begin{proof}(of Multi-stage Algorithm~\ref{alg:minmax:general})
In this case, we focus on the Algorithm~\ref{alg:minmax:general} with $K_i \geq 1$ and $\alpha_i$ increasing with $i$.

At each stage $i$, because $\alpha_i$ is fixed, the analysis is similar to the one-stage version of the MinimaxOPT algorithm: that is 
If $\eta_i^{\lambda} \leq \frac{1}{\ell_{\Gamma}}$, then
\begin{align}\label{eq: grad_Gamma_core_bound_multistage}
    \frac{1}{K_i}\sum_{k=0}^{K_i-1}\E[\left\|\nabla \Gamma^{\alpha_i}(\lambda_k^i) \right\|^2] & \leq \frac{2\E[\Gamma^{\alpha}(\lambda_0^i)] -2\Gamma^{\alpha_i}(\lambda_{K_i}^i) }{\eta_i^{\lambda}} +  \frac{\ell_{\Gamma}  (\eta_i^{\lambda})}{B_i}\left(\sigma_1^2 + 2\alpha_i^2\sigma_2^2\right) \notag \\
    & + \frac{1}{K_i}\sum_{k=0}^{K_i-1}\left\| \nabla \Gamma^{\alpha_i}(\lambda_k^i) - \nabla_{\lambda} L^{\alpha}(u_{k+1}^i, \omega_{k+1}^i, \lambda_k^i)\right\|^2.
\end{align}
At each stage $i$, in order to achieve $\left\| \nabla \Gamma^{\alpha_i}(\lambda_k^i) - \nabla_{\lambda} L^{\alpha}(u_{k+1}^i, \omega_{k+1}^i, \lambda_k^i)\right\|^2 \leq \zeta_i^2$, we properly choose $T_k^i$ and $B_k^i$ such that
\begin{subequations}
\begin{align}
T_k^i  & \geq \frac{3\kappa-1}{4}\ln\left(\frac{12(\ell_{11}^2 + \alpha_i^2\ell_{21}^2)\max \left\lbrace \E[(\delta_k^i)^2], \E[(r_k^i)^2]\right\rbrace}{\zeta_i^2}\right) \label{eq: Tk_max_multistage}\\
B_i & =  \frac{12\kappa \left(\frac{1}{2\alpha_i} + 1\right)(\sigma_1^2 + \alpha_i^2\sigma_2^2)}{\zeta^2} \label{eq: B_max_multistage}
\end{align}
\end{subequations}
where
$$\max \left\lbrace \E[(\delta_k^i)^2], \E[(r_k^i)^2]\right\rbrace \leq  \left\{ 
\begin{aligned}
  &\frac{\zeta_i^2}{3\alpha_i^2\ell_{21}^2} + 2\ell_{\omega^{\ast}}^2 (\eta_i^{\lambda})^2 \left(2\zeta_i^2 + 4\ell_{11}^2 + 4\kappa^2 + \frac{\sigma_1^2 + \alpha_i^2\sigma_2^2}{B_i} \right),  &  k \geq 1  \\
& \max \left\lbrace \left\|u_0^i - u^{\ast}(\lambda_0^i)\right\|^2, \left\|\omega_0^i - \omega_{\alpha_i}^{\ast}(\lambda_0^i)\right\|^2 \right\rbrace  & k=0
\end{aligned}
\right.
$$ 
Let $\Delta_i(\alpha_i, \zeta_i, B_i)$ be the right side term when $k \geq 1$. Let $\alpha_i = \alpha_0 \tau^{i} $ ($i\in [N]$), $\zeta_i = 1/\tau^i$, and $B_i=1/\zeta_i^2$. We can see that $\Delta_i$ is non-expansion (i.e., $\Delta_{i+1}\leq \Delta_i$ for $i \geq 0$).
By induction, from $i=0$,  we have $\max\left\lbrace \E[(\delta_0^1)^2], \E[(r_0^1)^2] \right\rbrace \leq \Delta_0$ when $k\geq 1$. The next step is to show $\left\|u_0^i - u^{\ast}(\lambda_0^i)\right\|^2, \left\|\omega_0^i -\omega_{\alpha_i}^{\ast}(\lambda_0^i)\right\|^2$ are bounded. Due to that $u^{\ast}(\lambda)$ is independent on $\alpha_i$ and $u_0^{i+1} = u_{K_i}^i, \lambda_0^{i+1} = \lambda_{K_i}^i$. When $i=0$
\begin{align}
\left\|u_0^{1} - u^{\ast}(\lambda_0^{1})\right\|^2 = \left\|u_{K_0}^{0} - u^{\ast}(\lambda_{K_0}^{0})\right\|^2 \leq \max\left\lbrace \Delta_i(\alpha_0, \zeta_0, B_0), \left\|u_0^0 - u^{\ast}(\lambda_0^0) \right\|^2 \right\rbrace
\end{align}
For $i \geq 1$, by induction, we have
\begin{align}
\left\|u_0^{i+1} - u^{\ast}(\lambda_0^{i+1})\right\|^2 & = \left\|u_{K_i}^{i} - u^{\ast}(\lambda_{K_i}^{i})\right\|^2 \leq \left\lbrace \Delta_i(\alpha_i, \zeta_i, B_i), \left\|u_0^i - u^{\ast}(\lambda_0^i) \right\|^2 \right\rbrace \notag \\
& \leq  \max \left\lbrace \Delta_i(\alpha_i, \zeta_i, B_i), \max(\left\|u_0^{i-1} - u^{\ast}(\lambda_0^{i-1})\right\|^2, \Delta_{i-1}(\alpha_{i-1}, \zeta_{i-1}, B_{i-1})) \right\rbrace \notag \\
& \leq \max \left\lbrace \Delta_0(\alpha_0, \zeta_0, B_0), \left\|u_0^{0} - u^{\ast}(\lambda_0^{0})\right\|^2 \right\rbrace
\end{align}
Next, we estimate $\left\|\omega_0^i -\omega_{\alpha_i}^{\ast}(\lambda_0^i)\right\|^2$. Because $\omega_{\alpha_i}^{\ast}$ is dependent on $\alpha_i$, the analysis is a little different from $u$. 
For $ i \geq 1$,  
\begin{align}
   \E[r_{K_i}^i] = \E[\left\|\omega_{K_i}^i - \omega_{\alpha_i}^{\ast}(\lambda_{K_i}^i) \right\|^2] = \E[\left\|\omega_{0}^{i+1} - \omega_{\alpha_i}^{\ast}(\lambda_{0}^{i+1}) \right\|^2]\leq \Delta_i(\alpha_i, \zeta_i, B_i)   
\end{align}
\begin{align}\label{eq: r_i+10}
\E[r_{0}^{i+1}] &  = \E[\left\|\omega_{0}^{i+1} - \omega_{\alpha_{i+1}}^{\ast}(\lambda_{0}^{i+1}) \right\|^2] = \E\left[\left\|\omega_{0}^{i+1} - \omega_{\alpha_i}^{\ast}(\lambda_{0}^{i+1}) + \omega_{\alpha_i}^{\ast}(\lambda_{0}^{i+1}) - \omega_{\alpha_{i+1}}^{\ast}(\lambda_{0}^{i+1})  \right\|^2 \right] \notag \\
& \leq 2 \E\left[\left\|\omega_{0}^{i+1} - \omega_{\alpha_i}^{\ast}(\lambda_{0}^{i+1})   \right\|^2 \right] + 2\E\left[\left\| \omega_{\alpha_i}^{\ast}(\lambda_{0}^{i+1}) - \omega_{\alpha_{i+1}}^{\ast}(\lambda_{0}^{i+1}) \right\|^2 \right] \notag \\
& \leq 2\Delta_i(\alpha_i, \zeta_i, B_i)   + 2\E\left[\left\| \omega_{\alpha_i}^{\ast}(\lambda_{0}^{i+1}) - \omega_{\alpha_{i+1}}^{\ast}(\lambda_{0}^{i+1}) \right\|^2 \right]
\end{align}
We then estimate $\E\left[\left\| \omega_{\alpha_i}^{\ast}(\lambda_{0}^{i+1}) - \omega_{\alpha_{i+1}}^{\ast}(\lambda_{0}^{i+1}) \right\|^2 \right]$. 
Using the optimality of $\omega_{\alpha_{i}}^{\ast}$, we have
\begin{align}\label{eq: omega_ast_alphai}
  L_1(\omega_{\alpha_{i}}^{\ast}(\lambda_{0}^{i+1}), \lambda_{0}^{i+1})  + \alpha_i L_2(\omega_{\alpha_{i}}^{\ast}(\lambda_{0}^{i+1}), \lambda_{0}^{i+1}) \leq 
       L_1(\omega_{\alpha_{i+1}}^{\ast}(\lambda_{0}^{i+1}), \lambda_{0}^{i+1})  + \alpha_{i} L_2(\omega_{\alpha_{i+1}}^{\ast}(\lambda_{0}^{i+1}), \lambda_{0}^{i+1})
\end{align}
Due to the strongly convexity of $L_2$, we have
\begin{align}\label{eq: L2_func_omega_alpha}
L_2(\omega_{\alpha_{i}}^{\ast}(\lambda_{0}^{i+1}), \lambda_{0}^{i+1})  - L_2(\omega_{\alpha_{i+1}}^{\ast}(\lambda_{0}^{i+1}), \lambda_{0}^{i+1})   \geq \frac{\mu_2}{2} \left\|\omega_{\alpha_{i}}^{\ast}(\lambda_{0}^{i+1})- \omega_{\alpha_{i+1}}^{\ast}(\lambda_{0}^{i+1}) \right\|^2,
\end{align}
and then using the Lipschtiz continuity of $L_1$ gives  
\begin{align}\label{eq: L1_func_omega_alpha}
 L_1(\omega_{\alpha_{i+1}}^{\ast}(\lambda_{0}^{i+1}), \lambda_{0}^{i+1}) -  L_1(\omega_{\alpha_{i}}^{\ast}(\lambda_{0}^{i+1}), \lambda_{0}^{i+1}) \leq \ell_{10} \left\|\omega_{\alpha_{i+1}}^{\ast}(\lambda_{0}^{i+1})- \omega_{\alpha_{i}}^{\ast}(\lambda_{0}^{i+1})\right\|.
\end{align}
Then by (\ref{eq: omega_ast_alphai}) and combining (\ref{eq: L1_func_omega_alpha}) and (\ref{eq: L2_func_omega_alpha}), we have
\begin{align}
  \alpha_i\frac{\mu_2}{2} \left\|\omega_{\alpha_{i}}^{\ast}(\lambda_{0}^{i+1})- \omega_{\alpha_{i+1}}^{\ast}(\lambda_{0}^{i+1}) \right\|^2 
&
\leq \alpha_i\left(L_2(\omega_{\alpha_{i}}^{\ast}(\lambda_{0}^{i+1}), \lambda_{0}^{i+1})  - L_2(\omega_{\alpha_{i+1}}^{\ast}(\lambda_{0}^{i+1}), \lambda_{0}^{i+1})\right) \notag \\
& \leq    L_1(\omega_{\alpha_{i+1}}^{\ast}(\lambda_{0}^{i+1}), \lambda_{0}^{i+1}) - L_1(\omega_{\alpha_{i}}^{\ast}(\lambda_{0}^{i+1}), \lambda_{0}^{i+1})  \notag \\
 & \leq \ell_{10} \left\|\omega_{\alpha_{i+1}}^{\ast}(\lambda_{0}^{i+1})- \omega_{\alpha_{i}}^{\ast}(\lambda_{0}^{i+1})\right\|.
\end{align}
Then
\begin{align}\label{eq: omega_ast_alpha}
  \left\|\omega_{\alpha_{i+1}}^{\ast}(\lambda_{0}^{i+1})- \omega_{\alpha_{i}}^{\ast}(\lambda_{0}^{i+1})\right\| \leq \frac{2\ell_{10}}{\mu_2\alpha_i}  
\end{align}
Since $\Delta_i$ is non-expansion and $\alpha_i$ is increasing, applying (\ref{eq: omega_ast_alpha}) into (\ref{eq: r_i+10}) gives 
\begin{align}
   \E[r_{0}^{i+1}] &    \leq 2\Delta_i(\alpha_i, \zeta_i, B_i)   + \frac{8\ell_{10}^2}{\mu_2^2\alpha_i^2} \leq 2\Delta_0(\alpha_0, \zeta_0, B_0)   + \frac{8\ell_{10}^2}{\mu_2^2\alpha_0^2}.
\end{align}
Thus 
\begin{align}
\max \left\lbrace \left\|u_0^i - u^{\ast}(\lambda_0^i)\right\|^2, \left\|\omega_0^i - \omega_{\alpha_i}^{\ast}(\lambda_0^i)\right\|^2 \right\rbrace \leq M_0:= \max\left\lbrace \delta_0^0, r_0^0, 2\Delta_0(\alpha_0, \zeta_0, B_0)   + \frac{8\ell_{10}^2}{\mu_2^2\alpha_0^2}\right\rbrace    
\end{align}
Overall, 
by properly choosing $T_k^i$ and $B_k^i$ such that
\begin{subequations}
\begin{align}
T_k^i  & \geq \frac{3\kappa-1}{4}\ln\left(\frac{12(\ell_{11}^2 + \alpha_i^2\ell_{21}^2)\max \left\lbrace \E[(\delta_k^i)^2], \E[(r_k^i)^2]\right\rbrace}{\zeta_i^2}\right) \label{eq: Tk_max_multistage_final}\\
B_i & =  \frac{12\kappa \left(\frac{1}{2\alpha_i} + 1\right)(\sigma_1^2 + \alpha_i^2\sigma_2^2)}{\zeta_i^2} \label{eq: B_max_multistage_final}
\end{align}
\end{subequations}
where
$$\max \left\lbrace \E[(\delta_k^i)^2], \E[(r_k^i)^2]\right\rbrace \leq  \left\{ 
\begin{aligned}
  &\frac{\zeta_i^2}{3\alpha_i^2\ell_{21}^2} + 2\ell_{\omega^{\ast}}^2 (\eta_i^{\lambda})^2 \left(2\zeta_i^2 + 4\ell_{11}^2 + 4\kappa^2 + \frac{\sigma_1^2 + \alpha_i^2\sigma_2^2}{B_i} \right),  &  k \geq 1  \\
& \max \left\lbrace \delta_0^0, r_0^0, 2\Delta_0(\alpha_0, \zeta_0, B_0)   + \frac{8\ell_{10}^2}{\mu_2^2\alpha_0^2}\right\rbrace   & k=0
\end{aligned}
\right.
$$
we can achieve that $\left\| \nabla \Gamma^{\alpha_i}(\lambda_k^i) - \nabla_{\lambda} L^{\alpha}(u_{k+1}^i, \omega_{k+1}^i, \lambda_k^i)\right\|^2 \leq \zeta_i^2$ at each stage $i$. Let 
\begin{align}
\eta_i^{\lambda}= \frac{1}{\ell_{\Gamma}};\quad  \eta_i^{u}= \frac{2}{\alpha_i(\mu_2 + \ell_{21})}; \quad \eta_i^{\omega}= \frac{4}{\alpha_i(\mu_2 + 3\ell_{21})}
\end{align}
Telescoping (\ref{eq: grad_Gamma_core_bound_multistage}) from $i=0,\cdots, N$, then
\begin{align}
\frac{1}{\sum_{i=0}^N K_i} \sum_{i=0}^N \sum_{k=0}^{K_i-1} \E[\left\|\nabla \Gamma^{\alpha_i}(\lambda_k^i) \right\|^2]  & \leq \frac{1}{\sum_{i=0}^N K_i}\frac{2\E[\Gamma^{\alpha_0}(\lambda_0^0)] -2\Gamma_{\min}^{\alpha_N} }{\eta^{\lambda}} +  \frac{1}{\sum_{i=1}^N K_i}\sum_{i=1}^N K_i\frac{\ell_{\Gamma}  (\eta^{\lambda})}{B_i}\left(\sigma_1^2 + 2\alpha_i^2\sigma_2^2\right) \notag \\
    & + \frac{1}{\sum_{i=1}^N K_i}\sum_{i=1}^N K_i\zeta_i^2.  
\end{align}
To achieve that
\begin{align*}
    \frac{1}{\sum_{i=0}^N K_i }\sum_{i=0}^{N}\sum_{k=0}^{K_i-1}\E[\left\| \nabla \mathcal{L}(\lambda_k^i) \right\|^2] \leq \frac{2}{\sum_{i=0}^N K_i} \sum_{i=0}^N \sum_{k=0}^{K_i-1} \E[\left\|\nabla \Gamma^{\alpha_i}(\lambda_k^i) \right\|^2 ] + \frac{2}{\sum_{i=0}^N K_i }\sum_{i=0}^N K_i\frac{C^2}{\alpha_i^2} \leq \epsilon^2
\end{align*}
we let $\alpha_i = \alpha_0 \tau^i, \zeta_i=\tau^{-i}$, $K_i=\tau^{2i}$, $N=\log_{\tau}(1/\epsilon)$,  then  
the total number of iterations $\Sigma = \sum_{i=0}^{N} \sum_{k=0}^{K_i} T_k^i = \epsilon^{-2}$ and the total complexity is $\mathcal{O}\left(\sum_{i=0}^N K_i B_i\right) = \mathcal{O}(\epsilon^{-6}\log(1/\epsilon)) $
\end{proof}

\begin{proof}(of (Stochastic) Multi-Stage Algorithm~\ref{alg:minmax:general}) In this setting, we prove the stochastic version of multistage GDA. 
We recall the iterating formula of $\lambda$ in Algorithm \ref{alg:minmax:general} that  $\lambda_{k+1}^i - \lambda_k^i =  - \eta_{i}^{\lambda} \nabla_{\lambda} L_{D_k^i}^{\alpha_i}(u_k^i, \omega_k^i, \lambda_k^i)$. At each iteration, we have
 \begin{align}
      \E[\nabla L_{D_k^i}^{\alpha_i}(u_k^i, \omega_k^i, \lambda_k^i) \mid \mathcal{F}_k^i] = \nabla L^{\alpha_i}(u_k^i, \omega_k^i, \lambda_k^i)
 \end{align}
By the smoothness of $\Gamma$ from Lemma \ref{lem:onestage:v2}, we have
\begin{align}\label{inequ:gamma:1:multi}
\Gamma^{\alpha_i}(\lambda_{k+1}^i) & \leq \Gamma^{\alpha_i}(\lambda_k^i) + \left\langle \nabla \Gamma^{\alpha_i}(\lambda_k^i), \lambda_{k+1}^i - \lambda_{k}^i\right\rangle + \frac{\ell_{\Gamma}}{2}\left\| \lambda_{k+1}^i - \lambda_k^i \right\|^2 \notag \\
& = \Gamma^{\alpha_i}(\lambda_k^i) - \eta_i^{\lambda} \left\langle \nabla \Gamma^{\alpha_i}(\lambda_k^i), \nabla_{\lambda} L_{D_k^i}^{\alpha_i}(u_k^i, \omega_k^i, \lambda_k^i)\right\rangle + \frac{\ell_{\Gamma}  (\eta_{i}^{\lambda})^2}{2}\left\| \nabla_{\lambda} L_{D_k^i}^{\alpha_i}(u_k^i, \omega_k^i, \lambda_k^i)\right\|^2 
\end{align}
Taking conditional expectation w.r.t. $\mathcal{F}_k^i$ on the above inequality, we have
\begin{align}
\E[\Gamma^{\alpha_i}(\lambda_{k+1}^i) \mid \mathcal{F}_k^i] &\leq \Gamma^{\alpha_i}(\lambda_k^i) - \eta_{\lambda} \left\langle \nabla \Gamma^{\alpha_i}(\lambda_k^i), \E[\nabla_{\lambda} L_{D_k^i}^{\alpha_i}(u_k^i, \omega_k^i, \lambda_k^i) \mid \mathcal{F}_k^i]\right\rangle \notag \\
& + \frac{\ell_{\Gamma}  (\eta_i^{\lambda})^2}{2}\E\left[\left\| \nabla_{\lambda} L_{D_k^i}^{\alpha}(u_k^i, \omega_k^i, \lambda_k^i)\right\|^2 \mid \mathcal{F}_k^i\right] \notag \\
& \leq \Gamma^{\alpha_i}(\lambda_k^i) - \eta_i^{\lambda} \left\langle \nabla \Gamma^{\alpha_i}(\lambda_k^i), \nabla_{\lambda} L_{D_k^i}^{\alpha_i}(u_k^i, \omega_k^i, \lambda_k^i)\right\rangle + \frac{\ell_{\Gamma}  (\eta_i^{\lambda})^2}{2}\E\left[\left\| \nabla_{\lambda} L_{D_k^i}^{\alpha_i}(u_k^i, \omega_k^i, \lambda_k^i)\right\|^2 \mid \mathcal{F}_k^i\right] 
\end{align}
where the inequality follows the fact that $L_{D_k^i}^{\alpha_i}$ is an unbiased estimation of $L^{\alpha_i}$ and 
\begin{align}
\E\left[\left\| \nabla_{\lambda} L_{D_k^i}^{\alpha_i}(u_k^i, \omega_k^i, \lambda_k^i)\right\|^2 \mid \mathcal{F}_k^i \right] & =     \E\left[\left\| \nabla_{\lambda} L_{D_k^i}^{\alpha_i}(u_k^i, \omega_k^i, \lambda_k^i) - \nabla_{\lambda} L^{\alpha_i}(u_k^i, \omega_k^i, \lambda_k^i) + \nabla_{\lambda} L^{\alpha_i}(u_k^i, \omega_k^i, \lambda_k^i) \right\|^2\mid \mathcal{F}_k^i \right] \notag \\
& \leq  \E\left[\left\| \nabla_{\lambda} L_{D_k^i}^{\alpha_i}(u_k^i, \omega_k^i, \lambda_k^i) - \nabla_{\lambda} L^{\alpha_i}(u_k^i, \omega_k^i, \lambda_k^i) \right\|^2 \mid \mathcal{F}_k^i \right] + \left\|\nabla_{\lambda} L^{\alpha_i}(u_k^i, \omega_k^i, \lambda_k^i) \right\|^2 \notag \\
& \leq \sigma_1^2 + 2\alpha_i^2\sigma_2^2 + \left\|\nabla_{\lambda} L^{\alpha_i}(u_k^i, \omega_k^i, \lambda_k^i) \right\|^2
\end{align}
If we suppose that 
\begin{itemize}
    \item $\E[\left\| \nabla L_1(\omega, \lambda; S_{val}) - \nabla L_1(\omega, \lambda) \right\|^2] \leq \sigma_1^2$
       \item $\E[\left\| \nabla L_2(\omega, \lambda; S_{train}) - \nabla L_2(\omega, \lambda) \right\|^2] \leq \sigma_2^2$
\end{itemize}
Then 
\begin{align}
  & \E[\left\| \nabla_{\lambda} L_{D_k^i}^{\alpha_i}(u_k^i, \omega_k^i, \lambda_k^i) - \nabla_{\lambda} L^{\alpha_i}(u_k^i, \omega_k^i, \lambda_k^i) \right\|^2 \mid \mathcal{F}_k^i ] =   \E[\left\| \nabla_{\lambda} L_1(\omega_k^i, \lambda_k^i; S_{val,k}^i) - \nabla_{\lambda} L_1(\omega_k^i, \lambda_k^i) \right\|^2 \mid \mathcal{F}_k^i ]   \notag \\
  & +  \alpha_i^2 \E[ \left\| \nabla_{\lambda} L_2(\omega_k^i, \lambda_k^i; S_{train,k}^i) - \nabla_{\lambda} L_2(\omega_k^i, \lambda_k^i) \right\|^2\mid \mathcal{F}_k^i ] + \alpha_i^2 \E[ \left\| \nabla_{\lambda} L_2(u_k^i, \lambda_k^i; S_{train,k}^i) - \nabla_{\lambda} L_2(u_k^i, \lambda_k^i) \right\|^2\mid \mathcal{F}_k^i ] \notag \\
  & \leq \sigma_1^2 + 2\alpha_i^2\sigma_2^2
\end{align}
By the above results, we have
\begin{align}
\E[\Gamma^{\alpha_i}(\lambda_{k+1}^i) \mid \mathcal{F}_k^i] & \leq \Gamma^{\alpha_i}(\lambda_k^i) -  \eta_i^{\lambda} \left\langle \nabla \Gamma^{\alpha_i}(\lambda_k^i), \nabla_{\lambda} L^{\alpha_i}(u_k^i, \omega_k^i, \lambda_k^i)\right\rangle  + \frac{\ell_{\Gamma}  (\eta_i^{\lambda})^2}{2} \left(\sigma_1^2 + 2\alpha_i^2\sigma_2^2 + \left\|\nabla_{\lambda} L^{\alpha_i}(u_k^i, \omega_k^i, \lambda_k^i) \right\|^2 \right) \notag \\
& \leq \Gamma^{\alpha_i}(\lambda_k^i) -  \eta_i^{\lambda} \left\langle \nabla \Gamma^{\alpha_i}(\lambda_k^i), \nabla_{\lambda} L^{\alpha_i}(u_k^i, \omega_k^i, \lambda_k^i)\right\rangle  + \frac{\ell_{\Gamma}  (\eta_i^{\lambda})^2}{2}\left\|\nabla_{\lambda} L^{\alpha_i}(u_k^i, \omega_k^i, \lambda_k^i) \right\|^2 \notag \\
& + \frac{\ell_{\Gamma}  (\eta_i^{\lambda})^2}{2}\left(\sigma_1^2 + 2\alpha^2\sigma_2^2\right) \notag \\
& = \Gamma^{\alpha_i}(\lambda_k^i) - \frac{\eta_i^{\lambda}}{2} \left\|\nabla \Gamma^{\alpha_i}(\lambda_k^i) \right\|^2 - \left(\frac{\eta_i^{\lambda}}{2} - \frac{\ell_{\Gamma}  (\eta_i^{\lambda})^2}{2}\right)\left\| \nabla_{\lambda} L^{\alpha_i}(u_k^i, \omega_k^i, \lambda_k^i)\right\|^2 \notag \\
& + \frac{\eta_i^{\lambda}}{2}\left\| \nabla \Gamma^{\alpha_i}(\lambda_k^i) - \nabla_{\lambda} L^{\alpha_i}(u_k^i, \omega_k^i, \lambda_k^i)\right\|^2 +  \frac{\ell_{\Gamma}  (\eta_i^{\lambda})^2}{2}\left(\sigma_1^2 + 2\alpha_i^2\sigma_2^2\right).
\end{align}
Next we turn to estimate $\left\| \nabla \Gamma^{\alpha_i}(\lambda_k^i) - \nabla_{\lambda} L^{\alpha_i}(u_k^i, \omega_k^i, \lambda_k^i)\right\|^2$.
\begin{align}
  & \left\| \nabla \Gamma^{\alpha_i}(\lambda_k^i) - \nabla_{\lambda} L^{\alpha_i}(u_k^i, \omega_k^i, \lambda_k^i)\right\|^2 =   \left\| \nabla_{\lambda} L^{\alpha_i}(u^{\ast}(\lambda_k^i), \omega_{\alpha_i}^{\ast}(\lambda_k^i), \lambda_k^i) - \nabla_{\lambda} L^{\alpha_i}(u_k^i, \omega_k^i, \lambda_k^i)\right\|^2 \notag \\
  & = 3\left\|\nabla_{\lambda} L_1(\omega_{\alpha_i}^{\ast}(\lambda_k^i), \lambda_k^i) - \nabla_{\lambda} L_1(\omega_k^i, \lambda_k^i)\right\|^2 + 3\alpha_i^2\left\|\nabla_{\lambda} L_2(\omega_{\alpha_i}^{\ast}(\lambda_k^i), \lambda_k^i)  - \nabla_{\lambda} L_2(\omega_k^i, \lambda_k^i) \right\|^2 \notag \\
  & \quad + 3\alpha_i^2 \left\|\nabla_{\lambda} L_2(u^{\ast}(\lambda), \lambda) - \nabla_{\lambda} L_2(u_k, \lambda_k) \right\|^2 \notag \\
  & \leq 3(\ell_{11}^2 + \alpha_i^2 \ell_{21}^2) \left\| \omega_{\alpha_i}^{\ast}(\lambda_k^i) - \omega_k^i\right\|^2 + 3\alpha_i^2 \ell_{21}^2 \left\| u^{\ast}(\lambda_k^i) - u_k^i \right\|^2.
\end{align}
Let $\delta_k^i = \left\|u_k^i - u^{\ast}(\lambda_k^i) \right\|^2$ and $r_k^i = \left\| \omega_k^i - \omega_{\alpha_i}^{\ast}(\lambda_k^i) \right\|^2$, then the inequality (\ref{inequ:gamma:1:multi}) can be simplified as
\begin{align}\label{inequ:gamma:core:multi:sto}
 \E[ \Gamma^{\alpha_i}(\lambda_{k+1}^i) \mid \mathcal{F}_k^i] & \leq \Gamma^{\alpha_i}(\lambda_k^i)  - \frac{1}{2}\eta_i^{\lambda} \left\| \nabla \Gamma^{\alpha_i}(\lambda_k^i) \right\|^2 + \frac{3\eta_i^{\lambda}}{2}\left(\left(\ell_{11}^2 + \alpha^2 \ell_{21}^2 \right) r_k^i + \alpha_i^2 \ell_{21}^2 \delta_k^i \right) \notag \\
  & - \left(\frac{\eta_i^{\lambda}}{2} - \frac{\ell_{\Gamma}(\eta_i^{\lambda})^2}{2} \right) \left\|\nabla_{\lambda} L^{\alpha_i}(u_k^i, \omega_k^i, \lambda_k^i)\right\|^2 + \frac{\ell_{\Gamma}  (\eta_i^{\lambda})^2}{2}\left(\sigma_1^2 + 2\alpha_i^2\sigma_2^2\right).
\end{align}
Then, we focus on estimating $\delta_k^i$ and $r_k^i$. First, we turn to evaluate $\delta_k^i$. By the strongly concavity of $L^{\alpha_i}$ with respect to $u$ and taking conditional expectation on (\ref{inequ:gamma:core:multi:sto}) then
\begin{align}\label{inequ:uk:multi}
  \E[\left\|  u^{\ast}( \lambda_{k}^i) - u_{k+1}^i  \right\|^2 \mid \mathcal{F}_k^i] & = \E\left[\left\|u^{\ast}(\lambda_{k}^i) - u_k^i - \eta_i^{u} \nabla_{u} L_{D_k^i}^{\alpha_i}(u_k^i, \omega_k^i, \lambda_k^i) \right\|^2 \mid \mathcal{F}_k^i \right] \notag \\
  & \leq \left\|u^{\ast}( \lambda_{k}^i) - u_k^i \right\|^2 - 2\eta_i^{u}\left\langle u^{\ast}( \lambda_{k}^i) - u_k^i, \nabla_{u} L^{\alpha_i}(u_k^i, \omega_k^i, \lambda_k^i) \right\rangle \notag \\
  & + (\eta_i^{u})^2 \E\left[\left\| \nabla_{u} L_{D_k^i}^{\alpha}(u_k^i, \omega_k^i, \lambda_k^i) \right\|^2 \mid \mathcal{F}_k^i\right]\notag \\
  & \mathop{\leq}^{(a)} \left\|u^{\ast}( \lambda_{k}^i) - u_k^i \right\|^2 - 2\eta_i^u\left( L^{\alpha_i}(u^{\ast}(\lambda_k^i), \omega_k^i, \lambda_k^i) - L^{\alpha_i}(u_k^i, \omega_k^i, \lambda_k^i) + \frac{\mu_2\alpha_i}{2}\left\|u^{\ast}(\lambda_k^i) - u_k^i \right\|^2 \right) \notag \\ 
  & \quad + (\eta_i^u)^2 \E\left[\left\|\nabla_{u} L_{D_k^i}^{\alpha_i}(u_k^i, \omega_k^i, \lambda_k^i) - \nabla_{u} L^{\alpha_i}(u_k^i, \omega_k^i, \lambda_k^i) \right\|^2 \mid \mathcal{F}_k^i\right] + (\eta_i^u)^2\left\| \nabla_{u} L^{\alpha_i}(u_k^i, \omega_k^i, \lambda_k^i) \right\|^2 \notag \\
& \mathop{\leq}^{(b)}  \left\|u^{\ast}( \lambda_{k}^i) - u_k^i \right\|^2 - 2\eta_i^u\left(L^{\alpha_i}(u^{\ast}(\lambda_k^i), \omega_k^i, \lambda_k^i) - L^{\alpha_i}(u_k^i, \omega_k^i, \lambda_k^i) + \frac{\mu_2\alpha_i}{2}\left\|u^{\ast}(\lambda_k^i) - u_k^i \right\|^2 \right) \notag \\ 
  & \quad + 2\alpha_i\ell (\eta_i^u)^2\left(-L^{\alpha_i}(u_k^i, \omega_k^i, \lambda_k^i) +  L^{\alpha_i}(u^{\ast}(\lambda_k^i), \omega_k^i, \lambda_k^i)\right) + (\eta_i^{u})^2\alpha_i^2\sigma_2^2 \notag \\
  & \leq \left(1-\mu_{2}\alpha_i \eta_i^u \right) \left\|u^{\ast}(\lambda_k^i) - u_k^i \right\|^2 + (\eta_i^{u})^2\alpha_i^2\sigma_2^2  .
\end{align}
where $(a)$ follows from the strongly concavity of $L^{\alpha_i}$ w.r.t. $u$ which implies that 
\begin{align}
  -L^{\alpha_i}(u^{\ast}(\lambda_k^i), \omega_k^i, \lambda_k^i) \geq -L^{\alpha_i}(u_k^i, \omega_k^i, \lambda_k^i) - \left\langle \nabla_{u}L^{\alpha_i}(u_k^i, \omega_k^i, \lambda_k^i), u^{\ast}(\lambda_k^i) - u_k^i \right\rangle + \frac{\mu_2\alpha_i}{2}\left\| u_k^i -u^{\ast}(\lambda_k^i) \right\|^2  \notag  
\end{align}
and $(b)$ uses the smoothness of $-L^{\alpha_i}$ with respect to $u$ such that 
\begin{align}
& -L^{\alpha_i}(u^{\ast}(\lambda_k^i), \omega_k^i, \lambda_k^i) + L^{\alpha_i}(u_k^i, \omega_k^i, \lambda_k^i) \notag \\
& \leq -L^{\alpha_i}(\tilde{u}, \omega_k^i, \lambda_k^i) +L^{\alpha_i}(u_k^i, \omega_k^i, \lambda_k^i) \notag \\
 & \leq  -L^{\alpha_i}(u_k^i, \omega_k^i, \lambda_k^i) + \left\langle -\nabla_u L^{\alpha_i}(u_k^i, \omega_k^i, \lambda_k^i), \tilde{u} - u_k^i \right\rangle + \frac{\alpha_i \ell}{2}\left\|\tilde{u} - u_k^i \right\|^2 +L^{\alpha_i}(u_k^i, \omega_k^i, \lambda_k^i)  \notag \\
 & = - \frac{1}{2 \alpha_i \ell} \left\|\nabla_u L^{\alpha_i}(u_k^i, \omega_k^i, \lambda_k^i) \right\|^2
\end{align}
where $\tilde{u} = u_k^i + \frac{1}{\alpha_i \ell} \nabla_{u}L^{\alpha_i}(u_k^i, \omega_k^i, \lambda_k^i)$.
Recalling the definition of $\delta_k^i$, we have
\begin{align}\label{inequ:deltak:multi}
\E[\delta_{k+1}^i \mid \mathcal{F}_k ] =& \E[\left\|u^{\ast}(\lambda_{k+1}^i) - u_{k+1}^i  \right\|^2 \mid \mathcal{F}_k ] \notag \\ \mathop{\leq}^{(a)} & (1 + \gamma_1 ) \E[\left\|u^{\ast}(\lambda_{k+1}^i) -u^{\ast}(\lambda_{k}^i)  \right\|^2\mid \mathcal{F}_k ]  + (1+1/\gamma_1) \E[\left\|  u^{\ast}(\lambda_{k}^i) - u_{k+1}^i  \right\|^2\mid \mathcal{F}_k ]  \notag \\
\mathop{\leq}^{(b)} &  (1 + \gamma_1) \kappa^2\E[\left\| \lambda_{k+1}^i - \lambda_k^i \right\|^2\mid \mathcal{F}_k ]  + (1+1/\gamma_1) \E[\left\|  u^{\ast}(\lambda_{k}^i) - u_{k+1}^i  \right\|^2 \mid \mathcal{F}_k ] \notag \\
\mathop{\leq}^{(c)} &  (1 + \gamma_1) \kappa^2(\eta_i^{\lambda})^2 \E[\left\| \nabla_{\lambda} L_{D_k^i}^{\alpha_i}(u_k^i, \omega_k^i, \lambda_k^i)\right\|^2 \mid \mathcal{F}_k ] + (1+1/\gamma_1)\left(1-\mu_2\alpha_i\eta_i^{\lambda}\right) \delta_k^i \notag \\
& +  (1+1/\gamma_1)(\eta_i^{u})^2\alpha_i^2\sigma_2^2  \notag\\
&  (1 + \gamma_1) \kappa^2(\eta_i^{\lambda})^2 \left(\sigma_1^2 + \alpha_i^2\sigma_2^2 + \left\| \nabla_{\lambda} L_{D_k^i}^{\alpha_i}(u_k^i, \omega_k^i, \lambda_k^i)\right\|^2 \right) + (1+1/\gamma_1)\left(1-\mu_2\alpha_i\eta_i^{\lambda}\right) \delta_k^i \notag \\
& +  (1+1/\gamma_1)(\eta_i^{u})^2\alpha_i^2\sigma_2^2  
\end{align}
where $(a)$ follows from Cauchy-Schwartz inequality with $\gamma_1 >0$; (b) uses the Lipschitz continuity of $u^{\ast}$; $(c)$ follows from the inequality (\ref{inequ:uk:multi}) and the iterating formula of $\lambda_{k+1}^i$.

By the strongly convexity of $\Phi^{\alpha_i}$ with respect to $\omega$, we have
\begin{align}
   \E\left[ \left\|\omega^{\ast}(\lambda_k^i) - \omega_{k+1}^i \right\|^2 \mid \mathcal{F}_k^i \right]  \leq \left( 1- \frac{\mu_2\alpha_i\eta_i^{\omega}}{2}\right)r_k^i + \left(\eta_i^{\omega}\right)^2\left(\sigma_1^2 + \alpha_i^2 \sigma_2^2\right). 
\end{align}
Similarly, we estimate $r_k^i$:
\begin{align}\label{inequ:rk:multi}
 \E\left[r_{k+1}^i \mid \mathcal{F}_k^i \right] & \leq (1+\gamma_2)\E\left[\left\| \omega^{\ast}(\lambda_{k+1}^i) - \omega^{\ast}(\lambda_k^i)\right\|^2 \mid \mathcal{F}_k^i \right]  + (1+ \gamma_2^{-1}) \E\left[\left\| \omega^{\ast}(\lambda_k^i) - \omega_{k+1}^i \right\|^2\mid \mathcal{F}_k^i \right]  \notag \\
 & \leq (1+\gamma_2)\kappa^2 \E\left[\left\|\lambda_{k+1}^i - \lambda_k^i\right\|^2 \mid \mathcal{F}_k^i \right]  + (1+ \gamma_2^{-1})\left( 1- \frac{\mu_2\alpha_i\eta_i^{\omega}}{2}\right)r_k^i + (1+ \gamma_2^{-1})\left(\eta_i^{\omega}\right)^2\left(\sigma_1^2 + \alpha_i^2 \sigma_2^2\right) \notag \\
 & \leq (1 + \gamma_2) \kappa^2(\eta_i^{\lambda})^2\E \left[\left\| \nabla_{\lambda} L_{D_k^i}^{\alpha_i}(u_k^i, \omega_k^i, \lambda_k^i)\right\|^2 \mid \mathcal{F}_k^i \right]  + (1+ \gamma_2^{-1})\left( 1- \frac{\mu_2\alpha_i\eta_i^{\omega}}{2}\right)r_k^i \notag \\
 & + (1+ \gamma_2^{-1})\left(\eta_i^{\omega}\right)^2\left(\sigma_1^2 + \alpha_i^2 \sigma_2^2\right) \notag \\
  & \leq (1 + \gamma_2) \kappa^2(\eta_i^{\lambda})^2 \left( \sigma_1^2 + 2\alpha_i^2\sigma_2^2 + \left\| \nabla_{\lambda} L^{\alpha_i}(u_k^i, \omega_k^i, \lambda_k^i)\right\|^2 \right)  + (1+ \gamma_2^{-1})\left( 1- \frac{\mu_2\alpha_i\eta_i^{\omega}}{2}\right)r_k^i \notag \\
 & + (1+ \gamma_2^{-1})\left(\eta_i^{\omega}\right)^2\left(\sigma_1^2 + \alpha_i^2 \sigma_2^2\right)
\end{align}
where $\gamma_2 > 0$.

The difference between $\Gamma^{\alpha}(\lambda_0) $ and $ \Gamma_{\min}^{\alpha}$ can be estimated as:
\begin{align}
    \Gamma^{\alpha}(\lambda_{0}) - \Gamma^{\alpha_i}(\lambda_{K_i+1}^i) & = L^{\alpha}(u^{\ast}(\lambda_0), \omega_{\alpha}^{\ast}(\lambda), \lambda_0) - L^{\alpha}(u^{\ast}(\lambda_{K_i+1}^i), \omega_{\alpha}^{\ast}(\lambda_{K_i+1}^i), \lambda_{K_i+1}^i) \notag \\
    & = L_1(\omega_{\alpha}^{\ast}(\lambda_0), \lambda_0) - L_1(\omega_{\alpha}^{\ast}(\lambda_{K_i+1}^i), \lambda^{\ast}) + \alpha \left(L_2(\omega_{\alpha}^{\ast}(\lambda_0), \lambda_0) - L_2(u^{\ast}(\lambda_0), \lambda_0) \right) \notag \\
    & + \alpha \left(L_2(\omega_{\alpha}^{\ast}(\lambda^{\ast}), \lambda_{K_i+1}^i) - L_2(u^{\ast}(\lambda_{K_i+1}^i), \lambda_{K_i+1}^i) \right) \notag \\
    & \leq L_1(\omega_{\alpha}^{\ast}(\lambda_0), \lambda_0) - L_1(\omega_{\alpha}^{\ast}(\lambda_{K_i+1}^i), \lambda_{K_i+1}^i) + \alpha \frac{\ell_{21}}{2}\left\|\omega_{\alpha}^{\ast}(\lambda_0) -  u^{\ast}(\lambda_0) \right\|^2 \notag \\
    & + \alpha_i \frac{\ell_{21}}{2}\left\|\omega_{\alpha}^{\ast}(\lambda_{K_i+1}^i) -  u^{\ast}(\lambda_{K_i+1}^i) \right\|^2 \notag \\
    & \leq L_1(\omega_{\alpha}^{\ast}(\lambda_0), \lambda_0) - L_1(\omega_{\alpha}^{\ast}(\lambda_{K_i+1}^i), \lambda_{K_i+1}^i) + \frac{\ell_{21}\ell_{10}^2}{2\mu_2^2\alpha_i} + \frac{\ell_{21}\ell_{10}^2}{2\mu_2^2\alpha_i} \notag \\
    & \leq \ell_{10}\left(\left\|\omega_{\alpha}^{\ast}(\lambda_0) - \omega_{\alpha}^{\ast}(\lambda_{K_i+1}^i) \right\| + \left\|\lambda_0 - \lambda_{K_i+1}^i\right\|\right) + \frac{2\ell_{21}\ell_{10}^2}{2\mu_2^2\alpha} \notag \\
    & \leq \ell_{10}\left(1 + \frac{3\ell_{21}}{\mu_2} \right)\left\|\lambda_0 - \lambda_{K_i+1}^i\right\| + \frac{2\ell_{21}\ell_{10}^2}{2\mu_2^2\alpha_i}  \leq \mathcal{O}\left(\kappa \left\|\lambda_0 - \lambda_{K_i+1}^i\right\| + \frac{\kappa}{\alpha_i}\right) 
\end{align}

\end{proof}

\section{Supplementary Details of Numerical Experiments}

In this section, we provide the details of the experiments in Section~\ref{sec:numerical}  and some additional results. 

\subsection{Numerical Details of Logistic Regression in Subsection~\ref{sec:numerical:1}}\label{append:logistic}
 For the experiments of regularized logistic regression on a synthesis dataset, the details of each algorithm are addressed below.
 We set the inner optimization step for the three bilevel methods as 100, and the learning rate for both inner and outer is 1. We set the truncated step $K_0=10$ for the reverse method and use $K=10$ steps applying for the fixed-point method and conjugate gradient method to compute hyper-gradient. Note that for the minimax algorithm (Algorithm~\ref{alg:minmax:general}), we have done the grid searching in the learning rate of $(u, \omega)$ and learning rate of $\lambda$ independently. We observe that the two scales of learning rate do not improve the performance much compared to a single learning rate for all parameters $(u, \omega, \lambda)$. Thus, to reduce the number of the hyper-parameters in Algorithm~\ref{alg:minmax:general}, we use the same learning rate for $(u, \omega, \lambda)$. For the proposed minimax algorithm, we set the inner step $K =100$  and $\eta_0 =\eta_0^{\lambda}= 1$, $\alpha_0=1$, and $\tau=1.5$. We initialize all the algorithms by $u_0=\omega_0=[0,0,\cdots, 0]^{T}$ and $\lambda_0 = [1,1,\cdots, 1]^{T}$.

 In the experiment on 20newsgroup, the gradient descent method is employed for both inner and outer problems, and the hyper-gradient of the outer problem is evaluated by the three bilevel methods: (1) truncated reverse ($K_0=10$); (2) fixed-point method; (3) conjugate gradient (CG) method. The details of the experiments on the real dataset 20newsgroup are shown below:
For the three bilevel methods, we set the outer optimization step is 50, and the inner optimization step is 500; the learning rate for both inner and outer optimization methods is 100. As the experiment on the synthesis dataset, we apply the fixed-point method and conjugate gradient method with $K_0=10$ iterations, respectively.  For the proposed minimax algorithm (Algorithm \ref{alg:minmax:general}): we set inner optimization step $K=1000$; the initial learning rate $\eta_0=\eta_0^{\lambda}=100$; the initial value $\alpha_0=1$, and $\tau=1.5$. We initialize the minimax algorithm by $\omega_0^0=u_0^0 = [0, 0, \cdots, 0]$ and $\lambda = \lambda_0^0 = [0, 0, \cdots, 0]$ and use the same values for bilevel algorithms.

In addition, Table~\ref{tab:logistic:time} provides the detailed time cost of each method for experiments on 20newsgroup.

\begin{table}[t]
    \centering
    \begin{tabular}{c|c}
    \toprule
    Method & Time ($s$) \\ \hline
      Reverse   &  94.30 \\ 
       Fixed-point  & 93.58 \\ 
       CG &  93.3 \\ 
       MinimaxOPT &  57.3  \\ \bottomrule
    \end{tabular}
    \vspace{0.1in}
    \caption{Time of the methods on 20newsgroups dataset}
    \label{tab:logistic:time}
\end{table}

\subsection{Numerical Details of CIFAR10 in Subsection~\ref{sec:num:dnn}}\label{append:dnn}
We generalize the multi-stage gradient descent and ascent to the stochastic setting and accelerate the algorithm by momentum, shown in Algorithm \ref{alg:minmax:general}. Each experiment is run 3 times and the results are averaged to eliminate the effect of randomness.

\begin{algorithm}[ht]
\caption{Multi-Stage Stochastic Gradient Descent and Ascent with Momentum}
\label{alg:minmax:3}
\begin{algorithmic}[1]
\STATE {\bfseries Input:}  $u_0^0$, $\lambda_0^0$, $\omega_0^0$ and $\alpha_0$; $\tau > 1$ and $\eta_0 > 0$, batch size $b$; $G_u^1= G_{\omega}^1=G_{\lambda}^1=0$; 
\FOR{$i = 0: N$}
\STATE{$\alpha_{i} = \alpha_0 \tau^{i} $}
\STATE{$\eta_i = \eta_0 / \tau^{i} \times \text{lr\_schedule}$}
\FOR{$k = 0:K$}
\STATE{Randomly generating the mini-batch samples $S_{\text{train}}^k, S_{\text{val}}^k$ from $S_{\text{train}}$ and $S_{\text{val}}$}
\STATE{$G_u^{k+1} = \beta G_u^k + \alpha_i \nabla_{u} L_2(u_k^i, \lambda_k^i; S_{\text{train}}^k)  $}

\STATE{$u_{k+1}^i = u_k^i - \eta_i  G_u^{k+1}$
}
\vspace{0.5em}
\STATE{$G_{\omega}^{k+1} = \beta G_{\omega}^k +  \nabla_{\omega} L_1(\omega_k^i, \lambda_k^i; S_{\text{val}}^k) + \alpha_i \nabla_{\omega} L_2(\omega_k^i, \lambda_k^i; S_{\text{train}}^k)  $}
\STATE{$\omega_{k+1}^i = \omega_k^i - \eta_i G_{\omega}^{k+1}$ }
\vspace{0.5em}
\STATE{$G_{\lambda}^{k+1} = \beta G_{\lambda}^k + \nabla_{\lambda} L_1(\omega_{k}^i, \lambda_k^i; S_{\text{val}}^k) + \alpha_i \left(\nabla_{\lambda} L_2(\omega_{k}^i, \lambda_k^i; S_{\text{train}}^k) - \nabla_{\lambda} L_2(u_{k}^i, \lambda_k^i; S_{\text{train}}^k) \right)  $}
\STATE{$\hat{\lambda}_{k+1}^i = \lambda_k^i - \eta_i G_{\lambda}^{k+1}$}
\STATE{$\lambda_{k+1}^i = \text{Proj}_{\Lambda}(\hat{\lambda}_{k+1}^i)$}
\ENDFOR
\ENDFOR
\end{algorithmic}
\end{algorithm}

\deleted{The details of all test methods are presented below.} For the bilevel methods, batch size 256 is adopted for mini-batch stochastic gradient estimation in the inner optimization. The full validation data is utilized to update the outer parameters. We use stochastic gradient descent with momentum as the optimizer and cosine as the learning rate for both inner and outer optimization.
For \deleted{the} bilevel algorithms, we tune the initial learning rate of the inner optimizer from $\left\lbrace 10^{-4}, 10^{-3}, 10^{-2}, 10^{-1}\right\rbrace$ and the initial learning rate for the outer optimizer is from $\left\lbrace 10^{-4}, 5\times 10^{-4}, 10^{-3},5\times 10^{-3}, 10^{-2}, 5\times 10^{-2}, 10^{-1}, 5\times 10^{-1}\right\rbrace$. For the truncated reverse method, we use the truncated step $K_0=500$ with running 50 inner epochs and $K_0=175$ (full) for one inner epoch training.

For our minimax algorithm, \added{we implement the stochastic version of Algorithm \ref{alg:minmax:general} and use momentum to accelerate the convergence, as shown in Algorithm \ref{alg:minmax:3}. We use the same batch size of 256 for both training and validation datasets. The cosine schedule~\citep{loshchilovsgdr} is introduced into learning rate $\eta_k^i$: $\eta_k^i = \eta_0 /\tau^{i} \times 0.5 \times \left(\cos(\pi \times t/T) + 1\right)$ where $t = i \times K + k + 1$ and $T=KN$.} we set $\tau=1.5$ and the length of the inner loop $K$ is selected from $\left\lbrace 500, 1000, 1500, 2000\right\rbrace$ and the length of the outer loop $N$ is chosen from $\left\lbrace 5, 10, 15, 20\right\rbrace$. For simplicity, we use the same learning rate for all the parameters. The initial learning rate $\eta_0$ is selected from $\left\lbrace 10^{-4}, 5\times 10^{-4}, 10^{-3}, 5\times 10^{-3}, 10^{-2}, 5\times 10^{-2}, 10^{-1}\right\rbrace$ and initial $\alpha_0$ is chosen from $\left\lbrace 1, 5, 10, 50, 100\right\rbrace$. 

Regarding the environment, we run all experiments on NVIDIA GeForce RTX 2080 Ti GPUs, along with Python 3.7.6 and torch 1.13.1~\citep{paszke2019pytorch} for our software dependency.


\subsection{Experimental Details of Hyper-data Cleaning Task in Subsection~\ref{sec:num:hyper:clean}}\label{append:hyper:clean}
The details of the test methods are presented below. The inner optimization step $K$ is best-tuned from the set $\left\lbrace 100, 200, 300, 400, 500 \right\rbrace$ and the learning rate for both inner and outer optimization methods is selected from the set $\left\lbrace 10^{-4}, 5 \times 10^{-4}, 10^{-3}, 5 \times 10^{-3}, 10^{-2}, 5 \times 10^{-2}, 10^{-1}, 5 \times 10^{-1},  1\right\rbrace$.  For stocBiO, we set the inner optimization step $K=300$, learning rate for both inner and outer is 0.1 and 0.01.  For the minimax method, we set inner optimization iteration $K=300$, $\tau=1.5$, the initial learning rate $\eta_0$ is 0.1, initial value of $\alpha$ is $\alpha_0=0.05$. For the reverse and CG methods, we apply gradient descent on the entire dataset to optimize the inner and outer optimization with inner steps $K=50$; the learning rates for the inner and outer problems are 0.1 and 0.001, respectively.

In Subsection~\ref{sec:num:hyper:clean}, to make a fair comparison with stocBiO, we implement the stochastic version of Algorithm~\ref{alg:minmax:general} without momentum. We observe that the proposed minimax method can achieve a higher test accuracy and is more robust to the noise than stocBiO. By introducing momentum and cosine learning rate scheduler into Algorithm~\ref{alg:minmax:general} (see Algorithm~\ref{alg:minmax:3}), the performance of the minimax method can be improved further. The results are addressed in Table~\ref{tab:accu:noise} \added{, where the averaged best test accuracy from five different seeds is reported to eliminate the randomness.}

\begin{table}
 \caption{The averaged test accuracy}
     \label{tab:accu:noise} 
           \centering
           \footnotesize
     \begin{tabular}{c|c|c|c}
     \toprule
        &  \multicolumn{3}{c}{Test accuracy (\%, best)} \\ \midrule
       Noise   & stocBiO &  MinimaxOPT & MinimaxOPT + momentum + cosine \\ \midrule
  $p = 0.1$   & 90.09  & $ 90.91$ & $92.33$ \\ \midrule 
      $p=0.3$ & 85.79  & $90.45$ & $91.33$ \\\midrule  
      $p=0.5$ & 78.47  & ${90.38}$ & $91.03$ \\ \midrule 
     \end{tabular}
\end{table}

\end{document}